\PassOptionsToPackage{square,numbers}{natbib}
\PassOptionsToPackage{dvipsnames}{xcolor}
\documentclass{article}
\usepackage[preprint]{neurips_2026}
\usepackage{algorithm,algorithmic}
\usepackage{tikz}
\usetikzlibrary{positioning, calc, arrows.meta, backgrounds}
\usepackage{graphicx}
\usepackage[utf8]{inputenc}
\usepackage[T1]{fontenc}
\usepackage{tcolorbox}
\usepackage{multirow}
\usepackage[square,numbers]{natbib}
\usepackage{url}
\usepackage{booktabs}
\usepackage{amsfonts}
\usepackage{nicefrac}
\usepackage{microtype}
\usepackage{float}
\usepackage{needspace}
\usepackage{amsthm}
\usepackage{amssymb}
\usepackage{mathtools}
\usepackage{bm}
\usepackage{svg}
\newtheorem{definition}{Definition}
\newtheorem{theorem}{Theorem}
\newtheorem{lemma}{Lemma}
\newtheorem{corollary}{Corollary}
\newtheorem{proposition}{Proposition}
\newtheorem{assumption}{Assumption}
\newtheorem{remark}{Remark}

\newcommand{\Vin}{\mathcal{V}_K^\star}
\newcommand{\Vout}{\mathcal{V}_M^\star}
\newcommand{\xin}{\bm{x}_{\textnormal{in}}}
\newcommand{\xout}{\bm{x}_{\textnormal{out}}}

\definecolor{codegreen}{rgb}{0,0.6,0}
\definecolor{codegray}{rgb}{0.5,0.5,0.5}
\definecolor{codepurple}{rgb}{0.58,0,0.82}
\definecolor{backcolour}{rgb}{0.97,0.97,0.94}
\usepackage{listings}
\lstdefinestyle{pythonstyle}{
    backgroundcolor=\color{backcolour},
    commentstyle=\color{codegreen},
    keywordstyle=\color{magenta},
    numberstyle=\tiny\color{codegray},
    stringstyle=\color{codepurple},
    basicstyle=\ttfamily\scriptsize,
    breakatwhitespace=false,
    breaklines=true,
    captionpos=b,
    keepspaces=true,
    numbers=left,
    numbersep=3pt,
    showspaces=false,
    showstringspaces=false,
    showtabs=false,
    tabsize=4,
    aboveskip=6pt,
    belowskip=6pt,
    language=Python
}
\setcitestyle{numbers}
\setcitestyle{square}
\title{Entropy Distribution as a Fingerprint for Hallucinations in Generative Models}

\author{
  Mattia J. Villani\thanks{Equal contribution. Email: \texttt{\{mattia.villani, pranav.deshpande, akshay.seshadri\}@jpmchase.com.}}
  \quad
  Pranav Deshpande$^*$ \quad
  Akshay Seshadri$^*$ \quad
  \\\textbf{Romina Yalovetzky} \quad
  \textbf{Niraj Kumar}\thanks{Principal Investigator. Email: \texttt{niraj.x7.kumar@jpmchase.com}}\\
  \\Global Technology Applied Research, JPMorganChase, New York, NY 10001, USA\\
}

\begin{document}

\maketitle

\begin{abstract}
Large Language Models (LLMs) often generate factually incorrect outputs, commonly termed hallucinations, that undermine trust and limit deployment in high-stakes settings.
Existing hallucination detection methods typically require multiple forward passes, or access to model internals.
In this work, we provide theoretical background and empirical evidence that the \emph{distribution} of token-level entropies, beyond the mean captured by perplexity or length-normalised entropy, serves as a fingerprint of hallucination, with distributional shape and tail behaviour carrying independent signal.
We formalize hallucination detection as a statistical hypothesis test and propose the \textbf{C}alibrated \textbf{E}ntropy \textbf{S}core (\textbf{CES}), a lightweight algorithm requiring only a single forward pass and black-box access to token logits.
CES combines the mean signal with the maximum signal of the generated entropy through a calibrated reference CDF, producing scores that are directly comparable across models and tasks.
We establish finite-sample calibration guarantees via a novel random-length Dvoretzky--Kiefer--Wolfowitz inequality, and also prove that CES detects hallucinations with probability converging to one exponentially fast in the generation length.
Across eight QA benchmarks and ten generator models spanning open-source and API access models, CES achieves the highest detection performance among all single-pass black-box methods while providing formal error guarantees that existing heuristics lack.
Remarkably, CES is statistically indistinguishable from multi-sample methods that require far greater computational cost, closing the gap between lightweight and expensive detection and making it suitable for real-time, large-scale deployment.
\end{abstract}

\begin{figure}[h]
\centering
\resizebox{0.9\textwidth}{!}{%
\begin{tikzpicture}[
    >=stealth,
    font=\small,
    node distance=0.6cm,
    llmbox/.style={draw=black!80, fill=blue!8, rounded corners=3pt,
                   minimum height=0.9cm, minimum width=1.8cm, align=center,
                   line width=0.7pt},
    procbox/.style={draw=black!70, fill=orange!10, rounded corners=3pt,
                    minimum height=0.7cm, minimum width=1.5cm, align=center,
                    line width=0.6pt},
    statbox/.style={draw=black!70, rounded corners=3pt,
                    minimum height=0.85cm, minimum width=2.2cm, align=center,
                    line width=0.6pt},
    combbox/.style={draw=black!80, fill=red!6, rounded corners=5pt,
                    minimum height=0.85cm, minimum width=3.6cm, align=center,
                    line width=0.7pt},
    decbox/.style={draw=black!80, fill=green!8, rounded corners=5pt,
                   minimum height=0.75cm, minimum width=2.4cm, align=center,
                   line width=0.7pt},
    calbox/.style={draw=violet!50, fill=violet!6, rounded corners=3pt,
                   minimum height=0.85cm, minimum width=2cm, align=center,
                   line width=0.6pt, dashed},
    arr/.style={->, thick, color=black!65},
    darr/.style={->, thick, color=violet!55, dashed},
    lbl/.style={font=\tiny, text=black!60},
]

\node[llmbox] (llm) at (-0.5, 0) {\textbf{LLM} $f_\theta$};
\node[left=0.8cm of llm, font=\small] (prompt) {$\bm{x}_{\textnormal{in}}$};
\draw[arr] (prompt) -- (llm);

\node[right=0.65cm of llm, font=\small, align=left] (tokens) {$\bm{p}^{(1)},\ldots,\bm{p}^{(m)}$};
\draw[arr] (llm) -- (tokens);

\node[procbox] (entropy) at (5.5, 0) {\textbf{Entropy} $h^{(t)}\!=\!H(\bm{p}^{(t)})$};
\draw[arr] (tokens) -- (entropy);

\begin{scope}[shift={(8.0, 0)}]
  \draw[black!30, line width=0.3pt] (0,0) -- (1.2,0);
  \foreach \x/\ht/\col in {
    0.00/0.18/green!60!, 0.12/0.14/green!60,
    0.24/0.22/green!65, 0.36/0.16/green!60,
    0.48/0.20/green!60, 0.60/0.40/orange!80!red,
    0.72/0.50/red!65,  0.84/0.55/red!70,
    0.96/0.45/orange!75!red, 1.08/0.32/orange!65
  }{
    \fill[\col, opacity=0.8] (\x, 0) rectangle (\x+0.08, \ht);
  }
  \node[font=\tiny, text=black!50] at (0.6, -0.15) {$t$};
  \node[font=\tiny, above] at (0.6, 0.58) {$\bm{h}$};
\end{scope}


\node[statbox, fill=cyan!10] (meanstat) at (3.5, -1.8) {%
  \textbf{Mean signal}\\[-2pt]
  {\tiny $\widehat{F}_0(\bar{h})$}
};
\node[statbox, fill=yellow!15] (maxstat) at (7.5, -1.8) {%
  \textbf{Max signal}\\[-2pt]
  {\tiny $\widehat{F}_0(h_{\max})$}
};

\draw[arr] (entropy.south) -- ++(0, -0.4) -| (meanstat.north)
    node[lbl, pos=0.75, left] {$\bar{h}$};
\draw[arr] (entropy.south) -- ++(0, -0.4) -| (maxstat.north)
    node[lbl, pos=0.75, right] {$h_{\max}$};

\node[font=\tiny, text=cyan!45!black, below left=0.05cm and -1.2cm of meanstat] {distributional shift};
\node[font=\tiny, text=yellow!40!black, below right=0.05cm and -1cm of maxstat] {tail anomaly};

\node[combbox] (ces) at (5.5, -3.4) {%
  \textbf{CES} \;
  $\sqrt{\widehat{F}_0(\bar{h})\;\cdot\;\widehat{F}_0(h_{\max})}$
};

\draw[arr, color=cyan!55!black] (meanstat.south) -- ++(0, -0.35) -| ([xshift=-0.6cm]ces.north)
    node[lbl, pos=0.3, below left] {$\widehat{F}_0(\bar{h})$};
\draw[arr, color=yellow!50!black] (maxstat.south) -- ++(0, -0.35) -| ([xshift=0.6cm]ces.north)
    node[lbl, pos=0.3, below right] {$\widehat{F}_0(h_{\max})$};

\node[decbox] (decision) at (5.5, -4.7) {%
  Reject $\mathcal{H}_0$ if $\mathrm{CES} > c_\alpha$
};
\draw[arr] (ces) -- (decision) node[lbl, midway, right, xshift=1pt] {CES};

\node[font=\small, right=0.5cm of decision, align=left] (outlbl) {%
  {\color{red!65!black}\boldmath$\times$\unboldmath\; Hallucination}\\[1pt]
  {\color{green!50!black}\checkmark\; Faithful}
};
\draw[arr] (decision.east) -- (outlbl.west);

\node[calbox, minimum height=1cm] (calib) at (-0.2, -1.8) {%
  \textbf{Calibration}\\[-2pt]
  {\tiny Oracle $\mathcal{O} \!\to\! \widehat{F}_0$}
};

\draw[darr] (llm.south) -- ++(0, -0.45) -| (calib.north)
    node[lbl, pos=0.25, left, text=violet!50] {\tiny cal.\ data};
\draw[darr] (calib.east) -- (meanstat.west)
    node[lbl, midway, above, text=violet!55] {$\widehat{F}_0$};
\draw[darr] (calib.south) -- ++(0, -0.06) -| ([xshift=-0.5cm]maxstat.west) -- (maxstat.south west |- maxstat.west)
    node[lbl, midway, above left, text=violet!55] {$\widehat{F}_0$};

\node[draw=black!25, rounded corners=2pt, fill=white,
      font=\tiny, align=left, inner sep=3pt] at (10.2, -3.4) {%
  $\mathcal{H}_0$: $\bm{h} \sim F_0$\\[1pt]
  $\mathcal{H}_1$: $\bm{h} \not\sim F_0$
};

\end{tikzpicture}%
}
\caption{%
\textbf{Calibrated Entropy Score (CES) hallucination detection.}
A single forward pass produces token distributions $\bm{p}^{(t)}$, from which entropies $h^{(t)}$ are computed.
Two summary statistics (the mean entropy $\bar{h}$ and the maximum entropy $h_{\max}$) are mapped through the calibrated CDF $\widehat{F}_0$ and combined via a geometric mean.
The reference CDF, $\widehat{F}_0$, is estimated offline from an oracle-labeled calibration set (dashed path; Algorithm~\ref{alg:calibration}).
}
\label{fig:method_overview}
\end{figure}
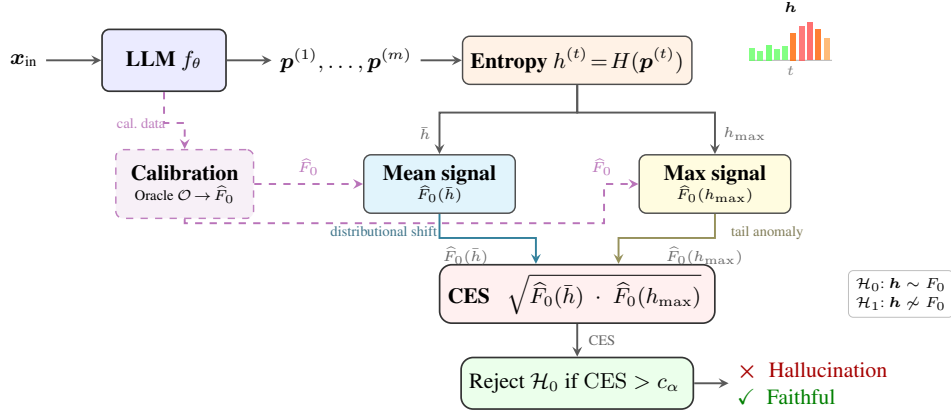
\section{Introduction}

Despite Large Language Models (LLMs) fluency in generating text and solving an increasing number of tasks, these models are prone to fabricate information, a phenomenon
known as hallucination \cite{xu2024hallucination}. 
This hinders the trust towards the model, limiting the extent to which tasks can be delegated. Mitigating, eliminating, or detecting hallucinations is an increasingly important research effort \cite{ji2023survey, huang2025survey}. 
A growing body of work seeks to detect hallucinations post-generation, broadly falling into three families: uncertainty-quantification methods that study the output trace \cite{jelinek1977perplexity, malinin2020uncertainty, huang2025look, xia2025survey}, consistency-based methods that compare multiple generations \cite{farquhar2024detecting, nikitin2024kernel, manakul2023selfcheckgpt}, and elicitation-based methods that prompt the model to self-evaluate \cite{pedapati2024large, kadavath2022language}.
Despite empirical progress, these approaches share important limitations: consistency methods require $K \geq 5$ forward passes, internal-state methods require access to hidden representations \cite{duan2024llms, chen2024inside}, and single-pass heuristics such as perplexity or length-normalized entropy reduce the full entropy signal to a single scalar, discarding distributional information.

Model output logits and entropy have been shown to possess predictive power in detecting hallucinations. 
The most successful approaches to date estimate perplexity \cite{ren2022out} and length-normalized entropy \cite{malinin2020uncertainty}, while even generation length has been found to correlate with the presence of hallucinations. 
These methods rest on the intuition that entropy, as a measure of uncertainty, can signal whether a model is confident in its next-token prediction. 
However, they typically reduce the full entropy signal to a single scalar summary, such as a mean or normalized aggregate, potentially discarding richer distributional information that could further improve detection.

In this work, we gather evidence that the distributional information is important.
We demonstrate empirically that the entropy distributions over the tokens are fingerprints of hallucinations: \textit{when a model hallucinates, the underlying distribution from which the token entropies are sampled, diverges from the distribution observed when generations are faithful}.
Crucially, this divergence is not merely a shift in mean entropy:
we demonstrate through a synthetic mean-shift null experiment (Section~\ref{sec:exp_distributions}) that after
removing the mean difference entirely, the \emph{shape} of the
entropy distribution carries substantial independent signal.

This motivates our proposed hallucination detection method:
without requiring queries to the LLMs or model internals, we formalize hallucination detection as a statistical hypothesis test.  
We build a reference cumulative distribution function (CDF) from non-hallucinated calibration data, then test whether a new generation's entropy sequence is statistically consistent with this reference distribution. 
The framework cleanly separates the \emph{definition} of hallucination (delegated to an oracle during
calibration) from its \emph{detection} (performed via the statistical test), allowing our methods to be applied regardless of the choice of model, task, or hallucination definition.
Moreover, our method applies when labeled data are insufficient to construct a reference CDF: we introduce an unsupervised detector that achieves performance comparable to its supervised counterpart.

\textbf{Contributions.} Our contributions are as follows:
\begin{itemize}
    \item \textbf{Entropy distributions as hallucination fingerprints.}
    We provide a theoretical framework and empirical evidence to capture hallucinations in an LLM by detecting the divergence of the distribution of token entropies between hallucinated and faithful generations. The faithful distribution can be estimated with labelled data or in an unsupervised fashion. 

    \item \textbf{A statistically principled detection framework.}
    We formalize hallucination detection as a hypothesis test: when the model is not hallucinating (null hypothesis), the token entropies are consistent with a reference CDF estimated from non-hallucinated calibration data.
    We introduce the \emph{Calibrated Entropy Score} (CES), which combines the mean and maximum entropy through the reference CDF to capture both distributional shift and tail anomalies.
    CES is computed in a single forward pass from black-box token logits, requires no hyperparameter tuning, and is comparable across models and tasks without threshold recalibration, making it well suited for high-throughput production use.

    \item \textbf{Theoretical guarantees.} We show that CDF constructed by pooling varying-length entropy sequences produces a good estimate of the non-hallucinated reference CDF by generalizing the Dvoretzky–Kiefer–Wolfowitz (DKW) inequality to random-length random vectors. We further prove that the hypothesis test obtained from CES has strong performance, with false positive and false negative rates decaying exponentially with generation length.

    \item \textbf{Extensive empirical evidence.} We evaluate CES on 8 QA benchmarks and 10 models (Llama-2 7B/13B/chat, Llama-3 8B-Instruct, Llama-3.2-1B, GPT-4.1/mini/nano, GPT-4o-mini) and find that it consistently outperforms query-access baselines while matching more complex approaches (e.g., multi-trajectory or multi-pass methods). CES requires no architectural changes and uses only black-box token logits from a single forward pass, unlike consistency-based methods that require more than $5$ generations.
\end{itemize}
Separation of the hallucination \emph{definition} (delegated to an oracle during calibration) from \emph{detection} (via the statistical test) makes the framework agnostic to
the model, task, or hallucination taxonomy.

\section{Background}
\textbf{Notation.}
Let $[m]=\{1,\ldots m\}$ for $m \in \mathbb{Z}_+$.
$\Delta(A)$ is the probability simplex over a finite set of symbols $A$. Any element $\bm{p} \in \Delta(A)$ can be considered as a vector in $\mathbb{R}^{|A|}$. Let $\mathbf{1}_S$ be the indicator function of a set $S$. Given a set $\mathcal{V}$, $\mathcal{V}^\star_K = \sqcup_{i = 1}^K \mathcal{V}^K$ denotes the set of tuples/strings of size at most $K$ constructed from elements in $\mathcal{V}$.
Given two cumulative distribution functions $F_0$ and $F_1$, we define the Kolmogorov-Smirnov (KS) distance between $F_0$ and $F_1$ as $d_{KS}(F_0, F_1) = \sup_{x \in \mathbb{R}} |F_0(x) - F_1(x)|$.

\textbf{Autoregressive language models.}\label{par:autoregllm}
Let $\mathcal{V}$ be a finite vocabulary, with size $|\mathcal{V}| = d$; $\mathcal{V}^\star_K$ is the set of strings of length at most $K$. 
LLMs generate text via autoregresive sampling from a learned next-token distributions.
Let $f_\theta: \mathcal{V}^\star_K \rightarrow \Delta(\mathcal{V})$ be the LLM function, which maps a context to the next-token distribution. 
Specifically, at a step $t \in \mathbb{N}$ of the generation, the model produces $\bm{p}^{(t)} = f_\theta(x_t)$ distribution, where $x_t \in \mathcal{V}_K^\star$. 
\begin{definition}[Token Entropy]
\label{def:token_entropy}
At a generation step $t \in [0,m]$ of a length $m\in \mathbb{N}$ generation, the \textbf{token entropy} is defined as
\[h^{(t)} = H(\bm{p}^{(t)}) = - \sum_{v \in \mathcal{V}} p^{(t)}_v \log p_v ^{(t)},\]
where $H$ is the Shannon entropy with natural logarithm and $\mathcal{V}$ is a finite set of $d$ tokens (the vocabulary).
The \textbf{entropy sequence} is given by $\bm{h} = (h^{(1)}, ..., h^{(m)}).$
\end{definition}

\section{Methodology}
\label{sec:methodology}
Our approach consists of two stages: (1) constructing a reference entropy distribution from non-
hallucinated calibration data, and (2) testing whether a new generation’s entropy sequence is consistent
with this reference. 
In order for our methods to be applicable no matter the definition of hallucination required by the application, we abstract away the details by using an \textit{oracle} that detects hallucination according to the chosen definition. 

Formally, given a maximum generation length $M$ and a maximum input context length $K$, a hallucination detection oracle is a function $\mathcal{O}: \mathcal{V}_K^\star \times \mathcal{V}^\star_M \rightarrow \{0,1\}$ that satisfies $\mathcal{O}(x_\text{in}, x_\text{out}) = 1$ if the output given the input is a hallucination and $0$ otherwise.
Such an oracle, for example, can be a human expert who determines whether a given input-output pair is a hallucination, or it could be an LLM judge that does the same.
While an oracle could directly label hallucinations, it is typically costly or unavailable during production.
We therefore seek to capture the decisions of the oracle statistically with high probability and low computational overhead.
%
With this in mind, we propose a method for detecting hallucination that consists of two main components:
\begin{enumerate}
    \item \textbf{Non-hallucinated distribution:} We use the distribution of the token entropies when the model does not hallucinate as a proxy for capturing the statistical behaviour of the oracle. In practice, such a distribution can be obtained (approximately) using calibration data along with the oracle.
    \item \textbf{Hypothesis test:} We test whether the token entropies obtained during generation are consistent with the non-hallucinated distribution of the token entropies and flag deviations from this reference distribution as hallucinations.
\end{enumerate}
We expand on each of these components below.

\textbf{Non-hallucinated reference distribution.}
When a model is well-supported by its training data, strong contextual constraints concentrate probability on a small number of plausible tokens, yielding low entropy; when these constraints weaken, as is expected during hallucination, the predictive distribution becomes less structured and entropy rises.
This motivates us to focus on token entropies in this study, which are obtained from the logit probabilities at each time step during generation.
In general, entropy-based signals have consistently demonstrated empirical predictive power for hallucination detection across a range of prior studies
\cite{huang2025look, janiak2025illusion, varshney2023stitch,
malinin2020uncertainty}, typically summarized as scalar aggregates
such as perplexity or length-normalized entropy.
Our work extends this line by showing that the shape of the distribution of token entropies carries additional information.
Moreover, since entropies are scalars (as opposed to high-dimensional logit probability vectors), we can perform statistical tests with entropies efficiently.

We denote $\mathcal{E} = ([0, \log(d)])_M^\star$ to be the set of entropy sequences of length at most $M$, noting that $H(\bm{p}) \in [0, \log(d)]$ for all distributions $\bm{p} \in \Delta(\mathcal{V})$. If instead we work with the top-$k$ logit probabilities, we can choose $\mathcal{E} = ([0, \log(k)])_M^\star$. The discussion below applies to either case.

We now seek to construct the distribution of the token entropy over non-hallucinated instances. For this purpose, we first define a function $\mathfrak{f}_{\mathcal{O}}$ that filters out the hallucinated instances using the oracle $\mathcal{O}$. Formally, $\mathfrak{f}_{\mathcal{O}}\colon \mathcal{V}_K^\star \times \mathcal{V}_M^\star \times \mathcal{E} \to \mathcal{E} \sqcup \{\bot\}$ is defined as
\begin{equation}
    \mathfrak{f}_{\mathcal{O}}(x_\text{in}, x_\text{out}, \bm{h}_\text{out}) 
    = \begin{cases}
          \bm{h}_{\textnormal{out}} & 
          \textnormal{ if } \mathcal{O}(x_\text{in}, x_\text{out}) = 0 \\
          \bot                                        
          & \textnormal{ if } \mathcal{O}(x_\text{in}, x_\text{out}) = 1,
      \end{cases}  \label{eqn:oracle_filter}
\end{equation}
where $\bm{h}_{\textnormal{out}} \in \mathcal{E}$ is the sequence of token entropies obtained during the output generation. Suppose that $\mathbb{P}$ is the joint distribution over the input context, the output, as well as the token entropies. Then, the distribution $\mathbb{P}_{\mathcal{O}}$ induced by $\mathfrak{f}_{\mathcal{O}}$ on $\mathcal{E} \sqcup \{\bot\}$ is the pushforward of $\mathbb{P}$ under $\mathfrak{f}_{\mathcal{O}}$.
Since we only want the distribution of the token entropy for non-hallucinated instances, we condition on the event that the model is not hallucinating. This gives us the distribution
\begin{equation}
    \overline{\mathbb{P}}_0(A) = \frac{\mathbb{P}(A \cap \mathcal{E})}{\mathbb{P}(\mathcal{E})}\qquad (A \subseteq \mathcal{E} \textnormal{ measurable}). \label{eqn:proxy_nonhal_dist}
\end{equation}
While the distribution $\overline{\mathbb{P}}_0$ can look unwieldy, we make a simplifying assumption noted below that helps with both the theoretical analysis and also ensures that the hypothesis tests perform well.
\begin{assumption}[Conditional i.i.d.]
\label{asmp:iid_m}
Conditioned on the generation length $T = m$, the token entropies $h^{(1)}, \ldots, h^{(m)}$ of a non-hallucinated generation are i.i.d.\ draws from a distribution $\mathbb{P}_0$ with CDF $F_0$.
\end{assumption}
While such an independence assumption does not strictly hold in practice, since we are only working with the token entropy and not the tokens directly, we expect it to hold approximately for the token entropy, as seen from experimental results in Appendix~\ref{app:exp_independence}.

\textbf{Calibration.}
In order to obtain an approximation of the distribution $\mathbb{P}_0$, we use calibration data where we assume access to labeled data from an oracle. 
Algorithm~\ref{alg:calibration} shows how to approximate the cumulative distribution function (CDF) $F_0$ of the distribution $\mathbb{P}_0$ using calibration data.
\begin{algorithm}[H]
\caption{CDF Calibration}\label{alg:calibration}
\begin{algorithmic}[1]
\REQUIRE Calibration data $(\bm{x}_{\textnormal{in}, i}, \bm{x}_{\textnormal{out}, i}, \bm{h}_i)_{i = 1}^n$; hallucination detection oracle $\mathcal{O}$
\STATE Set $\mathcal{D} = \varnothing$.
\STATE For each $i \in [n]$, do $\mathcal{D} \gets \mathcal{D} \cup \{\bm{h}_i\}$ if $\mathcal{O}(\bm{x}_{\textnormal{in}, i}, \bm{x}_{\textnormal{out}, i}) = 0$.
\STATE Compute the function:
        \begin{equation}
            \widehat{F}_0(z) = \frac{1}{N} \sum_{\bm{h} \in \mathcal{D}} \sum_{t = 1}^{\textnormal{dim}(\bm{h})} \bm{1}\{h^{(t)} \leq z\}
        \end{equation}
        where $N = \sum_{\bm{h} \in \mathcal{D}} \textnormal{dim}(\bm{h})$ and $z \in [0, \log(d)]$.
\STATE \textbf{return} $\widehat{F}_0$.
\end{algorithmic}
\end{algorithm}
We show in Corollary~\ref{cor:cdfcalalg_converge} that the calibrated distribution $\widehat{F}_0$ constructed using Algorithm~\ref{alg:calibration} converges to the true distribution $F_0$ exponentially quickly with the number of non-hallucinated instances $|\mathcal{D}|$. This result is based on the following generalization of DKW inequality to random length vectors.

\begin{theorem}[Random-length DKW inequality, informal]
    \label{thm:random_length_DKW_informal}
    Consider a random variable $\bm{X}$ taking values in $\mathbb{R}_M^\star$ with length $T = \textnormal{dim}(\bm{X})$. Suppose that there is some distribution $\mathbb{P}_0$ on $\mathbb{R}$ such that for all $m \in [M]$,
    \begin{equation}
        \mathbb{P}(\bm{X} \in A| T = m) = \mathbb{P}_0^m(A \cap \mathbb{R}^m), \label{eqn:gen_iid_m_asmp}
    \end{equation}
    for all measurable subsets $A \subseteq \mathbb{R}_M^\star$.
    Then, if $\bm{X}_1, \dotsc, \bm{X}_L$ are iid copies of $\bm{X}$, $T_i = \textnormal{dim}(\bm{X}_i)$ for $i \in [L]$, $F_0$ is the CDF of $\mathbb{P}_0$, and
    \begin{equation}
        \widehat{F}_0(x) = \frac{1}{\sum_{i = 1}^L \textnormal{dim}(\bm{X}_i)} \sum_{i = 1}^L \sum_{k = 1}^{\textnormal{dim}(\bm{X}_i)} \bm{1}_{(-\infty, x]}(X_{i, k})
    \end{equation}
    is the empirical CDF, we have
    \begin{equation}
        \mathbb{P}^L\left(\sup_{x \in \mathbb{R}} |\widehat{F}_0(x) - F_0(x)| \geq \epsilon\right) \leq 2 \left(\mathbb{E}\left[\exp(-2 \textnormal{dim}(\bm{X}) \epsilon^2)\right]\right)^L \leq 2 \exp(-2 L \epsilon^2). \label{eqn:random_length_DKW}
    \end{equation}
\end{theorem}
A formal and extended version of this theorem along with the proof is given in Appendix~\ref{app:cal_theory}.
When oracle labels are unavailable but hallucinations are rare, our method can estimate $\mathbb{P}_0$ in an unsupervised manner, as described below.

\textbf{Calibration without access to the oracle.}\label{par:unsupervised}
When judge labels are unavailable or impractical to compute, one use an \textit{unsupervised} version of our method. In the unsupervised method, we compute the empirical CDF of token entropy directly from all the calibration samples. When the fraction of hallucinated samples $\gamma \in [0,1)$ is small, we can still estimate the non-hallucinated distribution well.

\begin{proposition}[Contamination Robustness]
    Suppose that the true CDF $F_\gamma$ of token entropy is $\gamma \in [0, 1)$ far from $F_0$ in KS distance, i.e., $d_{KS}(F_\gamma, F_0) \leq \gamma$.
    Let $\widehat{F}_\gamma$ denote the empirical CDF obtained using calibration samples within $\epsilon > 0$ error in KS distance of $F_\gamma$ with probability $1 - \delta \in (0, 1)$. Then, with probability $1 - \delta$, we have
    \begin{equation}
        d_{KS}(\widehat{F}_\gamma, F_0) \leq \epsilon + \gamma.
    \end{equation}
\end{proposition}
\begin{proof}
With probability $1 - \delta$, we have $d_{KS}(\widehat{F}_\gamma, F_\gamma) \leq \epsilon$. Since the KS distance is a metric, we have $d_{KS}(\widehat{F}_\gamma, F_0) \leq d_{KS}(\widehat{F}_\gamma, F_\gamma) + d_{KS}(F_\gamma, F_0) \leq \epsilon + \gamma$.
\end{proof}


\textbf{Hypothesis test.} We can now write down hallucination detection as a hypothesis test, which tests whether the generated entropies are consistent with the non-hallucinated distribution $F_0$. We phrase this formally below.
\begin{definition}[Hallucination Detection Test]
\label{def:HD_test}
A hallucination detection test corresponds to testing the null hypothesis:
\begin{equation}
    \mathcal{H}_0: \text{generated entropy sequence } \bm{h} \text{ is consistent with } F_0.
\end{equation}
We say that the model is hallucinating if $\mathcal{H}_0$ does not hold.
%
\end{definition}

In order to detect a hallucination, we need to reject the null hypothesis $\mathcal{H}_0$. Since this corresponds to checking whether the elements $h^{(1)}, \dotsc, h^{(m)}$ are generated iid from $F_0$, one can use any test for checking if samples are drawn from a given distribution. A popular and powerful method of doing this is the Kolmogorov-Smirnov (KS) test. In the context of hallucination detection, however, the KS test can underperform because of insufficiently many samples available when the generation length is small. Therefore, we leverage a measure that we have found to be extremely performant.

\textbf{Calibrated Entropy Score (CES).} 
The Calibrated Entropy Score is a measure given by the geometric mean of the CDF valued at the mean entropy $\bar{h} = \sum_{t = 1}^m h^{(t)} / m$ with the mean of the CDF valued at the maximum entropy $h_\text{max} = \max_{t \in [m]} h^{(t)}$ of a generation, i.e.,
\[ \text{CES}(\bm{h})= \sqrt{\hat{F}_0(\bar{h})\cdot\hat{F}_0(h_\text{max}}).\]
The CDF transformation $\widehat{F}_0(\cdot)$ serves as a normalization to convert the mean and max entropies of $\bm{h}$ to a score $\textnormal{CES}(\bm{h}) \in [0, 1]$.
Importantly, we can derive a hypothesis test using CES by setting a threshold and reject the null hypothesis (or equivalently, declare hallucination) when the observed CES value exceeds the threshold. Since one often plots a receiver operating curve (ROC) and chooses a threshold empirically based on that, we give results directly for the true and false positive rates in terms of tunable parameters that determine a threshold for CES. In Theorem~\ref{thm:ces_power_informal} below, we show that for an appropriate choice of tunable parameters, the true positive rate goes to $1$ and the false positive rate decays to $0$ exponentially quickly with the number of samples.

\begin{theorem}[Power of CES, informal]
\label{thm:ces_power_informal}
Let $F_0$ be the non-hallucinated distribution with mean $\mu_0$ and maximum $\zeta_0$, and define $\mu_1$ and $\zeta_1$ similarly for the hallucinated distribution $F_1$.
Let $\mu$ and $\zeta$ be tunable parameters that can be used to set a threshold.
Then, the following hold.

\begin{enumerate}
    \item[(a)] \textbf{True positive rate.}
    If the tunable parameters satisfy $\mu \in (-\infty, \mu_1)$ and $\zeta \in (-\infty, \zeta_1)$, we have
    \begin{equation}
      \mathbb{P}_{F_1}^m(\mathrm{CES} \geq \sqrt{F_0(\mu) F_0(\zeta)})
      \geq 1 - F_1(\zeta)^m - \exp\left(\frac{-2m (\mu_1 - \mu)^2}{(\log d)^2}\right). \label{eqn:ces_tpr_bd}
    \end{equation}

    \item[(b)] \textbf{False positive rate.} If the tunable parameters satisfy $\mu \in (\mu_0, \infty)$ and $\zeta \in (\zeta_0, \infty)$, we have
    \begin{equation}
      \mathbb{P}_{F_0}^m(\mathrm{CES} > \sqrt{F_0(\mu) F_0(\zeta)})
      \leq \exp\left(\frac{-2m (\mu - \mu_0)^2}{(\log d)^2}\right). \label{eqn:ces_fpr_bd}
    \end{equation}

    \item[(c)] \textbf{Consistency.} For tunable parameters in the range $\mu \in (\mu_0, \mu_1)$ and $\zeta \in (\zeta_0, \zeta_1)$,
    \begin{align}
        & \mathbb{P}_{F_1}^m(\textnormal{CES} \geq \sqrt{F_0(\mu) F_0(\zeta)}) \xrightarrow{\ m \to \infty\ } 1 && \textnormal{(True positive rate goes to $1$)} \\
        & \mathbb{P}_{F_0}^m(\textnormal{CES} > \sqrt{F_0(\mu) F_0(\zeta)})  \xrightarrow{\ m \to \infty\ } 0 && \textnormal{(False positive rate goes to $0$)}.
    \end{align}
\end{enumerate}
\end{theorem}
A more formal and extended version of the theorem along with the proof can be found in Appendix~\ref{app:ces_power}.

\section{Experiments}\label{sec:experiments}
Our experimental pipeline targets two questions (Q1) is there statistical evidence that hallucinated outputs tend to have markedly separate distributional fingerprints, beyond mean shift and (Q2) how does our proposed method (CES) perform against existing methods at detecting hallucinations in real world datasets? 
We describe the common experimental pipeline below.
Further experimental details are deferred to Appendix \ref{app:exp_details}. 
All experiments were run on dual Intel Xeon Platinum 8275CL CPUs with 96 cores, 192 threads, and 1.1 TB of RAM and 8 A100 GPUs with 40GB memory each. 

\textbf{Models and datasets.} We evaluate across ten instruction-tuned language models spanning parameter counts from 1B to 13B (open-weight) plus four API-only models:
Falcon-7B-Instruct \cite{almazrouei2023falcon}, Llama-3.2-1B-Instruct, Llama-2-7B-Chat, Llama-2-13B-Chat, Meta-Llama-3-8B-Instruct \cite{touvron2023llama}, Mistral-7B-Instruct-v0.3 \cite{jiang2023mistral7b}, GPT-4.1, GPT-4.1-mini, GPT-4.1-nano, and GPT-4o-mini.
Each model is evaluated on eight question-answering datasets: BioASQ \cite{krithara2023bioasq}, CoQA, DROP \cite{dua2019drop}, GSM8K \cite{cobbe2021training}, NQ-Open \cite{kwiatkowski2019natural}, SQuAD \cite{rajpurkar2018know}, SVAMP \cite{patel2021nlp}, and TriviaQA \cite{joshi2017triviaqa}
, yielding $10 \times 8 = 80$ model--dataset experiments (48 open-weight, 32 API).
For each experiment, we generate 500 responses with \texttt{max\_new\_tokens}$=128, 256$ and extract per-token logits (or log-probabilities for API models), from which Shannon entropy $H_t = -\sum_v p_v \log p_v$ is computed at each position $t$.
Binary labels (hallucinated vs.\ faithful) are assigned using a GPT-4.1-nano judge.
The median hallucination rate across experiments is 0.27, ranging from 0.04 (TriviaQA, CoQA) to 0.78 (SVAMP, NQ-Open) (see Appendix \ref{app:exp_grid} for details on hallucination rates).

\textbf{Evaluation protocol.}
All $N = 500$ samples per experiment are used both to construct the reference
ECDF (empirical CDF) $\widehat{F}_0$ and to compute detection AUROCs (no held-out partition).
This full-sample protocol maximises statistical power for the reference distribution; the lack of a separate evaluation split means reported AUROCs reflect in-sample ranking performance.

\textbf{Benchmark methods.}
We compare CES against 15 benchmark methods. 
These are: Semantic Entropy \cite{farquhar2024detecting}, Discrete Semantic Entropy, Kernel Language Entropy (KLE, KLE-heat, KLE-full) \cite{nikitin2024kernel}, Eigenscore \cite{chen2024inside}, Effective Rank, SelfCheckGPT \cite{manakul2023selfcheckgpt}, Embedding Regression, P(True) \cite{kadavath2022language}, FAVA \cite{mishra2024fine}, and Length-Normalised Entropy. 
Details on the methods are included in Appendix \ref{app:baselines}, detailed experimental results are reported in \ref{app:exp_grid}.

\subsection{Distributional Separation of Hallucinated and Faithful Generations}\label{sec:exp_distributions}

\textbf{Setup.}
For each of the 80 model--dataset pairs, we pool all token-level entropies from faithful generations and from hallucinated generations separately.
We apply the two-sample Kolmogorov--Smirnov test and compute nonparametric effect sizes ($d_{KS}$ distance).
To characterize the \emph{shape} signal beyond location shifts, we additionally apply the KS test to \emph{mean-centred} entropy sequences, removing any first-moment difference.

\textbf{Results.}
Across all 80 experiments, 72 (90\%) yield a significant two-sample KS test at $\alpha{=}0.05$ when using the full pooled token sets.
The median $d_{KS}$ distance is 0.100 (IQR: [0.059, 0.130]) and the median Cohen's $d$ is 0.192 (IQR: [0.106, 0.284]), confirming a small but statistically robust distributional shift.
Crucially, the \emph{mean-centred shape signal} remains significant in 80/80 experiments at $\alpha{=}0.05$, indicating that the distributional difference is not reducible to a simple location shift.
Figure \ref{fig:distributional_difference} illustrates the separation.
Panel~(a) shows representative ECDFs for a median-effect-size experiment, where faithful and hallucinated entropy distributions are visibly offset.
Panel~(b) displays the distribution of $d_{KS}$ distances across all 80 experiments, with the rejection threshold annotated.
Panel~(c) shows that the mean-centred shape signal is universally significant, confirming distributional shape differences beyond the mean.

\begin{figure}[h]
    \centering
    \includegraphics[width=\textwidth]{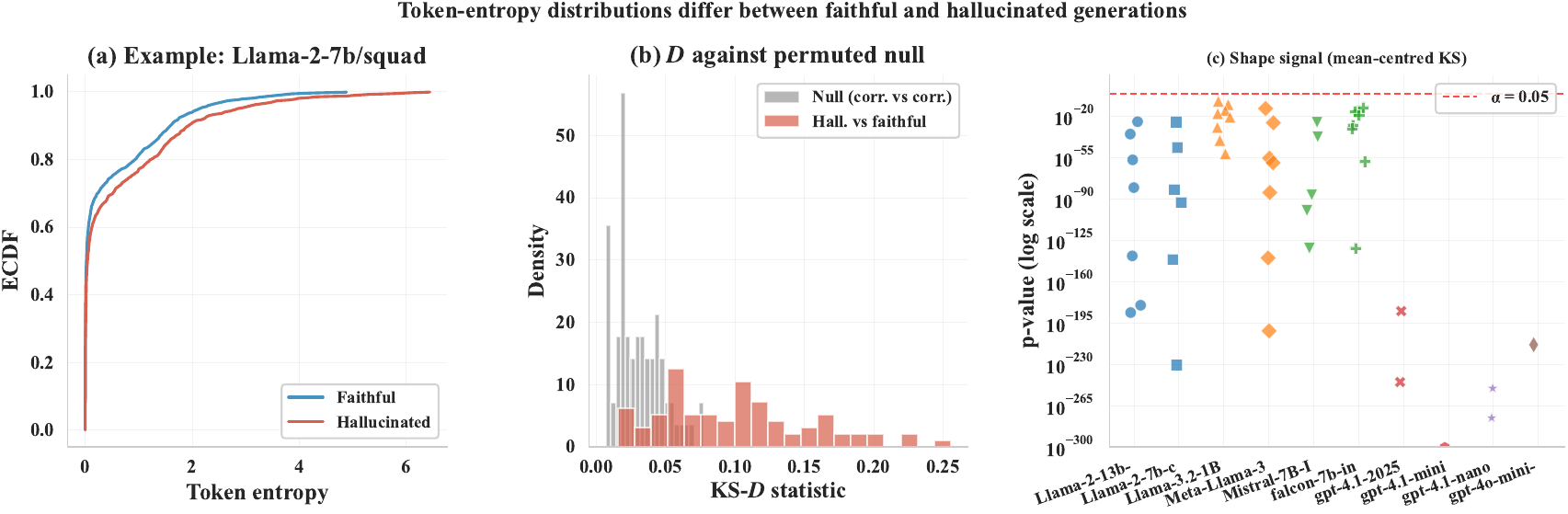}
    \caption{\textbf{Entropy distributions differ between faithful and hallucinated generations.}
    (a) Empirical CDFs for the median-effect-size experiment, with $d_{KS}$ annotated.
    (b) Distribution of $d_{KS}$ distances across 80 experiments; 72/80 are significant at $\alpha{=}0.05$, with the distributional difference being more marked for API models.
    (c) Mean-centred shape signal: 80/80 experiments exhibit significant residual distributional differences after removing the location shift.}
    \label{fig:distributional_difference}
\end{figure}

\subsection{Comparison to Benchmark Methods}\label{app:critical_difference_analysiss}

\textbf{Setup.}
We compare CES against 16 published hallucination detection methods spanning multi-sample semantic methods.
All methods are evaluated on the same 80 experiments with identical labels.
We use three complementary evaluation criteria: (i)~pairwise win rates, (ii)~Friedman/Nemenyi rank analysis, and (iii)~median AUROC.
Among methods operating under the minimal-access single-pass constraint, CES achieves the highest detection performance with formal guarantees. 
Among all methods regardless of cost, CES belongs to the top statistical clique while requiring significantly less computation than the other members of that clique.

\textbf{Pairwise win rates.}
CES (unsupervised) wins 854/1279 pairwise comparisons (66.8\%) against all 16 benchmark methods and achieves a \textgreater50\% win rate against 12/16 methods.
The strongest victories are against Generation Length (85.0\%, $\Delta = +0.085$), SelfCheckGPT (80.0\%, $\Delta = +0.034$), CES supervised (78.8\%), FAVA (78.8\%, $\Delta = +0.047$), and Semantic Entropy (72.2\%).
CES loses to KLE / KLE-heat (50.0\% win rate), KLE-full (48.8\%), and Embedding Regression (43.8\%, $\Delta = -0.017$).
We also note that our unsupervised variant outperforms all single pass methods without LLM querying in 43/80 of the experiments. 

\textbf{Friedman/Nemenyi analysis.}
With 17 methods across 80 experiments, the Friedman test strongly rejects exchangeability ($\chi^2_F = 272.80$, $p \approx 0$; Iman--Davenport $F = 21.47$).
The Nemenyi critical difference is $\text{CD} = 2.779$ ($\alpha = 0.05$, $k = 17$).
CES (unsupervised) achieves an average rank of 6.29 and belongs to the \emph{top clique} (methods within CD of the best rank 6.16). 
Holm-corrected Wilcoxon signed-rank tests confirm that CES is significantly better than 8/16 methods.
CES is \emph{not} significantly different from KLE variants, Embedding Regression, Discrete SE, Eigenscore, Effective Rank, or Semantic Entropy.

\textbf{Further Experimental Results.} The extended appendix presents a comprehensive set of ablations, robustness checks, and supplementary analyses.
Appendix \ref{app:exp_grid} reports the full $10 \times 8$ AUROC grid with per-model tables. Appendix \ref{app:exp_combinatorial} exhaustively evaluates 44 summary-statistic variants, confirming that the chosen $\text{geom}(\text{mean}, \text{max})$ formula achieves the best average rank.
Appendix \ref{app:exp_independence} verifies the independence assumption by measuring lag-1 autocorrelation in token entropy traces and deriving effective sample sizes.
Finally, Appendix \ref{app:exp_contamination} demonstrates that CES is robust to calibration contamination, maintaining performance even at 50\% hallucinated-token injection, and Appendix \ref{app:exp_noisy_judge} tests robustness to noisy hallucination labels, showing graceful degradation.

\begin{figure}[t]
    \centering
    \includegraphics[width=\textwidth]{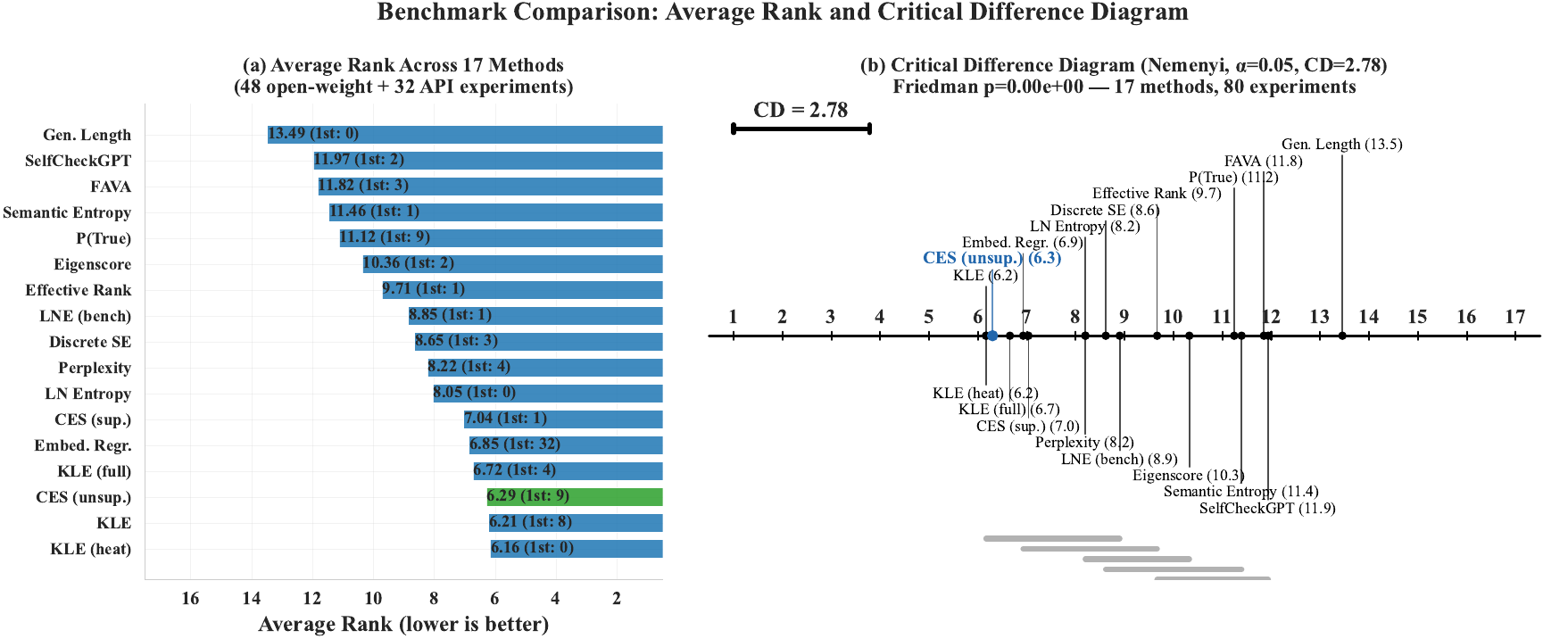}
    \caption{\textbf{Benchmark comparison of confabulation detection methods across all model--dataset experiments.}
    \textbf{(a)}~Average rank (lower is better) for each method, aggregated over all 80 open-weight and API model experiments. Bar annotations indicate the average rank and number of first-place finishes.
    \textbf{(b)}~Nemenyi Critical Difference (CD) diagram following \cite{demvsar2006statistical}.
    CES (unsupervised) ranks among the top clique without requiring access to model internals, supervised labels at test time, or external NLI models.}
    \label{fig:benchmark_cd}
\end{figure}
\begin{center}
\begin{table}[h]
\centering
\caption{AUROC $\pm$ std (rank) for representative open-weight and API models. \textbf{Bold} = best method per dataset overall. \textit{Underlined} = best among single-run methods. Best method marked with $^*$.}
\label{tab:auroc_main}
\resizebox{\textwidth}{!}{%
\begin{tabular}{llcccccccc}
\toprule
Model & Method & bioasq & coqa & drop & gsm8k & nq\_open & squad & svamp & triviaqa \\
\midrule
\multirow{12}{*}{\texttt{Meta-Llama-3-8B-Instruct}} 
& \multicolumn{9}{l}{\textit{Single-run methods}} \\
\cmidrule(lr){2-10}
 & CES (unsup) & \underline{0.584 $\pm$ 0.028} & \underline{0.694 $\pm$ 0.043} & \underline{0.627 $\pm$ 0.026} & \underline{\textbf{0.685 $\pm$ 0.033}} & 0.568 $\pm$ 0.026 & 0.650 $\pm$ 0.029 & \underline{0.792 $\pm$ 0.028} & 0.719 $\pm$ 0.049 \\
 & CES (sup) & 0.584 $\pm$ 0.030 & 0.693 $\pm$ 0.046 & 0.627 $\pm$ 0.025 & 0.684 $\pm$ 0.034 & 0.567 $\pm$ 0.026 & 0.650 $\pm$ 0.030 & 0.790 $\pm$ 0.029 & 0.719 $\pm$ 0.047 \\
 & LN-Entropy & 0.582 $\pm$ 0.029 & 0.686 $\pm$ 0.042 & 0.606 $\pm$ 0.027 & 0.668 $\pm$ 0.034 & \underline{\textbf{0.574 $\pm$ 0.025}} & \underline{0.652 $\pm$ 0.028} & 0.781 $\pm$ 0.029 & \underline{0.727 $\pm$ 0.045} \\
 & Perplexity & 0.582 $\pm$ 0.029 & 0.686 $\pm$ 0.042 & 0.606 $\pm$ 0.027 & 0.668 $\pm$ 0.035 & 0.574 $\pm$ 0.025 & 0.652 $\pm$ 0.029 & 0.781 $\pm$ 0.028 & 0.727 $\pm$ 0.044 \\
\cmidrule(lr){2-10}
& \multicolumn{9}{l}{\textit{Multi-trajectory / LLM-based methods}} \\
\cmidrule(lr){2-10}
 & SE & 0.567 $\pm$ 0.030 & 0.702 $\pm$ 0.050 & 0.633 $\pm$ 0.026 & 0.525 $\pm$ 0.039 & 0.544 $\pm$ 0.026 & 0.657 $\pm$ 0.027 & 0.767 $\pm$ 0.030 & 0.703 $\pm$ 0.052 \\
 & DSE & 0.581 $\pm$ 0.029 & 0.680 $\pm$ 0.045 & 0.638 $\pm$ 0.025 & 0.589 $\pm$ 0.039 & 0.561 $\pm$ 0.027 & 0.667 $\pm$ 0.028 & 0.798 $\pm$ 0.027 & 0.723 $\pm$ 0.048 \\
 & KLE & 0.592 $\pm$ 0.029 & \textbf{0.739 $\pm$ 0.043}$^*$ & 0.652 $\pm$ 0.025 & 0.651 $\pm$ 0.038 & 0.557 $\pm$ 0.026 & 0.672 $\pm$ 0.028 & \textbf{0.806 $\pm$ 0.026}$^*$ & 0.734 $\pm$ 0.047 \\
 & KLE-Full & 0.591 $\pm$ 0.030 & 0.739 $\pm$ 0.041 & 0.650 $\pm$ 0.026 & 0.645 $\pm$ 0.037 & 0.562 $\pm$ 0.025 & 0.670 $\pm$ 0.029 & 0.800 $\pm$ 0.027 & 0.734 $\pm$ 0.048 \\
 & Eigenscore & 0.591 $\pm$ 0.027 & 0.675 $\pm$ 0.041 & 0.600 $\pm$ 0.027 & 0.647 $\pm$ 0.034 & 0.567 $\pm$ 0.026 & 0.670 $\pm$ 0.029 & 0.729 $\pm$ 0.030 & 0.744 $\pm$ 0.047 \\
 & SelfCheckGPT & 0.588 $\pm$ 0.027 & 0.640 $\pm$ 0.040 & 0.589 $\pm$ 0.027 & \textbf{0.702 $\pm$ 0.034}$^*$ & 0.563 $\pm$ 0.027 & 0.663 $\pm$ 0.028 & 0.732 $\pm$ 0.031 & 0.709 $\pm$ 0.046 \\
 & FAVA & 0.537 $\pm$ 0.028 & 0.667 $\pm$ 0.043 & 0.641 $\pm$ 0.025 & 0.594 $\pm$ 0.041 & 0.523 $\pm$ 0.026 & 0.504 $\pm$ 0.029 & 0.721 $\pm$ 0.031 & 0.659 $\pm$ 0.043 \\
 & EmbRegress & \textbf{0.601 $\pm$ 0.029}$^*$ & 0.729 $\pm$ 0.039 & \textbf{0.687 $\pm$ 0.024}$^*$ & 0.559 $\pm$ 0.038 & \textbf{0.576 $\pm$ 0.027}$^*$ & \textbf{0.678 $\pm$ 0.071}$^*$ & 0.670 $\pm$ 0.032 & \textbf{0.863 $\pm$ 0.031}$^*$ \\
\midrule
\multirow{12}{*}{\texttt{gpt-4o-mini-2024-07-18}} 
& \multicolumn{9}{l}{\textit{Single-run methods}} \\
\cmidrule(lr){2-10}
 & CES (unsup) & \underline{\textbf{0.698 $\pm$ 0.027}}$^*$ & \underline{0.661 $\pm$ 0.063} & \underline{0.621 $\pm$ 0.031} & 0.669 $\pm$ 0.040 & \underline{0.700 $\pm$ 0.023} & 0.744 $\pm$ 0.023 & 0.692 $\pm$ 0.056 & \underline{0.636 $\pm$ 0.046} \\
 & CES (sup) & 0.697 $\pm$ 0.028 & 0.661 $\pm$ 0.062 & 0.620 $\pm$ 0.030 & 0.669 $\pm$ 0.042 & 0.700 $\pm$ 0.024 & 0.744 $\pm$ 0.024 & 0.692 $\pm$ 0.055 & 0.635 $\pm$ 0.045 \\
 & LN-Entropy & 0.690 $\pm$ 0.029 & 0.659 $\pm$ 0.063 & 0.601 $\pm$ 0.029 & \underline{0.682 $\pm$ 0.040} & 0.696 $\pm$ 0.024 & \underline{\textbf{0.748 $\pm$ 0.022}}$^*$ & \underline{0.694 $\pm$ 0.053} & 0.622 $\pm$ 0.044 \\
 & Perplexity & 0.690 $\pm$ 0.027 & 0.659 $\pm$ 0.061 & 0.601 $\pm$ 0.029 & 0.682 $\pm$ 0.038 & 0.696 $\pm$ 0.024 & 0.748 $\pm$ 0.022 & 0.694 $\pm$ 0.056 & 0.622 $\pm$ 0.045 \\
\cmidrule(lr){2-10}
& \multicolumn{9}{l}{\textit{Multi-trajectory / LLM-based methods}} \\
\cmidrule(lr){2-10}
 & SE & 0.597 $\pm$ 0.033 & 0.530 $\pm$ 0.062 & 0.596 $\pm$ 0.028 & 0.536 $\pm$ 0.037 & 0.645 $\pm$ 0.025 & 0.602 $\pm$ 0.029 & 0.560 $\pm$ 0.067 & 0.555 $\pm$ 0.043 \\
 & DSE & 0.636 $\pm$ 0.031 & 0.597 $\pm$ 0.040 & 0.611 $\pm$ 0.022 & 0.554 $\pm$ 0.021 & 0.672 $\pm$ 0.022 & 0.665 $\pm$ 0.023 & 0.675 $\pm$ 0.045 & 0.580 $\pm$ 0.032 \\
 & KLE & 0.643 $\pm$ 0.034 & 0.651 $\pm$ 0.054 & \textbf{0.660 $\pm$ 0.026}$^*$ & 0.576 $\pm$ 0.029 & 0.671 $\pm$ 0.024 & 0.670 $\pm$ 0.027 & 0.685 $\pm$ 0.053 & 0.566 $\pm$ 0.046 \\
 & KLE-Full & 0.644 $\pm$ 0.034 & 0.646 $\pm$ 0.052 & 0.656 $\pm$ 0.026 & 0.576 $\pm$ 0.030 & 0.672 $\pm$ 0.025 & 0.664 $\pm$ 0.026 & 0.685 $\pm$ 0.056 & 0.564 $\pm$ 0.044 \\
 & Eigenscore & 0.581 $\pm$ 0.035 & 0.624 $\pm$ 0.063 & 0.601 $\pm$ 0.030 & 0.579 $\pm$ 0.036 & 0.638 $\pm$ 0.025 & 0.617 $\pm$ 0.029 & 0.640 $\pm$ 0.050 & 0.565 $\pm$ 0.048 \\
 & SelfCheckGPT & 0.589 $\pm$ 0.031 & 0.693 $\pm$ 0.054 & 0.552 $\pm$ 0.027 & 0.475 $\pm$ 0.030 & 0.584 $\pm$ 0.027 & 0.669 $\pm$ 0.025 & 0.590 $\pm$ 0.051 & 0.554 $\pm$ 0.040 \\
 & FAVA & 0.545 $\pm$ 0.032 & 0.647 $\pm$ 0.048 & 0.618 $\pm$ 0.029 & \textbf{0.766 $\pm$ 0.035}$^*$ & 0.652 $\pm$ 0.024 & 0.521 $\pm$ 0.027 & 0.678 $\pm$ 0.046 & 0.609 $\pm$ 0.041 \\
 & EmbRegress & 0.529 $\pm$ 0.028 & \textbf{0.792 $\pm$ 0.041}$^*$ & 0.606 $\pm$ 0.028 & 0.456 $\pm$ 0.036 & \textbf{0.733 $\pm$ 0.022}$^*$ & 0.553 $\pm$ 0.066 & \textbf{0.784 $\pm$ 0.037}$^*$ & \textbf{0.720 $\pm$ 0.036}$^*$ \\
\bottomrule
\end{tabular}}
\end{table}
\end{center}
\section{Related Work}
\label{sec:related_work}

\textbf{Hallucination Detection in LLMs.}
\cite{kalai2025language} show that hallucination
classification is computationally easier than generation,
motivating post-hoc detection.
Retrieval-based methods verify claims against external
evidence \cite{wei2024long, shah2025validation,
obeso2025real} but incur latency and depend on retrieval
coverage.
Attention-based detectors use self-attention structure as a
grounding proxy \cite{chuang2024lookback,
ogasa2025hallucination} but require access to attention
matrices or supervised span-level labels.
\cite{azaria2023internal} and \cite{niu2025robust} exploit
model internals or learned detection policies;
\cite{moslonka2026learned} target RAG-specific
hallucinations.
CES requires only logit access.

\textbf{Uncertainty quantification.}
Single-pass methods, such as perplexity \cite{jelinek1977perplexity},
length-normalised entropy \cite{malinin2020uncertainty}, and
generation length \cite{janiak2025illusion}, reduce the
entropy sequence to a scalar, discarding distributional
information and requiring task-specific thresholds.
Multi-sample methods achieve stronger performance at higher
cost: Semantic Entropy \cite{farquhar2024detecting} and
KLE \cite{nikitin2024kernel} cluster $K$ generations by
meaning; EigenScore \cite{chen2023going} and
SAR \cite{duan2024shifting} measure embedding spread or
relevance-weighted uncertainty;
SelfCheckGPT \cite{manakul2023selfcheckgpt} and
$\mathbb{P}(\text{True})$ \cite{kadavath2022language} probe
consistency or self-evaluation.
All require $K \geq 2$ forward passes.
INSIDE \cite{chen2024inside} uses hidden-state probes but
requires model internals.
Surveys by \cite{xia2025survey} and
\cite{janiak2025illusion} provide comprehensive overviews.
\cite{pillutla2021mauve} uses distributional divergence
for corpus-level generation evaluation; conformal prediction
methods \cite{quach2023conformal, kumar2023conformal} provide
coverage guarantees on prediction sets.
Our closest work is \cite{du2024haloscope}, which formalizes hallucinations as sampled from a separate distribution; however, they leverage a Huber contamination model, whereas our judge formalism empowers a statistical analysis of our proposed test.  
To our knowledge, no prior work formalises hallucination
detection as a hypothesis test over entropy distributions.
Extended related works are in Appendix \ref{app:extended_related_work}.

\section{Conclusion}
\label{sec:conclusion}

We have introduced a framework that casts hallucination
detection as a statistical hypothesis test over token-level
entropy distributions.
The central finding is that the \emph{distribution} of token
entropies, beydon mean effects, constitutes a fingerprint
of hallucination, carrying independent signal in its shape
and tails beyond what length-normalised entropy or perplexity
capture.
The resulting Calibrated Entropy Score (CES) combines mean
and tail signals through a calibrated reference CDF,
operates efficiently, from a single forward
pass with black-box logit access, and without need of further LLM queries. 

The framework cleanly separates the \emph{definition} of
hallucination, delegated to an oracle during offline
calibration, from its \emph{detection}, making it agnostic
to the choice of model, task, or hallucination taxonomy.
We have provided a comprehensive theoretical foundation:
finite-sample calibration guarantees via a random-length DKW
inequality (Theorem~\ref{thm:random_length_DKW}), exponential power
against mean-shift and tail alternatives
(Theorem~\ref{thm:ces_power}), calibration perturbation bounds
(Proposition~\ref{prop:ces_calibration}), and contamination robustness
for unsupervised deployment.
Empirically, across 8 QA benchmarks and 10 generator models
spanning open-source and proprietary families, CES
consistently matches or exceeds perplexity and
length-normalised entropy in AUROC while providing the formal
error guarantees these heuristics lack.
Broader societal applications include enabling lightweight, real-time detection of hallucinations with formal error guarantees. CES can improve the trustworthiness of LLM-deployed systems in high-stakes domains such as healthcare, finance, and legal reasoning, reducing the risk of users acting on fabricated information; however, over-reliance on any automated detector could engender false confidence, and adversarial actors might exploit knowledge of entropy-based detection to craft hallucinations that evade distributional tests.

\textbf{Limitations.}
When the model hallucinates with high probability, the entropy
trace is indistinguishable from faithful generation; if
$d_{KS}(F_0, F_1) = 0$, no test has power (for example if the model was trained on incorrect data).
CES power also degrades for short generations ($m < 10$ tokens),
a fundamental limitation of any distributional test on few
observations.
Our guarantees rest on the conditional i.i.d.\ assumption;
future work should model sequential structure explicitly.
CES requires access to next-token or top-$k$ logit
probabilities, unavailable from some API providers, and its
quality is bounded by the calibration oracle.
Finally, our experiments use greedy decoding on short-answer
QA tasks; extending to stochastic decoding and long-form
generation remains open.

\section*{Disclaimer}

This paper was prepared for informational purposes by the Global Technology Applied Research center of JPMorgan Chase \& Co. This paper is not a merchandisable/sellable product of the Research Department of JPMorgan Chase \& Co. or its affiliates. Neither JPMorgan Chase \& Co. nor any of its affiliates makes any explicit or implied representation or warranty and none of them accept any liability in connection with this paper, including, without limitation, with respect to the completeness, accuracy, or reliability of the information contained herein and the potential legal, compliance, tax, or accounting effects thereof. This document is not intended as investment research or investment advice, or as a recommendation, offer, or solicitation for the purchase or sale of any security, financial instrument, financial product or service, or to be used in any way for evaluating the merits of participating in any transaction.

\bibliographystyle{plainnat}
\bibliography{references}

\appendix
\section{Theory for Calibration \label{app:cal_theory}}

In this section, we expand on the theory supporting our calibration algorithm, and prove the corresponding theorems stated in the main text.
We begin by setting up the problem. We consider the vocabulary to be a finite non-empty set $\mathcal{V}$, equipped with the discrete $\sigma$-algebra. $\mathcal{V}_K^\star = \sqcup_{i = 1}^K \mathcal{V}^i$ and $\mathcal{V}_M^\star = \sqcup_{i = 1}^M \mathcal{V}^i$ denote the sets of input contexts and the outputs generated by the model, respectively. In order to turn $\mathcal{V}_K^\star$ and $\mathcal{V}_M^\star$ into measurable spaces, we use the standard construction of $\sigma$-algebra on disjoint union of measurable spaces, described as follows. Given measurable spaces $(\Omega_i, \Sigma_i)$ for $i \in [n]$, the disjoint union $\Omega = \sqcup_{i = 1}^n \Omega_i$ can be equipped with the $\sigma$-algebra $\Sigma$ with $A \in \Sigma$ iff $A \cap \Omega_i \in \Sigma_i$ for all $i \in [n]$. We identify $\{i\} \times \Omega_i$ with $\Omega_i$ for brevity of notation when dealing with disjoint unions. This construction is also used to define a $\sigma$-algebra on $\mathcal{E} = ([0, \log(d)])_M^\star$ and $\mathcal{E} \sqcup \{\bot\}$, where $[0, \log(d)]^m$ is equipped with the Borel $\sigma$-algebra for $m \in [M]$.

The oracle $\mathcal{O}: \Vin \times \Vout \to \{0, 1\}$ is taken to be a measurable function. We show below that the filter function $\mathfrak{f}_{\mathcal{O}}$ is also measurable.
\begin{proposition}
    The function $\mathfrak{f}_{\mathcal{O}}\colon \Vin \times \Vout \times \mathcal{E} \to \mathcal{E} \sqcup \{\bot\}$ defined in Equation~\ref{eqn:oracle_filter} is measurable.
\end{proposition}
\begin{proof}
    Denote $\mathcal{A}_0 = \{(\xin, \xout) \in \Vin \times \Vout \mid \mathcal{O}(\xin, \xout) = 0\}$ and $\mathcal{A}_1 = \{(\xin, \xout) \in \Vin \times \Vout \mid \mathcal{O}(\xin, \xout) = 1\}$. Both $\mathcal{A}_0$ and $\mathcal{A}_1$ are measurable subsets of $\Vin \times \Vout$ since $\mathcal{O}$ is measurable.
    Let $A \subseteq \mathcal{E} \sqcup \{\bot\}$ be a measurable set. Observe that
    \begin{equation}
        \mathfrak{f}_{\mathcal{O}}^{-1}(A) = \begin{cases} (A \cap \mathcal{E}) \times A_0,                             & \bot \notin A \\
                                                           (A \cap \mathcal{E}) \times A_0 \cup \mathcal{E} \times A_1, & \bot \in A.
                                            \end{cases}
    \end{equation}
    Since $A_0$, $A_1$ are measurable subsets of $\Vin \times \Vout$, $A \cap \mathcal{E}$ is a measurable subset of $\mathcal{E}$, it follows that $\mathfrak{f}^{-1}(A)$ is a measurable subset of $\Vin \times \Vout \times \mathcal{E}$.
\end{proof}
Since $\mathfrak{f}_{\mathcal{O}}$ is measurable, given the joint distribution $\mathbb{P}$ on $\Vin \times \Vout \times \mathcal{E}$, the pushforward $\mathbb{P}_{\mathcal{O}}(\cdot) = \mathbb{P}(\mathfrak{f}_{\mathcal{O}}^{-1}(\cdot))$ is a well-defined distribution on $\mathcal{E} \sqcup \{\bot\}$. From this, we obtain the distribution $\overline{\mathbb{P}}_0$ by restricting to $\mathcal{E}$, as defined in Equation~\ref{eqn:proxy_nonhal_dist}. As noted in Section~\ref{sec:methodology}, we make the following simplifying assumption.
\begin{assumption}
    \label{asmp:iid_m}
    Conditioned on the length of the generated output, the distribution of token entropy for non-hallucinated instances is iid according to some distribution $\mathbb{P}_0$.
    Mathematically, this means that the distribution $\overline{\mathbb{P}}_0$ defined in Equation~\ref{eqn:proxy_nonhal_dist} satisfies
    \begin{equation}
        \overline{\mathbb{P}}_0(A | T = m) = \mathbb{P}_0^m(A \cap \mathbb{R}^m) \qquad \forall m \in [M],\ \forall A \subseteq \mathbb{R}_M^\star \textnormal{ measurable}, \label{eqn:iid_m_asmp}
    \end{equation}
    where $T$ is the random variable that outputs the generation length.
\end{assumption}

Our first goal is to show that under this assumption, Algorithm~\ref{alg:calibration} gives an approximation of the distribution $\mathbb{P}_0$ with sufficiently many calibration samples.
The first ingredient in showing this is a technical result that the product distribution conditioned on a Cartesian product of measurable sets $E_1, \dotsc, E_N$ is the same as the product of the respective conditional distributions.
\begin{lemma}
    \label{lem:conditional_prod_dist}
    Let $(\Omega, \Sigma, \mathbb{Q})$ be a probability space and let $E_1, \dotsc, E_N \in \Sigma$ be measurable sets with non-zero probability.
    For $i \in [N]$, let $\mathbb{Q}_i(\cdot) = \mathbb{Q}(\cdot | E_i)$ denote the distribution $\mathbb{Q}$ conditioned on the event $E_i$.
    Then, for all measurable sets $A \subseteq \Omega^N$, we have
    \begin{equation}
        \mathbb{Q}^N(A | E_1 \times \dotsm \times E_N) = \left(\prod_{i = 1}^N \mathbb{Q}_i\right)(A).
    \end{equation}
\end{lemma}
\begin{proof}
    We follow a strategy used for showing the uniqueness of extensions of measures.
    To that end, recall that a $\Pi$-system on $\Omega^N$ is a non-empty collection of subsets of $\Omega^N$ that is closed under intersections.
    A $\lambda$-system on $\Omega^N$ is a non-empty collection of subsets of $\Omega^N$ that is closed under taking complements and union of disjoint subsets.
    For brevity, define $E = E_1 \times \dotsm \times E_N$ and $\overline{\mathbb{Q}} = \prod_{i = 1}^N \mathbb{Q}_i$.
    First, note that $\mathcal{I} = \{A_1 \times \dotsm \times A_N \mid A_1, \dotsc, A_N \in \Sigma\}$ is a $\Pi$-system on $\Omega^N$ that generates the product $\sigma$-algebra $\Sigma^{\otimes N}$ (by definition of product $\sigma$-algebra). Next, define the set $\mathcal{J} = \{A \in \Sigma^{\otimes N} \mid \mathbb{Q}^N(A | E) = \overline{\mathbb{Q}}(A)\}$ to be the measurable subsets $\Omega^N$ on which $\mathbb{Q}^N(\cdot | E)$ and $\overline{\mathbb{Q}}$ agree. Observe that if $A \in \mathcal{J}$, then $A^c \in \mathcal{J}$ because $\mathbb{Q}^N(A^c | E) = 1 - \mathbb{Q}^N(A | E) = 1 - \overline{\mathbb{Q}}(A) = \overline{\mathbb{Q}}(A)$. Similarly, if $(B_n)$ is a sequence of disjoint sets in $\mathcal{J}$, then $\mathbb{Q}^N(\cup_{i = 1}^\infty B_i | E) = \sum_{i = 1}^\infty \mathbb{Q}^N(B_i | E) = \sum_{i = 1}^\infty \overline{\mathbb{Q}}(B_i) = \overline{\mathbb{Q}}(\cup_{i = 1}^\infty B_i)$. Therefore, $\mathcal{J}$ is a $\lambda$-system.

    Next, for all $A = A_1 \times \dotsm \times A_N \in \mathcal{I}$, we have
    \begin{equation}
        \mathbb{Q}^N(A | E) = \frac{\mathbb{Q}^N(A \cap E)}{\mathbb{Q}^N(E)} = \prod_{i = 1}^N \frac{\mathbb{Q}(A_i \cap E_i)}{\mathbb{Q}(E_i)} = \overline{\mathbb{Q}}(A).
    \end{equation}
    By Dynkin $\pi$-$\lambda$ theorem, we have $\Sigma^{\otimes N} = \sigma(\mathcal{I}) \subseteq \mathcal{J}$. However, $\mathcal{J} \subseteq \Sigma^{\otimes N}$ by definition, and therefore, $\Sigma^{\otimes N} = \mathcal{J}$. It follows that on all measurable subsets of $\Omega^N$, $\mathbb{Q}^N(\cdot | E)$ and $\overline{\mathbb{Q}}$ agree, completing the proof.
\end{proof}

The next technical result we need is the Dvoretzky-Kiefer-Wolfowitz (DKW) inequality, which gives a confidence bound for the deviation of the empirical CDF from the true CDF.
\begin{theorem}[DKW Inequality] \label{thm:DKW}
    Let $X_1,...,X_R$ be iid, real-valued random variables with CDF $F_0$. Let $\widehat{F}_0(x) = \sum_{i = 1}^R \bm{1}_{(-\infty, z]}(X_i) / R$ denote the empirical CDF.
    Then for all $\epsilon > 0$,
    \begin{equation}
        \mathbb{P}^R\left(\sup_{x \in \mathbb{R}} |\hat{F}_0(x) - F_0(x)| > \epsilon \right) \leq 2 \exp(-2 R \epsilon^2).
    \end{equation}
\end{theorem}
The reason we cannot directly use the DKW inequality in our study is because the number $R$ or random variable in Theorem~\ref{thm:DKW}, which corresponds in our case to the length of the generated output, is itself random. We, therefore, derive a random-length DKW inequality.

\begin{theorem}[Random-length DKW inequality]
    \label{thm:random_length_DKW}
    Let $(\Omega, \Sigma, \mathbb{P})$ be a probability space, and consider a random variable $\bm{X}\colon \Omega \to \mathbb{R}_M^\star$ with length $T = \textnormal{dim}(\bm{X})$. Suppose that there is some distribution $\mathbb{P}_0$ on $\mathbb{R}$ such that for all $m \in [M]$, the distribution $\mathbb{P}$ satisfies
    \begin{equation}
        \mathbb{P}(\bm{X} \in A| T = m) = \mathbb{P}_0^m(A \cap \mathbb{R}^m) \label{eqn:gen_iid_m_asmp}
    \end{equation}
    for all measurable subsets $A \subseteq \mathbb{R}_M^\star$.
    Then, if $\bm{X}_1, \dotsc, \bm{X}_L$ are iid copies of $\bm{X}$, $T_i = \textnormal{dim}(\bm{X}_i)$ for $i \in [L]$, $F_0$ is the CDF of $\mathbb{P}_0$, and
    \begin{equation}
        \widehat{F}_0(x) = \frac{1}{\sum_{i = 1}^L \textnormal{dim}(\bm{X}_i)} \sum_{i = 1}^L \sum_{k = 1}^{\textnormal{dim}(\bm{X}_i)} \bm{1}_{(-\infty, x]}(X_{i, k})
    \end{equation}
    is the empirical CDF, we have for all $m_1, \dotsc, m_L \in [M]$ and all $\epsilon > 0$,
    \begin{equation}
        \mathbb{P}^L\left(\sup_{x \in \mathbb{R}} |\widehat{F}_0(x) - F_0(x)| \geq \epsilon | T_1 = m_1, \dotsc, T_L = m_l\right) \leq 2 \exp\left(-2 \sum_{i = 1}^L m_i \epsilon^2\right), \label{eqn:condtional_random_length_DKW}
    \end{equation}
    and consequently, 
    \begin{equation}
        \mathbb{P}^L\left(\sup_{x \in \mathbb{R}} |\widehat{F}_0(x) - F_0(x)| \geq \epsilon\right) \leq 2 \left(\mathbb{E}\left[\exp(-2 \textnormal{dim}(\bm{X}) \epsilon^2)\right]\right)^L \leq 2 \exp(-2 L \epsilon^2). \label{eqn:random_length_DKW}
    \end{equation}
\end{theorem}
\begin{proof}
    Since $\bm{X}_1, \dotsc, \bm{X}_L$ are iid copies of $\bm{X}$, $T_1, \dotsc, T_L$ are iid copies of $T$, where $T_i = \textnormal{dim}(\bm{X}_i)$ for $i \in [L]$.
    Fix $m_1, \dotsc, m_L \in [M]$, and consider $E_i = \{\omega \in \Omega \mid T(\omega) = m_i\}$.
    By independence, $\{\omega \in \Omega^L \mid T_1(\omega) = m_1, \dotsc, T_L(\omega) = m_L\} = E_1 \times \dotsm \times E_L$.
    For now, consider only those $m_1, \dotsc, m_L$ for which $\mathbb{P}(E_1), \dotsc, \mathbb{P}(E_L) > 0$. In the proof below, we will see that it suffices to consider such $m_1, \dotsc, m_L$.

    Fix a measurable set $B \subseteq (\mathbb{R}_M^\star)^K$. For $i \in [L]$, denote $\mathbb{P}_i(C) = \mathbb{P}(C | E_i) = \mathbb{P}(C | T = m_i)$ for measurable $C \subseteq \Sigma$. Then, by Lemma~\ref{lem:conditional_prod_dist}, we have $\mathbb{P}^L((\bm{X}_1, \dotsc, \bm{X}_L) \in B | T_1 = m_1, \dotsc, T_L = m_L) = (\prod_{i = 1}^L \mathbb{P}_i)((\bm{X}_1, \dotsc, \bm{X}_L) \in B)$. However, $(\prod_{i = 1}^L \mathbb{P}_i)((\bm{X}_1, \dotsc, \bm{X}_L) \in B) = (\prod_{i = 1}^L \overline{\mathbb{P}}_i)(B)$, where $\overline{\mathbb{P}}_i$ is the pushforward of $\mathbb{P}_i$ under $\bm{X}$. By assumption (see Eq.~\eqref{eqn:gen_iid_m_asmp}), we have $\overline{\mathbb{P}}_i = \mathbb{P}_0^{m_i}$ for all $i \in [L]$. Therefore, $(\prod_{i = 1}^L \mathbb{P}_i)((\bm{X}_1, \dotsc, \bm{X}_L) \in B) = \mathbb{P}_0^N(B)$, where $N = \sum_{i = 1}^L m_i$. In particular, for
    $B = \{(\bm{x}_1, \dotsc, \bm{x}_L) \in \mathbb{R}^{m_1} \times \dotsm \mathbb{R}^{m_L} \mid \sup_{x \in \mathbb{R}} |\frac{1}{N} \sum_{i = 1}^L \sum_{k = 1}^{m_i} \bm{1}_{(-\infty, x]}(x_{i, k}) - F_0(x)| \geq \epsilon\}$, we have by DKW inequality (Theorem~\ref{thm:DKW}), $\mathbb{P}_0^N(B) \leq 2 \exp(-2 N \epsilon^2)$.

    It follows that $\mathbb{P}^L(\sup_{x \in \mathbb{R}} |\widehat{F}_0(x) - F_0(x)| \geq \epsilon | T_1 = m_1, \dotsc, T_L = m_L) \leq 2 \exp(-2 N \epsilon^2)$. Since this holds for all $m_1, \dotsc, m_L$ for which $\mathbb{P}(T = m_1), \dotsc, \mathbb{P}(T = m_L) > 0$, and
    \begin{align*}
        \mathbb{P}^L(\sup_{x \in \mathbb{R}} |\widehat{F}_0(x) - F_0(x)| \geq \epsilon) &= \sum_{i = 1}^L \sum_{m_i = 1}^L \mathbb{P}(T = m_1) \dotsc \mathbb{P}(T = m_L)\\
            &\qquad \mathbb{P}^L(\sup_{x \in \mathbb{R}} |\widehat{F}_0(x) - F_0(x)| \geq \epsilon | T_1 = m_1, \dotsc, T_L = m_L) \\
            &\leq 2 \left(\mathbb{E}\left[\exp\left(-2 T \epsilon^2\right)\right]\right)^L,
    \end{align*}
    where we used $\exp(-2 N \epsilon^2) = \prod_{i = 1}^L \exp(-2 m_i \epsilon^2)$.
\end{proof}

Observe that the bound $2 (\mathbb{E}[\exp(-2 \textnormal{dim}(\bm{X}) \epsilon^2)])^L$ in Theorem~\ref{thm:random_length_DKW} depends on the \textit{average} of $\exp(-2 \textnormal{dim}(\bm{X}) \epsilon^2)$. Crucially, this average is only over the lengths of the generation, and the distribution of the generated length for a given model can usually be empirically estimated. Intuitively, one would expect such an average to appear in the bound in Eq.~\eqref{eqn:random_length_DKW}, because how many data points one can use to construct the empirical CDF $\widehat{F}_0$ depends on the dimensions/lengths of the random vectors $\bm{X}_1, \dotsc, \bm{X}_L$. If the length of the random vector $\bm{X}$ is always bounded below by $\ell \geq 1$, then we must have $2 (\mathbb{E}[\exp(-2 \textnormal{dim}(\bm{X}) \epsilon^2)])^L \leq 2 \exp(-2 L \ell \epsilon^2)$. Thus, regardless of the distribution of the length of $\bm{X}$, as long as $L$ is large enough, we can always get a good concentration. Theorem~\ref{thm:random_length_DKW} further strengthens this by noting that if the length of $\bm{X}$ is large with high probability, then we get an even better concentration.

Using Theorem~\ref{thm:random_length_DKW}, we show below that the calibration in Algorithm~\ref{alg:calibration} well approximates the distribution $\mathbb{P}_0$.
\begin{corollary}
    \label{cor:cdfcalalg_converge}
    Let $\widehat{F}_0$ is the empirical CDF constructed in Algorithm~\ref{alg:calibration}, and denote $\mathcal{D}$ to be the set of token entropy sequences corresponding to the non-hallucinated instances (according to the oracle). If for a fixed $L$, we sample $L$ non-hallucinated instances, then for all $\epsilon > 0$, we have
    \begin{equation}
        \overline{\mathbb{P}}_0^L\left(\sup_{x \in \mathbb{R}} |\widehat{F}_0(x) - F_0(x)| \geq \epsilon\right) \leq 2 (\mathbb{E}[\exp(-2 \textnormal{dim}(\bm{h}) \epsilon^2)])^L \leq 2 \exp(-2 L \epsilon^2). \label{eqn:tail_bound_cal}
    \end{equation}
    In particular, given an error $\epsilon > 0$ and a confidence confidence level $1 - \delta \in (0, 1)$, one needs at most
    \begin{equation}
        \left\lceil \frac{\log(2/\delta)}{2 \epsilon^2} \right\rceil \label{eqn:sample_complexity_cal}
    \end{equation}
    non-hallucinated instances to construct an empirical CDF $\widehat{F}_0$ with at most $\epsilon$ KS distance from $F_0$.
\end{corollary}
\begin{proof}
    In Algorithm~\ref{alg:calibration}, by construction, every token entropy sequence $\bm{h} \in \mathcal{D}$ is a non-hallucinated instance, sampled independently from $\overline{\mathbb{P}}_0$. By Assumption~\ref{asmp:iid_m}, $\overline{\mathbb{P}}_0$ satisfies Eq.~\eqref{eqn:gen_iid_m_asmp}. Suppose that we have $L = |\mathcal{D}|$ such entropy sequences.
    Then, taking $\bm{X} = \bm{h}$ and $\mathbb{P} = \overline{\mathbb{P}}_0$ in Theorem~\ref{thm:random_length_DKW}, we get Eq.~\eqref{eqn:tail_bound_cal}.
    The sample complexity in Eq.~\eqref{eqn:sample_complexity_cal} can be obtained by setting $2 \exp(-2 L \epsilon^2) = \delta$ and solving for $L$.
\end{proof}

Note that the number of non-hallucinated instances required for approximating the CDF given in Eq.~\eqref{eqn:sample_complexity_cal} corresponds to the worst-case bound of $2 \exp(-2 L \epsilon^2)$ in Eq.~\eqref{eqn:tail_bound_cal}. In practice, the number of non-hallucinated instances required can be significantly smaller than Eq.~\eqref{eqn:sample_complexity_cal}, based on the tighter bound $2 (\mathbb{E}[\exp(-2 \textnormal{dim}(\bm{h}) \epsilon^2)])^L$ in Eq.~\eqref{eqn:tail_bound_cal}.

\section{Theory for Calibrated Entropy Score}
\label{app:ces_theory}

In this section, we study two main types of results for the Calibrated Entropy Score (CES).
First, in Section~\ref{app:ces_cal}, we study how well one can approximate the ``ideal" CES (computed using the true CDF) with the empirical CES (computed using the empirical CDF).
We show that when the empirical CDF is close to the true CDF in the KS distance, the empirical CES is also close to the ideal CES.
Subsequently, in Section~\ref{app:ces_power}, we analyze the statistical power of CES for hypothesis testing.
Herein, we show that for appropriately chosen threshold, the false positive rate as well as the false negative rate decays to zero exponentially quickly,
demonstrating that CES is a good measure for detecting hallucination.


\subsection{Calibration guarantees for CES}
\label{app:ces_cal}

\begin{proposition}[Calibration guarantee for CES]
\label{prop:ces_calibration}
Let $F_0$ be the true non-hallucinated entropy CDF and let $\widehat{F}_0$ be the empirical CDF from Algorithm~\ref{alg:calibration}
satisfying $d_{KS}(\widehat{F}_0, F_0) \leq \epsilon$ with probability at least $1 - \delta$.
Define the oracle and calibrated scores as
$$
  \mathrm{CES}^*
  = \sqrt{F_0(\bar{h}) \cdot F_0(h_{\max})},
  \qquad
  \widehat{\mathrm{CES}}
  = \sqrt{\widehat{F}_0(\bar{h}) \cdot \widehat{F}_0(h_{\max})}.
$$
given the observed average $\bar{h}$ and maximum $h_{\max}$ entropies.
Then, with probability at least $1 - \delta$ over the calibrated data, we have
\begin{equation}
    \bigl|\widehat{\mathrm{CES}} - \mathrm{CES}^*\bigr| \;\leq\; \sqrt{2 \epsilon}.
\end{equation}
\end{proposition}
\begin{proof}
Fix $x, y \in \mathbb{R}$ and write $\hat{a} = \widehat{F}_0(x)$, $a = F_0(x)$, $\hat{b} = \widehat{F}_0(y)$, $b = F_0(y)$.

Consider the event $d_{KS}(\widehat{F}_0, F_0) \leq \epsilon$, which holds with probability $1 - \delta$ by assumption.
On this event, we have both  $|\hat{a} - a| \leq \epsilon$ and $|\hat{b} - b| \leq \epsilon$.
Then,
\begin{align*}
  |\hat{a}\hat{b} - ab|
  &= |\hat{a}(\hat{b} - b) + b(\hat{a} - a)| \\
  &\leq |\hat{a}|\,|\hat{b} - b| + |b|\,|\hat{a} - a| \\
  &\leq 1 \cdot \epsilon + 1 \cdot \epsilon
  = 2\epsilon,
\end{align*}
where we used $\hat{a}, b \in [0,1]$.
Applying $|\sqrt{u} - \sqrt{v}| \leq \sqrt{|u - v|}$ for
$u, v \geq 0$ with $u = \hat{a}\hat{b}$ and $v = ab$:
$$
  \bigl|\sqrt{\hat{a}\hat{b}} - \sqrt{ab}\bigr|
  \leq \sqrt{|\hat{a}\hat{b} - ab|}
  \leq \sqrt{2 \epsilon}.
$$
Since this holds for all $x, y$, we have
$$
  \sup_{x, y \in \mathbb{R}} \bigl|\sqrt{\widehat{F}_0(x) \widehat{F}_0(y)} - \sqrt{F_0(x) F_0(y)}\bigr|
  \leq \sqrt{2 \epsilon}.
$$
Taking $x = \bar{h}$ and $y = h_{\max}$ completes the proof.
\end{proof}

When the CDF values are bounded away from zero, we obtain a
tighter bound.

\begin{proposition}[Refined calibration guarantee]
\label{prop:ces_calibration_refined}
Under the conditions of Proposition~\ref{prop:ces_calibration}, if
additionally
$\min\{F_0(\bar{h}),\, F_0(h_{\max})\} \geq \eta > 0$
and $\epsilon \leq \eta$, then with probability at least
$1 - \alpha$,
$$
  \bigl|\widehat{\mathrm{CES}} - \mathrm{CES}^*\bigr|
  \;\leq\; \frac{2 \epsilon}{\eta} 
$$
\end{proposition}
\begin{proof}
We use the notations in the proof of Proposition~\ref{prop:ces_calibration}.
On the event $\sup_z |\widehat{F}_0(z) - F_0(z)| \leq \epsilon$,
we have $\hat{a} \geq a - \epsilon \geq \eta - \epsilon > 0$
and similarly $\hat{b} \geq \eta - \epsilon > 0$.
Using the identity
$|\sqrt{u} - \sqrt{v}| = |u - v| / (\sqrt{u} + \sqrt{v})$
with $u = \hat{a}\hat{b}$ and $v = ab$, we have
$$
  \bigl|\widehat{\mathrm{CES}} - \mathrm{CES}^*\bigr|
  = \frac{|\hat{a}\hat{b} - ab|}
         {\sqrt{\hat{a}\hat{b}} + \sqrt{ab}}.
$$
The numerator is at most $2\epsilon$ by the same argument as
in Proposition~\ref{prop:ces_calibration}.
For the denominator, $\sqrt{ab} \geq \eta$ since $a, b \geq \eta$,
and $\sqrt{\hat{a}\hat{b}} \geq \sqrt{(\eta - \epsilon)^2}
= \eta - \epsilon \geq 0$.
Therefore
$$
  \bigl|\widehat{\mathrm{CES}} - \mathrm{CES}^*\bigr|
  \leq \frac{2\epsilon}{\eta - \epsilon + \eta}
  \leq \frac{2\epsilon}{\eta},
$$
where the last step uses $\epsilon \leq \eta$.
\end{proof}


\begin{remark}
The general bound
$\sqrt{2\epsilon}$ is larger than the error  $\epsilon$ for the KS statistic itself.
However, under $\mathcal{H}_0$ the CDF values are typically
bounded away from zero.  
For instance, if $\bar{h}$ falls near
the median of $F_0$, then $\eta \approx 0.5$ and the refined
bound gives $\mathcal{O}(\epsilon)$.
The worst case $\eta \to 0$ corresponds to generations with
unusually low entropy, which is precisely the regime where CES
would correctly identify the generation as atypical, so the
calibration error is irrelevant for the detection decision.
\end{remark}

\subsection{Power Analysis for CES}
\label{app:ces_power}


We establish that CES has non-trivial power against any alternative $F_1$ that shifts the mean entropy or the max entropy upward.
The analysis proceeds by studying the the mean and the max components of CES separately and then combining them to derive results on the true and false positive rates.

\begin{lemma}[Concentration of the mean component]
\label{lem:mean_power}
Let $\mu_0$ denote the mean of $F_0$ and $\mu_1$ denote the mean of $F_1$. Then, the following statements hold.
\begin{enumerate}
    \item[(a)] Suppose that $h^{(1)}, \dotsc, h^{(m)} \overset{\mathrm{iid}}{\sim} F_1$. Then, for any $\mu \in (-\infty, \mu_1]$, we have
    \begin{equation}
        \mathbb{P}_{F_1}^m\left(F_0(\bar{h}) \geq F_0(\mu)\right) \geq 1 - \exp\left(\frac{-2 m (\mu_1 - \mu)^2}{(\log d)^2}\right).
    \end{equation}

    \item[(b)] Suppose that $h^{(1)}, \dotsc, h^{(m)} \overset{\mathrm{iid}}{\sim} F_0$. Then, for any $\mu \in [\mu_0, \infty)$, we have
    \begin{equation}
        \mathbb{P}_{F_0}^m\left(F_0(\bar{h}) > F_0(\mu)\right) \leq \exp\left(\frac{-2 m (\mu - \mu_0)^2}{(\log d)^2}\right).
    \end{equation}

    \item[(c)] If $F_0$ is continuous at $\mu_1$, then $F_0(\bar{h}) \xrightarrow{\ m \to \infty\ } F_0(\mu_1)$ $F_1$-almost surely.
    Similarly, if $F_0$ is continuous at $\mu_0$, then $F_0(\bar{h}) \xrightarrow{\ m \to \infty\ } F_0(\mu_0)$ $F_0$-almost surely.
\end{enumerate}
\end{lemma}
\begin{proof}
(a) The random variables $h^{(1)}, \dotsc, h^{(m)}$ are bounded between $0$ and $\log(d)$. Since $F_0$ is a monotonically increasing function (as it is a CDF), $F_0(\bar{h}) < F_0(\mu)$ implies $\bar{h} < \mu$. Thus, $\mathbb{P}_{F_1}^m(F_0(\bar{h}) < F_0(\mu)) \leq \mathbb{P}_{F_1}^m(\bar{h} < \mu) = \mathbb{P}_{F_1}^m(\bar{h} - \mu_1 \leq -(\mu_1 - \mu)) \leq \exp(-2 m (\mu_1 - \mu)^2 / \log^2(d))$ by (one-sided) Hoeffding's inequality. Thus, we obtain $\mathbb{P}_{F_1}^m(F_0(\bar{h}) \geq F_0(\mu)) \geq 1 - \exp(-2 m (\mu_1 - \mu)^2 / \log^2(d))$.

(b) Since $F_0$ is a monotonically increasing function, $F_0(\bar{h}) > F_0(\mu)$ implies $\bar{h} > \mu$. Then, by (one-sided) Hoeffding's inequality, we have $\mathbb{P}_{F_0}^m(F_0(\bar{h}) > F_0(\mu)) \leq \mathbb{P}_{F_0}^m(\bar{h} > \mu) = \mathbb{P}_{F_0}^m(\bar{h} - \mu_0 > (\mu - \mu_0)) \leq \exp(-2 m (\mu_1 - \mu)^2 / \log^2(d))$.

(c) By strong law of large numbers, under $F_1$, we have $\bar{h} \to \mu_1$ $F_1$-almost surely. Since $F_0$ is continuous at $\mu$, $F_0(\bar{h}) \to F_0(\mu_1)$ $F_1$-almost surely. Similar arguments hold for samples from $F_0$.
\end{proof}

As a consequence of Lemma~\ref{lem:mean_power}, we can infer that if the mean $\mu_0$ of $F_0$ and the mean $\mu_1$ of $F_1$ are different, then $F_0(\bar{h})$ will asymptotically converge to $F_0(\mu_0)$ under $F_0$ and to $F_0(\mu_1)$ under $F_1$.
This contributes to the ability of CES in distinguishing between $F_0$ and $F_1$.



Now, we study the behaviour of the max component of CES.

\begin{lemma}[Concentration of the max component]
\label{lem:max_power}
Suppose $h^{(1)}, \dotsc, h^{(m)} \overset{\mathrm{iid}}{\sim} F_1$ and let $\zeta_1 = \sup\{z : F_1(z) < 1\}$ be the right endpoint of the support of $F_1$.
Similarly, let $\zeta_0 = \sup\{z : F_0(z) < 1\}$ be the right endpoint of $F_0$. Then, the following hold.
\begin{enumerate}
    \item[(a)] For all $z \in \mathbb{R}$, $\mathbb{P}_{F_1}^m(h_{\max} \leq z) = F_1(z)^m$.
    \item[(b)] For $z \in \mathbb{R}$, we have
    \begin{equation}
      \mathbb{P}_{F_1}^m\bigl(F_0(h_{\max}) \geq F_0(z)\bigr) \geq 1 - F_1(z)^m.
    \end{equation}
    \item[(c)] If $\zeta_1 > \zeta_0$ and $F_0$ is continuous at $\zeta_1$, then $F_0(h_{\max}) \xrightarrow{m \to \infty} 1$ $F_1$-almost surely.
\end{enumerate}
\end{lemma}
\begin{proof}
(a) Since $h^{(1)}, \dotsc, h^{(m)}$ are iid, we have $\mathbb{P}_{F_1}(\max_{i \in [m]} h^{(i)} \leq z) = \mathbb{P}_{F_1}(h^{(1)} \leq z, \dotsc, h^{(m)} \leq z) = (\mathbb{P}_{F_1}(h \leq z))^m$.

(b) By part (a), we have $\mathbb{P}_{F_1}^m(h_{\max} > z) = 1 - F_1(z)^m$.
Now, observe that $h_{\max} > z$ implies $F_0(h_{\max}) \geq F_0(z)$, since $F_0$ is a monotonically increasing function.
Therefore, $\mathbb{P}_{F_1}^m(F_0(h_{\max}) \geq F_0(z)) \geq \mathbb{P}_{F_1}^m(h_{\max} > z) = 1 - F_1(z)^m$.

(c) Since $F_0(\zeta_0) = 1$ and $\zeta_0 < \zeta_1$, we have $F_0(\zeta_1) = 1$. Since $\zeta_1 < \infty$ (because $h$ is a bounded random variable), we have $h_{\max} \to \zeta_1$ $F_1$-almost surely. By continuity of $F_0$ at $\zeta_1$, we have $F_0(h_{\max}) \to 1$ $F_1$-almost surely.
\end{proof}
Note that in Lemma~\ref{lem:max_power}(b), for $z \geq \zeta_1$, we have have the trivial bound of $1 - F_1(\zeta)^m = 0$ on the probability, and therefore, it suffices to restrict to $z < \zeta_1$. Furthermore, by definition of $\zeta_1$, for $z < \zeta_1$, we have $F_1(z) < 1$, so that the bound $1 - F_1(z)^m$ converges to $1$ exponentially quickly with $m$.


We are now ready to combine results on the mean and the max components to derive useful properties of CES.
In the following, for a cumulative distribution function $F$, we define $F^{-1}(p) = \inf \{x \in \mathbb{R} \mid F(x) \geq p\}$ for $p \in [0, 1]$.

\begin{theorem}[Power of CES]
\label{thm:ces_power}
Suppose that $h^{(1)}, \dotsc, h^{(m)}$ are sampled i.i.d. according to either $F_0$ (non-hallucinated distribution) or $F_1$ (hallucinated distribution).
Define $\textnormal{CES} = \sqrt{F_0(\bar{h}) F_0(h_{\max})}$.
Denote the mean of $F_0$ as $\mu_0$ and $F_1$ as $\mu_1$.
Let $\zeta_0 = \sup \{z \in \mathbb{R} \mid F_0(z) < 1\}$ and $\zeta_1 = \sup \{z \in \mathbb{R} \mid F_1(z) < 1\}$ be the right end points of $F_0$ and $F_1$ respectively.
Let $\mu$ and $\zeta$ be tunable parameters that can be used to set a threshold.
Then, the following hold.

\begin{enumerate}
    \item[(a)] \textbf{True positive rate.}
    If the tunable parameters satisfy $\mu \in (-\infty, \mu_1)$ and $\zeta \in (-\infty, \zeta_1)$, we have
    \begin{equation}
      \mathbb{P}_{F_1}^m(\mathrm{CES} \geq \sqrt{F_0(\mu) F_0(\zeta)})
      \geq 1 - F_1(\zeta)^m - \exp\left(\frac{-2m (\mu_1 - \mu)^2}{(\log d)^2}\right). \label{eqn:ces_tpr_bd}
    \end{equation}
    If, in addition, $\sqrt{F_0(\mu) F_0(\zeta)} > F_0(\mu_1)$ holds, then we have the tighter bound
    \begin{equation}
    \begin{aligned}
      &\mathbb{P}_{F_1}^m(\mathrm{CES} \geq \sqrt{F_0(\mu) F_0(\zeta)}) \\
      &\geq 1 - \min\bigg\{\exp\left(\frac{-2m (F_0^{-1}(\sqrt{F_0(\mu) F_0(\zeta)} - \mu_1))^2}{(\log d)^2}\right), \\
      &\hspace{2cm}         F_1(\zeta)^m + \exp\left(\frac{-2m (\mu_1 - \mu)^2}{(\log d)^2}\right)\bigg\}.
    \end{aligned}
    \end{equation}

    \item[(b)] \textbf{False positive rate.} If the tunable parameters satisfy $\mu \in (\mu_0, \infty)$ and $\zeta \in (\zeta_0, \infty)$, we have
    \begin{equation}
      \mathbb{P}_{F_0}^m(\mathrm{CES} > \sqrt{F_0(\mu) F_0(\zeta)})
      \leq \exp\left(\frac{-2m (\mu - \mu_0)^2}{(\log d)^2}\right). \label{eqn:ces_fpr_bd}
    \end{equation}

    \item[(c)] \textbf{Consistency.} For tunable parameters in the range $\mu \in (\mu_0, \mu_1)$ and $\zeta \in (\zeta_0, \zeta_1)$,
    \begin{align}
        & \mathbb{P}_{F_1}^m(\textnormal{CES} \geq \sqrt{F_0(\mu) F_0(\zeta)}) \xrightarrow{\ m \to \infty\ } 1 && \textnormal{(True positive rate goes to $1$)} \\
        & \mathbb{P}_{F_0}^m(\textnormal{CES} > \sqrt{F_0(\mu) F_0(\zeta)})  \xrightarrow{\ m \to \infty\ } 0 && \textnormal{(False positive rate goes to $0$)}.
    \end{align}
\end{enumerate}
\end{theorem}
\begin{proof}
(a) Consider the events $A = \{F_0(\bar{h}) \geq F_0(\mu)\}$ and $B = \{F_0(h_{\max}) \geq F_0(\zeta)\}$.
By Lemma~\ref{lem:mean_power}(a), we have $\mathbb{P}_{F_1}(A^c) = \mathbb{P}_{F_1}^m(F_0(\bar{h}) < F_0(\mu)) \leq \exp(-2 m (\mu_1 - \mu)^2)$ for $\mu \leq \mu_1$.
On the other hand, by Lemma~\ref{lem:max_power}(b), we have $\mathbb{P}_{F_1}(B^c) = \mathbb{P}_{F_1}^m(F_0(h_{\max}) < F_0(\zeta)) \leq F_1(\zeta)^m$ for $\zeta < \zeta_1$.
Furthermore, on $A \cap B$, we have $\textnormal{CES} \geq \sqrt{F_0(\mu) F_0(\zeta)}$.
Therefore, $\mathbb{P}_{F_1}^m(\textnormal{CES} \geq \sqrt{F_0(\mu) F_0(\zeta)}) \geq \mathbb{P}_{F_1}^m(A \cap B) \geq 1 - \mathbb{P}_{F_1}^m(A^c) - \mathbb{P}_{F_1}^m(B^c) \geq 1 - F_1(\zeta)^m - \exp(-2 m (\mu_1 - \mu)^2 / \log^2(d))$.

Now, denote $c = \sqrt{F_0(\mu) F_0(\zeta)}$ for convenience and assume that $c > F_0(\mu_1)$.
Since $h_{\max} \geq \bar{h}$ and $F_0$ is monotonically increasing, we have $F_0(h_{\max}) \geq F_0(\bar{h})$, so that $\textnormal{CES} \geq F_0(\bar{h})$ always holds.
Furthermore, $F_0(\bar{h}) < c$ implies $\bar{h} < F_0^{-1}(c)$.
Therefore, $\mathbb{P}_{F_1}^m(\textnormal{CES} < c) \leq \mathbb{P}_{F_1}^m(F_0(\bar{h}) < c) \leq \mathbb{P}_{F_1}^m(\bar{h} < F_0^{-1}(c))$.
Since $1 \geq c > F_0(\mu_1)$, we have $\mu_1 < F_0^{-1}(c)$. Thus, by one-sided Hoeffding's inequality, we have
\begin{equation}
    \mathbb{P}_{F_1}^m(\bar{h} < F_0^{-1}(c)) = \mathbb{P}_{F_1}^m(\bar{h} - \mu_1 < F_0^{-1}(c) - \mu_1) \leq \exp\left(\frac{-2 m (F_0^{-1}(c) - \mu_1)^2}{(\log d)^2}\right).
\end{equation}
Combining the inequalities, we obtain $\mathbb{P}_{F_1}^m(\textnormal{CES} < \sqrt{F_0(\mu) F_0(\zeta)}) \leq \min\{\exp(-2 m (F_0^{-1}(c) - \mu_1)^2 / \log^2(d)), F_1(\zeta)^m + \exp(-2 m (\mu_1 - \mu)^2 / \log^2(d))\}$.

(b) First, note that for $\zeta > \zeta_0$, we have $F_0(\zeta) = 1$, so that $\sqrt{F_0(\mu) F_0(\zeta)} = \sqrt{F_0(\mu)}$. We also have $\textnormal{CES} \leq \sqrt{F_0(\bar{h})}$ since $F_0(h_{\max}) \leq 1$. Therefore, $\mathbb{P}_{F_0}^m(\textnormal{CES} > \sqrt{F_0(\mu) F_0(\zeta)}) \leq \mathbb{P}_{F_0}^m(\sqrt{F_0(\bar{h})} > \sqrt{F_0(\mu)}) = \mathbb{P}_{F_0}^m(\bar{h} > \mu) = \mathbb{P}_{F_0}^m(\bar{h} - \mu_0 > \mu - \mu_0) \leq \exp(-2 m (\mu - \mu_0)^2 / \log^2(d))$, where we used one-sided Hoeffding's inequality in the last step for the bounded random variables $h^{(1)}, \dotsc, h^{(m)}$.

(c) For $\zeta < \zeta_1$, we must have $F_1(\zeta) < 1$ by definition of $\zeta_1$, and therefore, $F_1(\zeta)^m \to 0$ as $m \to \infty$. It follows from Eq.~\eqref{eqn:ces_tpr_bd} that true positive rate goes to $1$ as $m \to \infty$.
On the other hand, Eq.~\eqref{eqn:ces_fpr_bd} implies that the false positive rate goes to $0$ as $m \to \infty$.
\end{proof}

Theorem~\ref{thm:ces_power} shows that in the ideal scenario, as the number of samples becomes sufficiently large, both the Type I error (false positives) and Type II error (false negatives) can be made small enough. This observation can be used to determine a threshold $c$ for a hypothesis test $\bm{1}_{\textnormal{CES} > c}$ based on CES ($\textnormal{CES} > c$ means we declare hallucination). Choosing $c$ is equivalent to choosing an appropriate $\mu$ and $\zeta$. This choise can be made, for example, by maximizing the true positive rate, subject to the constraint that the false positive rate is bounded above by a small number $\alpha$. This is equivalent to minimizing the Type II error while bounding the Type I error by $\alpha$. In practice, this can be done, for example, by computing the ROC (which is the plot of true positive rate vs false positive rate) using some test samples set aside for this purpose, and choosing the threshold that gives the largest value on the $y$-axis (true positive rate) for while restricting to values less than or equal to $\alpha$ on the $x$-axis (false positive rate).

\section{Experimental Details and Design Choices}\label{app:exp_details}

This appendix documents the key design choices made in our experimental pipeline, together with ablations justifying each decision.
We also provide additional plots that complement the main paper.

\subsection{Datasets}\label{app:datasets}

All eight evaluation datasets are open-ended QA or mathematical reasoning tasks where model answers are multi-token (typically 5--50+ words).
Entropy-based hallucination detection metrics require sufficient output tokens to produce a meaningful distributional signal; single- or two-token outputs (as in multiple-choice benchmarks) do not provide enough observations for reliable statistical testing.
Table~\ref{tab:datasets} summarises each dataset together with its source, hosting path, evaluation split, sample size, maximum generation length, and rationale for inclusion.

\begin{table}[h]
\centering
\caption{\textbf{Dataset inventory.} All datasets produce multi-token open-ended answers suitable for entropy-based detection.}
\label{tab:datasets}
\resizebox{\textwidth}{!}{%
\begin{tabular}{@{}llllccl@{}}
\toprule
\textbf{Dataset} & \textbf{Source} & \textbf{HF ID / Path} & \textbf{Split} & $N$ & \textbf{Max Tok.} & \textbf{Description} \\
\midrule
TriviaQA & \citet{joshi2017triviaqa} & \texttt{TimoImhof/TriviaQA-in-SQuAD-format} & unmodified & 500 & 128 & Open-ended factoid QA; multi-sentence answers \\
Natural Questions & \citet{kwiatkowski2019natural} & \texttt{google-research-datasets/nq\_open} & validation & 500 & 128 & Real Google queries; meaningful uncertainty variation \\
SQuAD v2 & \citet{rajpurkar2018know} & \texttt{rajpurkar/squad\_v2} & validation & 500 & 128 & Standard extractive QA; includes unanswerable questions \\
BioASQ & \citet{tsatsaronis2015overview} & Local: \texttt{training11b.json} & train & 500 & 128 & Domain-specialised (biomedical) QA \\
DROP & \citet{dua2019drop} & \texttt{ucinlp/drop} & validation & 500 & 128 & Discrete reasoning (counting, sorting, arithmetic) \\
CoQA & \citet{reddy2019coqa} & \texttt{stanfordnlp/coqa} & validation & 500 & 128 & Conversational QA with coreference \\
GSM8K & \citet{cobbe2021training} & \texttt{openai/gsm8k} & test & 500 & 256 & Math chain-of-thought; high entropy variance \\
SVAMP & \citet{patel2021nlp} & \texttt{ChilleD/SVAMP} & test & 300 & 256 & Adversarial math word problems \\
\bottomrule
\end{tabular}%
}
\end{table}

\textbf{Split selection rationale.}
We use the \texttt{test} split when public gold labels are available (GSM8K, SVAMP).
We use \texttt{validation} when the official test set is hidden behind a leaderboard and no public labels exist (SQuAD~v2, Natural Questions, DROP, CoQA).
For TriviaQA, the \texttt{TimoImhof/TriviaQA-in-SQuAD-format} repository provides only a single split labelled ``unmodified'' containing the full dataset reformatted into SQuAD-style JSON; there is no train/val/test partition in that particular HuggingFace repository.
For BioASQ, competition test sets are released per round and are not redistributed publicly; the \texttt{training11b.json} file is the only freely available labelled data after registration at \url{bioasq.org}.

\subsection{Model Inventory}\label{app:models}

We evaluate 16 models spanning base, instruction-tuned, chat-tuned, and proprietary API variants.
Table~\ref{tab:models} provides the full inventory.

\begin{table}[h]
\centering
\caption{\textbf{Model inventory.} All open-weight models are run in FP16 without quantisation.}
\label{tab:models}
\begin{tabular}{@{}llll@{}}
\toprule
\textbf{Model} & \textbf{Type} & \textbf{Precision} & \textbf{GPU Strategy} \\
\midrule
Llama-2-7B & Base & FP16 & Single GPU \\
Llama-2-13B & Base & FP16 & Single GPU \\
Llama-2-70B & Base & FP16 & \texttt{device\_map="auto"} across 8 GPUs \\
Falcon-7B-Instruct & Instruct & FP16 & Single GPU \\
Falcon-40B-Instruct & Instruct & FP16 & \texttt{device\_map="auto"} across 8 GPUs \\
Mistral-7B-v0.3 & Base & FP16 & Single GPU \\
Llama-3.2-1B & Base & FP16 & Single GPU \\
Llama-3.2-1B-Instruct & Instruct & FP16 & Single GPU \\
Llama-2-7B-Chat & Chat & FP16 & Single GPU \\
Llama-2-13B-Chat & Chat & FP16 & Single GPU \\
Llama-2-70B-Chat & Chat & FP16 & \texttt{device\_map="auto"} across 8 GPUs \\
Meta-Llama-3-8B-Instruct & Instruct & FP16 & Single GPU \\
\midrule
gpt-4o-mini-2024-07-18 & API & N/A & Azure OpenAI \\
gpt-4.1-nano-2025-04-14 & API & N/A & Azure OpenAI \\
gpt-4.1-mini-2025-04-14 & API & N/A & Azure OpenAI \\
gpt-4.1-2025-04-14 & API & N/A & Azure OpenAI \\
\bottomrule
\end{tabular}
\end{table}

Multi-GPU inference uses HuggingFace Accelerate with \texttt{device\_map="auto"}, which automatically shards model layers across all available GPUs via pipeline parallelism (40\,GiB maximum per device).
For single-GPU models, the shell script runs all eight datasets in parallel, assigning one dataset per GPU in round-robin fashion via \texttt{--parallel \$NUM\_GPUS}.

\subsection{Generation Parameters}\label{app:generation}

Table~\ref{tab:gen_params} summarises the generation configuration for both primary (greedy) and multi-trajectory ($K{=}10$) settings.

\begin{table}[h]
\centering
\caption{\textbf{Generation parameters.}}
\label{tab:gen_params}
\resizebox{\textwidth}{!}{%
\begin{tabular}{@{}lllll@{}}
\toprule
\textbf{Setting} & \textbf{Local Primary} & \textbf{Local Trajectories ($K{=}10$)} & \textbf{API Primary} & \textbf{API Trajectories ($K{=}10$)} \\
\midrule
Temperature & 0 (greedy, \texttt{do\_sample=False}) & 1.0 (\texttt{do\_sample=True}) & 0.7 & 1.0 \\
\texttt{top\_k} & 0 (off) & 0 (off) & N/A & N/A \\
\texttt{top\_p} & 1.0 (off) & 1.0 (off) & N/A & N/A \\
Precision & FP16 & FP16 & N/A & N/A \\
Quantisation & None & None & N/A & N/A \\
Batching & Left-padded, variable batch size & Same prompt $\times\, K$ & Sequential & Sequential \\
Early stopping & \texttt{StoppingCriteriaSub} + post-hoc truncation & Post-hoc truncation only & N/A (API handles) & N/A \\
Logprobs & Full logits via \texttt{output\_scores=True} & Same & \texttt{top\_logprobs=20} & \texttt{top\_logprobs=20} \\
\bottomrule
\end{tabular}%
}
\end{table}

\textbf{Per-dataset batch sizes.}
Table~\ref{tab:batch_sizes} reports the base batch sizes before model-dependent scaling.
Scaling factors are: 70B $\to$ $/8$, 40B $\to$ $/4$, 13B $\to$ $/2$, and 7B/8B/3B/1B $\to$ $/1$.
These factors apply to the generation and FAVA batches but not to DeBERTa.

\begin{table}[h]
\centering
\caption{\textbf{Per-dataset batch sizes} (base, before model scaling).}
\label{tab:batch_sizes}
\resizebox{\textwidth}{!}{%
\begin{tabular}{@{}lccccl@{}}
\toprule
\textbf{Datasets} & \textbf{Gen} & \textbf{FAVA} & \textbf{DeBERTa} & \textbf{Max Tok.} & \textbf{Stop Sequences} \\
\midrule
triviaqa, nq\_open, squad, bioasq & 32 & 16 & 128 & 128 &
\texttt{Question:}, \texttt{Context:} \\
drop, coqa & 16 & 8 & 64 & 128 & (same as above) \\
gsm8k, svamp & 8 & 4 & 32 & 256 &
\texttt{Question:}, \texttt{Context:} \\
\bottomrule
\end{tabular}%
}
\end{table}

Table~\ref{tab:seeds} lists all random seeds used in the pipeline.

\begin{table}[h]
\centering
\caption{\textbf{Random seeds.}}
\label{tab:seeds}
\begin{tabular}{@{}lcl@{}}
\toprule
\textbf{Component} & \textbf{Seed} & \textbf{Controls} \\
\midrule
Dataset sampling (local) & 10 & Which 500 samples are drawn from each split \\
Dataset sampling (API) & 42 & Same, for API experiments \\
Few-shot example selection & 42 & Which examples form the 5-shot prefix \\
Bootstrap CI (AUROC/AUARC) & 42 & 1000-iteration bootstrap resampling \\
\bottomrule
\end{tabular}
\end{table}

\subsection{Hallucination Oracle}\label{app:oracle}

We adopt an LLM-as-a-Judge framework for hallucination labelling.
The primary judge for open-weight (local) models is \texttt{gpt-4.1-nano-2025-04-14}; for API-based models we use \texttt{gpt-5.4-2026-03-05}.
The judge is called at temperature $0.0$ (greedy) with a maximum of 256 tokens per item and a batch size of 20 items per API call.
For \texttt{gpt-5.4} (a routing model), no temperature parameter is exposed; for reasoning-class models (\texttt{gpt-5.*}, \texttt{o1}, \texttt{o3}), the \texttt{system} role is replaced by the \texttt{developer} role and \texttt{max\_completion\_tokens} is used in place of \texttt{max\_tokens}.

\textbf{Judge prompts.}
The single-item evaluation uses the following system and user prompts:

\begin{tcolorbox}[colback=gray!5, colframe=gray!50, title=Single-item judge prompt]
\small
\textbf{System:} You are an expert evaluator. Your task is to determine whether an AI-generated answer is factually correct when compared to a known ground-truth answer. Be lenient with formatting differences but strict on factual content. Reply ONLY with a JSON object.

\medskip
\textbf{User:}\\
Question: \{question\}\\
Ground-truth answer: \{ground\_truth\}\\
AI-generated answer: \{generated\_answer\}\\[4pt]
Is the AI-generated answer factually correct (i.e., consistent with the ground-truth)? Reply with a JSON object containing exactly three keys:\\
\texttt{\{"verdict": "correct" or "incorrect", "confidence": float 0-1, "reasoning": "brief explanation"\}}
\end{tcolorbox}

For batched evaluation (20 items per call), the system prompt is:

\begin{tcolorbox}[colback=gray!5, colframe=gray!50, title=Batch judge system prompt]
\small
You are an expert evaluator. For each numbered item below, determine whether the AI-generated answer is factually correct compared to the ground-truth answer. Be lenient with formatting differences but strict on factual content.

Reply ONLY with a JSON array of objects, one per item, in order. Each object must have exactly three keys: \texttt{"verdict"}, \texttt{"confidence"}, \texttt{"reasoning"}.
\end{tcolorbox}

\textbf{Hallucination labelling procedure.}
The judge returns a JSON object with keys \texttt{verdict}, \texttt{confidence}, and \texttt{reasoning}.
We set $\texttt{is\_hallucinated} = 1$ if the verdict is \texttt{"incorrect"}, and $0$ otherwise.
Two additional rejudge passes use deterministic metrics: SQuAD-F1 (token overlap $\geq 0.5$) and ROUGE-L (LCS F-measure $\geq 0.5$).
All three labelling schemes produce separate detection evaluation tables for robustness analysis.

\subsection{Other Experimental Details}\label{app:other_experimental_details}

Several auxiliary models support the multi-trajectory baselines and fine-grained detection.
Table~\ref{tab:supporting} lists each model and its role.

\begin{table}[h]
\centering
\caption{\textbf{Supporting models.}}
\label{tab:supporting}
\begin{tabular}{@{}p{2.8cm}p{7cm}p{4cm}@{}}
\toprule
\textbf{Role} & \textbf{Model} & \textbf{Usage} \\
\midrule
Entailment (NLI) & \texttt{microsoft/deberta-v2-xlarge-mnli} & Semantic equivalence clustering (SE, DSE, KLE) \\
Sentence embeddings & \texttt{sentence-transformers/all-MiniLM-L6-v2} & 384-dim vectors (eigenscore, eff.\ rank, KLE, emb.\ regression) \\
FAVA & \texttt{fava-uw/fava-model} & Fine-grained hallucination detection \\
\bottomrule
\end{tabular}
\end{table}

\textbf{Few-Shot Prompting and Chat Templates.}
All eight datasets use 5-shot prompting.
The five exemplars are selected with seed~42 from the training split of each dataset and formatted as \texttt{Question: \ldots\ Answer: \ldots} pairs preceding the target question.

For base models (Llama-2, Mistral, Llama-3.2-1B), the prompt is raw concatenated text.
For instruction-tuned and chat models (Llama-2-Chat, Llama-3-Instruct, Falcon-Instruct), the few-shot examples are wrapped in the model's native chat template via \texttt{tokenizer.apply\_chat\_template()}.
For API models, a system prompt (\texttt{"You are a helpful assistant. Answer the question directly and concisely."}) is prepended, and the few-shot examples plus target question form the user message.

\textbf{Multi-Trajectory Sampling.} For each question, we draw $K{=}10$ independent samples at temperature $T{=}1.0$.
Log-likelihoods are computed from the full logit vectors for local models and from the API's \texttt{logprobs} field (\texttt{top\_logprobs=20}) for API models.
These trajectories are consumed by the following baselines: semantic entropy (SE), discrete semantic entropy (DSE), kernel language entropy (KLE), eigenscore, effective rank, self-check-GPT, P(True), and embedding regression.

\textbf{Batching.}
Batching is a pure throughput optimisation that does not alter results.
Left-padding combined with attention masks ensures each sequence in a batch sees exactly the same tokens as if processed alone: padding tokens are masked out of all attention computations and contribute zero signal.
The model's forward pass is mathematically identical per-sequence regardless of batch size.
This guarantee holds for every model in the pipeline:
the generator (greedy argmax or temperature sampling over per-sequence logit distributions with padding masked out),
DeBERTa (deterministic classification with padding masked; a batch of 128 pairs yields identical predictions to sequential processing),
MiniLM (mean-pooling over attention-masked tokens only; identical 384-dimensional vectors regardless of batch size),
and FAVA (same transformer masking guarantees).
The per-dataset batch scaling (70B to $/8$, 40B to $/4$, etc.) keeps VRAM usage just below the 40\,GiB limit per GPU.

\textbf{Minimum Sequence Length.} We include all sequences with $\geq 2$ tokens (the minimum for a meaningful KS statistic), without imposing a stricter length filter.
Since our goal is a general-purpose detector, we retain all evaluable sequences.
The length stratification in Section~\ref{app:exp_length} separately characterises performance by generation length.

\textbf{Decoding Strategy.} All primary experiments use greedy decoding ($T = 0$).
Greedy decoding represents the most common deployment setting for factual QA systems and produces deterministic outputs, ensuring reproducibility.
The token-level entropy under greedy decoding reflects the model's posterior uncertainty at each position, not sampling noise.
This is the hardest setting for entropy-based detection: under temperature sampling, hallucinated sequences would exhibit \emph{higher} entropy by construction (more stochastic generation), inflating detection scores.
By evaluating at $T = 0$, we test whether intrinsic uncertainty (reflected in the full logit distribution) is discriminative even when the \emph{selected} token is always the argmax.
Future work can evaluate the impact of different decoding regimes. 

\textbf{Evaluation Metrics.} The main measure we compare our methods against is AUROC.
Bootstrap 95\% confidence intervals (where included) are computed with 1000 iterations (seed~42).
For metrics where a lower score indicates greater hallucination probability (e.g., P(True)), scores are sign-flipped before computing AUROC to ensure consistent directionality.




\textbf{Content Filtering.} 
Azure content filters occasionally refuse to generate an answer (3-4 out of 500 samples for CoQA; 0 for most other datasets).
These samples are skipped in metric computation since no trajectory is available.
Final sample counts range from 496 to 500 per dataset depending on the model-dataset combination.
We believe this does not affect the integrity of our results. 

\section{Baseline Methods}
\label{app:baselines}
In this section, we discuss the details of our benchmark baselines. 
We benchmark CES against three families of hallucination
detection methods, organised by the computational resources
they require at inference time.
Table~\ref{tab:baseline_taxonomy} summarises the key distinctions;
formal definitions follow.
Recent surveys \cite{xia2025survey, janiak2025illusion} provide
broader coverage and argue that, despite a proliferation of
methods, progress on detection accuracy has been limited. 
This finding is corroborated for the single-pass setting.

\begin{table}[h]
\centering
\caption{Taxonomy of baseline methods by inference-time
requirements.
\textbf{Passes}: number of forward passes through the
generator.
\textbf{Access}: what the method reads from the model.
\textbf{External}: whether an auxiliary model or knowledge
source is needed.}
\label{tab:baseline_taxonomy}
\small
\setlength{\tabcolsep}{4pt}
\begin{tabular}{@{}llccc@{}}
\toprule
\textbf{Family} & \textbf{Method} & \textbf{Passes} &
\textbf{Access} & \textbf{External} \\
\midrule
\multirow{3}{*}{\textit{Single-pass}}
  & Perplexity \cite{jelinek1977perplexity}
    & 1 & logits & \texttimes \\
  & LN Entropy \cite{malinin2020uncertainty}
    & 1 & logits & \texttimes \\
  & Generation length \cite{janiak2025illusion}
    & 1 & tokens & \texttimes \\
\midrule
\multirow{5}{*}{\textit{Multi-sample}}
  & Semantic Entropy \cite{farquhar2024detecting}
    & $K$ & logits & NLI model \\
  & KLE \cite{nikitin2024kernel}
    & $K$ & logits & NLI model \\
  & EigenScore \cite{chen2023going}
    & $K$ & embeddings & \texttimes \\
  & SAR \cite{duan2024shifting}
    & $K$ & logits & similarity fn \\
  & SelfCheckGPT \cite{manakul2023selfcheckgpt}
    & $K$ & tokens & \texttimes \\
\midrule
\textit{Elicitation}
  & $\mathbb{P}(\text{True})$ \cite{kadavath2022language}
    & 2 & logits & \texttimes \\
\bottomrule
\end{tabular}
\end{table}

Throughout this section, let $x$ denote the input prompt,
$s = (z_1, \ldots, z_m)$ the generated sequence of length~$m$,
and $\bm{p}^{(t)} = (p_v^{(t)})_{v \in \mathcal{V}}$ the
autoregressive next-token distribution at step~$t$, with
$p_v^{(t)} = p(v \mid s_{<t}, x)$.

\subsection{Single-Pass Methods}
\label{sec:baselines_single}

Single-pass methods require one forward pass and access to
token logits or the generated string.
They are the natural comparators for CES, since they operate
under identical computational constraints.

\textbf{Perplexity.}
Perplexity \cite{jelinek1977perplexity} is the exponentiated
average negative log-likelihood of the realised tokens:
\begin{equation}
\label{eq:ppl}
\mathrm{PPL}(x, s)
= \exp\!\Bigl(
  -\frac{1}{m}\sum_{t=1}^{m}
  \log p(z_t \mid s_{<t}, x)
\Bigr).
\end{equation}
Lower values indicate higher confidence.
Perplexity depends on the \emph{realised} token probabilities
and is therefore sensitive to the sampling path; it does not
use the full next-token distribution.

\textbf{Length-normalised entropy (LNE).}
LNE \cite{malinin2020uncertainty} averages the Shannon entropy
of the full next-token distribution at each step:
\begin{equation}
\label{eq:lne}
H_{\mathrm{len}}(x, s)
= -\frac{1}{m}\sum_{t=1}^{m}\sum_{v \in \mathcal{V}}
  p_v^{(t)} \log p_v^{(t)}.
\end{equation}
Unlike perplexity, LNE uses the entire predictive distribution
rather than only the probability of the chosen token.
It is the closest existing method to CES: CES can be viewed as
augmenting LNE with a tail statistic ($h_{\max}$) and a
calibration layer ($\widehat{F}_0$).

\textbf{Generation length.}
\cite{janiak2025illusion} observe that many uncertainty
metrics correlate with the number of generated tokens~$m$,
and that generation length alone is a surprisingly competitive
baseline.
We include $m$ as a standalone predictor to control for this
confound.

\subsection{Multi-Sample Methods}
\label{sec:baselines_multi}

Multi-sample methods draw $K \geq 2$ independent generations
(typically $K \in \{5, 10\}$ at temperature $T > 0$) and
measure agreement or spread.
They are not directly comparable to CES in computational cost
but represent the current state of the art on several
benchmarks.

\textbf{Semantic Entropy.}
\cite{farquhar2024detecting} cluster the $K$ generations into
semantic equivalence classes $\mathcal{C}$ using a natural
language inference (NLI) model, estimate class probabilities
$p_c$ by aggregating sequence log-likelihoods, and compute
\begin{equation}
\label{eq:se}
H_{\mathrm{sem}}(x)
= -\sum_{c \in \mathcal{C}} p_c \log p_c.
\end{equation}
This captures ambiguity over \emph{meanings} rather than
tokens, making it invariant to surface paraphrases.

\textbf{Kernel Language Entropy (KLE).}
\cite{nikitin2024kernel} replace the hard NLI clustering of
Semantic Entropy with a soft kernel defined over an entailment
graph, computing entropy via the graph Laplacian.
This avoids the discretisation artifacts of hard clustering.

\textbf{EigenScore.}
\cite{chen2023going} embed the $K$ generations as vectors
$\bm{Z} \in \mathbb{R}^{d \times K}$, form the regularised
centred Gram matrix
$\bm{G} = \bm{Z}^\top \bm{J}_d \bm{Z} + \alpha \bm{I}_K$
(where $\bm{J}_d = \bm{I}_d - d^{-1}\bm{1}\bm{1}^\top$),
and report
\begin{equation}
\label{eq:eigenscore}
\mathrm{ES}
= \frac{1}{K}\log\det(\bm{G})
= \frac{1}{K}\sum_i \log \lambda_i,
\end{equation}
measuring the volume spanned by the generation set in
embedding space.

\textbf{SelfCheckGPT.}
\cite{manakul2023selfcheckgpt} generate $K$ alternative
responses and measure consistency with the original generation
via $n$-gram overlap, NLI, or prompting.
Inconsistency across samples is taken as evidence of
hallucination.

\subsection{Elicitation-Based Methods}
\label{sec:baselines_elicitation}

Elicitation methods query the model about its own outputs.
Given prompt~$x$ and answer $a = f_\theta(x)$, a set of
meta-prompts $Q = \{q_1, \ldots, q_{|Q|}\}$ is constructed
(e.g., ``Is the above answer correct?'').
The model is queried with $(x \oplus a \oplus q)$ for each
$q \in Q$, and the responses are mapped to a scalar
confidence score via a downstream function
$\phi \colon \mathbb{R}^{|Q|} \to [0, 1]$.

\textbf{P(True)} The canonical instance is $\mathbb{P}(\text{True})$
\cite{kadavath2022language}, which appends a verification
question and reads off the probability assigned to an
affirmative token.
Elicitation methods are applicable in both full-logit and
top-$k$ settings \cite{pedapati2024large} but require at
least one additional forward pass per meta-prompt.

\subsection{Positioning of CES}
\label{sec:baselines_positioning}

Table \ref{tab:baseline_taxonomy} highlights the niche that CES
occupies.
Among single-pass methods, CES is the only one that provides formal statistical guarantees (calibrated Type~I error, exponential power, contamination robustness) while requiring the same computational budget as perplexity and LNE. 
Multi-sample and model-internal methods achieve higher raw AUROC on some benchmarks, but at $K\times$ the inference cost or with access requirements that preclude use through standard APIs.
CES is therefore best understood not as a replacement for these more expensive methods, but as the strongest detection one can perform under the minimal-access, single-pass constraint. 
This comes with the added benefit of statistical interpretability.

\section{Extended Experimental Results}\label{sec:extended_experiments}

We present a comprehensive evaluation of the Calibrated Entropy Score (CES) across 10 models (6 open-weight, 4 API-only), 8 datasets, and 16 benchmark methods including P(True).
All results are computed on 500 samples per experiment with a fixed random seed (42) for reproducibility.

\subsection{Summary of Key Findings}\label{sec:exp_summary}

\begin{enumerate}
    \item \textbf{Full Experimental Results}: we report the full experimental grids with interquantile ranges (IQRs) for each experiment in Appendix \ref{app:exp_grid}. 

    \item \textbf{Independence Assumption}: In Appendix \ref{app:exp_independence}, we show that the median lag-1 autocorrelation is $\rho_1 = 0.061$ with $15.8\%$ of the sequences exhibiting $|\rho_1|>0.3$, with $80.5\%$ of the sequences falling into white-noise bounds, witnessing a low autocorrelation for most samples.

    \item \textbf{Consistently Signficant KS test for Hallucinated versus non-Hallucinated Generations}: 
    In Appendix \ref{app:exp_power}, we find that power of the KS test increases monotonically with sample size; 
    65\% of all 3,200 resamples achieve significance. Specifically, we show that for all datasets and model pairs, as we increase i.i.d. sampling size, hallucinated v.s. faithful generation entropies are separated by the KS test, while we fail to reject the null for hallucinated v.s. hallucinated and faithful v.s. faithful distributions. 
    Token-level entropies of hallucinated and faithful generations are drawn from statistically different distributions in 72/80 (90\%) experiments at $\alpha{=}0.05$, with the \emph{shape} signal universally significant (80/80) beyond mere location shifts.
    We show examples of how the empirical cumulative distribution functions (ECDF) differ across datasets in Appendix \ref{app:ecdf_gallery}

    \item \textbf{CES Variants}: in Appendix \ref{app:exp_combinatorial}, we exhaustively evaluate 44 variants combining different entropy summaries (mean, median, max, $q_{25}$, $q_{75}$) under both arithmetic and geometric aggregation. 
    We find that CES outperforms the other combinations. 

    \item \textbf{Robustness Results}: In Appendix \ref{app:exp_contamination} and Appendix \ref{app:exp_noisy_judge}, we analyze the impact of adding noisy samples to the calibration distribution, finding robustness in the performance of CES.

    \item \textbf{CES is competitive}: in Appendix \ref{app:critical_difference_analysis}, we conclude that CES (unsupervised) ranks 3rd out of 17 methods by average rank (6.29) across all 80 experiments, belongs to the top statistical clique (Nemenyi $\text{CD} = 2.78$), and is significantly better than 7/16 benchmarks after Holm--Bonferroni correction.
    Unlike the top-performing methods (KLE, Embedding Regression) that require multiple forward passes or external embeddings, CES operates on a single generation's entropy trace, providing a significant computational advantage.

    \item \textbf{API model generalisation}: CES transfers effectively to API-only models (GPT-4.1 family), achieving median AUROC 0.669 on API models versus 0.642 on open-weight models (Mann-Whitney $p = 0.060$, KS $D = 0.323$, $p = 0.031$).
    This is analyzed in Appendix \ref{app:exp_api}.
    
    \item \textbf{Empirical Validation of Error Bounds}: in Appendix \ref{app:exp_error_bounds}, we validate Theorem \ref{thm:ces_power}, confirming the theorem's prediction on Type I and Type II exponential error decay rates. 

    \item \textbf{Example Outputs}: we show examples outputs from our generations in Appendix \ref{app:example_outputs}.
    
\end{enumerate}
\newpage
\subsection{Main Experimental Grid}\label{app:exp_grid}

\begin{figure}[H]
    \centering
    \includegraphics[width=\textwidth]{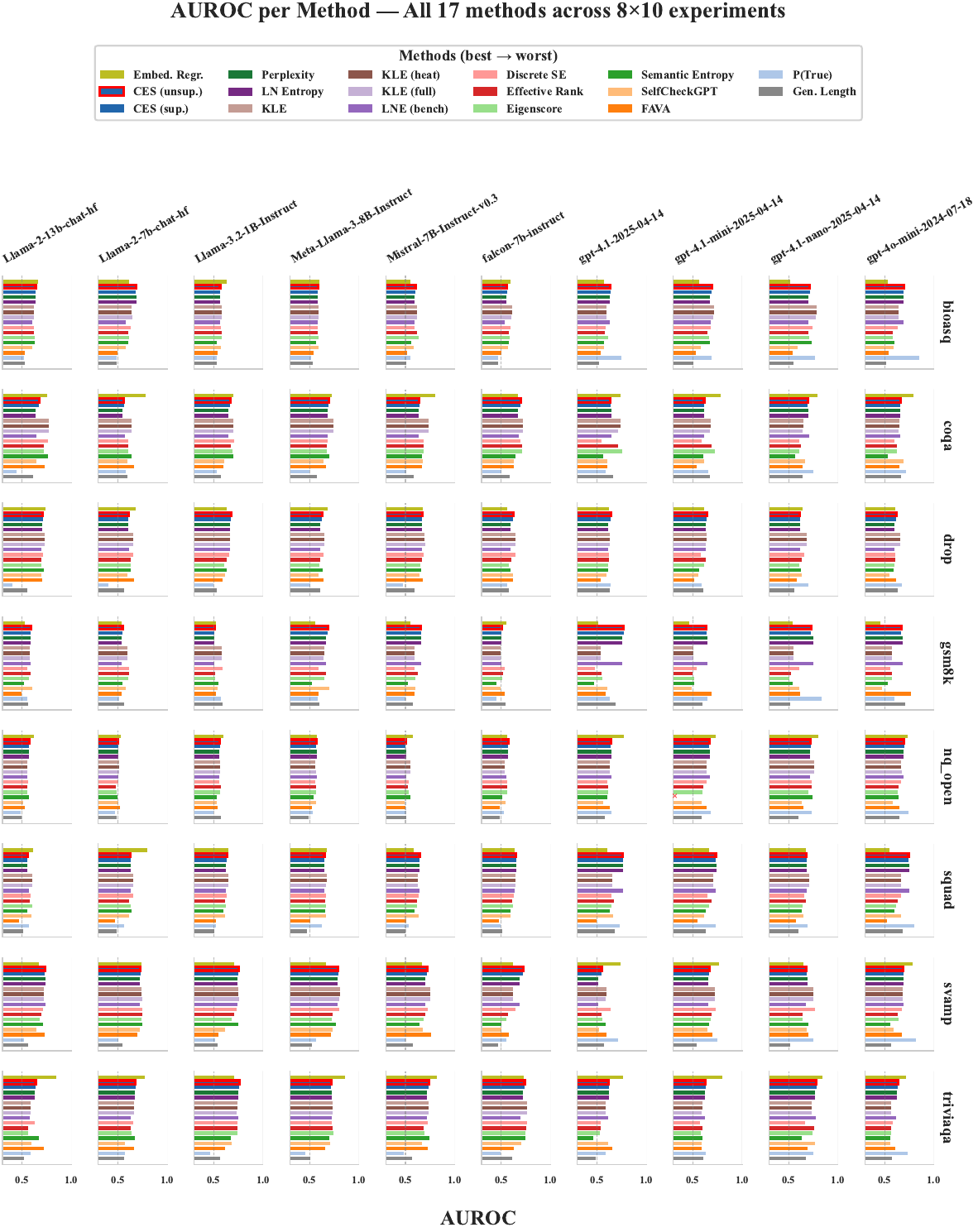}
    \caption{\textbf{Per-experiment AUROC for all 17 methods} across 10 models $\times$ 8 datasets.
    Each bar represents one method's AUROC; models are grouped by row and datasets by column.
    CES (unsupervised, starred) achieves a median AUROC of 0.653, comparable to Embedding Regression (0.665) and KLE variants (0.647--0.651).}
    \label{fig:full_matrix}
\end{figure}

\textbf{Setup.}
We evaluate 7 detection methods across our full $10 \times 8$ experimental grid: CES (supervised and unsupervised), KS 1-sample test (supervised and unsupervised), Perplexity, LN Entropy, and Generation Length.
Each experiment comprises 500 samples per model--dataset pair, with AUROC computed as the primary metric.

\textbf{Results.}
Tables \ref{tab:auroc_Llama213bcha}, 
\ref{tab:auroc_Llama27bchat}, 
\ref{tab:auroc_Llama321BIns}, 
\ref{tab:auroc_MetaLlama38B}, 
\ref{tab:auroc_Mistral7BIns}, 
\ref{tab:auroc_falcon7binst}, 
\ref{tab:auroc_gpt412025041}, 
\ref{tab:auroc_gpt41mini202}, 
\ref{tab:auroc_gpt41nano202}, 
\ref{tab:auroc_gpt4omini202} report the full AUROC matrix.
CES achieves a median AUROC of 0.653 (IQR: [0.617, 0.697]) across all 80 experiments.
The unsupervised variant (median 0.653, IQR: [0.617, 0.700]) is statistically indistinguishable from the supervised variant.
Both substantially outperform the KS 1-sample test (supervised median 0.555, unsupervised 0.526) and Generation Length (0.562).
The advantage over LN Entropy/Perplexity (both median 0.652) is consistent but small ($\Delta = +0.001$).
CES wins 43/80 experiments outright, compared to 18 for Perplexity, 11 for Generation Length, and 6 for the supervised variant.

\begin{table}[H]
\centering
\caption{AUROC $\pm$ std (rank) for \texttt{Llama-2-13b-chat-hf} across datasets.}
\label{tab:auroc_Llama213bcha}
\resizebox{\textwidth}{!}{%
\begin{tabular}{lcccccccc}
\toprule
Method & bioasq & coqa & drop & gsm8k & nq\_open & squad & svamp & triviaqa \\
\midrule
CES (unsup) & 0.636 $\pm$ 0.033 (2) & 0.674 $\pm$ 0.039 (8) & 0.709 $\pm$ 0.022 (7) & 0.590 $\pm$ 0.029 (2) & 0.572 $\pm$ 0.027 (4) & 0.556 $\pm$ 0.032 (8) & 0.734 $\pm$ 0.030 (3) & 0.636 $\pm$ 0.073 (4) \\
CES (sup) & 0.635 $\pm$ 0.034 (5) & 0.673 $\pm$ 0.041 (9) & 0.711 $\pm$ 0.023 (6) & 0.589 $\pm$ 0.027 (5) & 0.571 $\pm$ 0.026 (5) & 0.556 $\pm$ 0.031 (9) & 0.733 $\pm$ 0.029 (4) & 0.636 $\pm$ 0.074 (5) \\
SE & 0.630 $\pm$ 0.026 (6) & 0.758 $\pm$ 0.041 (4) & 0.717 $\pm$ 0.023 (4) & 0.518 $\pm$ 0.028 (12) & 0.566 $\pm$ 0.026 (6) & 0.553 $\pm$ 0.030 (10) & 0.713 $\pm$ 0.030 (9) & 0.668 $\pm$ 0.066 (3) \\
DSE & 0.619 $\pm$ 0.024 (10) & 0.760 $\pm$ 0.042 (3) & 0.713 $\pm$ 0.023 (5) & 0.555 $\pm$ 0.029 (9) & 0.558 $\pm$ 0.026 (7) & 0.585 $\pm$ 0.029 (6) & 0.714 $\pm$ 0.031 (8) & 0.624 $\pm$ 0.058 (8) \\
KLE & 0.621 $\pm$ 0.026 (8) & 0.775 $\pm$ 0.042 (1)$^*$ & 0.727 $\pm$ 0.022 (3) & 0.576 $\pm$ 0.028 (7) & 0.554 $\pm$ 0.028 (8) & 0.606 $\pm$ 0.029 (2) & 0.719 $\pm$ 0.029 (6) & 0.585 $\pm$ 0.064 (11) \\
KLE-Full & 0.620 $\pm$ 0.025 (9) & 0.771 $\pm$ 0.042 (2) & 0.730 $\pm$ 0.023 (2) & 0.580 $\pm$ 0.028 (6) & 0.552 $\pm$ 0.025 (9) & 0.606 $\pm$ 0.030 (3) & 0.717 $\pm$ 0.029 (7) & 0.588 $\pm$ 0.069 (10) \\
Eigenscore & 0.624 $\pm$ 0.025 (7) & 0.718 $\pm$ 0.044 (7) & 0.694 $\pm$ 0.025 (11) & 0.569 $\pm$ 0.028 (8) & 0.548 $\pm$ 0.026 (10) & 0.599 $\pm$ 0.029 (4) & 0.681 $\pm$ 0.031 (10) & 0.558 $\pm$ 0.074 (13) \\
SelfCheckGPT & 0.604 $\pm$ 0.025 (11) & 0.641 $\pm$ 0.045 (10) & 0.693 $\pm$ 0.023 (12) & 0.599 $\pm$ 0.027 (1)$^*$ & 0.508 $\pm$ 0.025 (12) & 0.592 $\pm$ 0.030 (5) & 0.645 $\pm$ 0.032 (12) & 0.594 $\pm$ 0.068 (9) \\
P(True) & 0.521 $\pm$ 0.026 (13) & 0.443 $\pm$ 0.048 (13) & 0.404 $\pm$ 0.025 (13) & 0.550 $\pm$ 0.028 (10) & 0.481 $\pm$ 0.026 (13) & 0.565 $\pm$ 0.028 (7) & 0.522 $\pm$ 0.033 (13) & 0.584 $\pm$ 0.077 (12) \\
FAVA & 0.529 $\pm$ 0.026 (12) & 0.731 $\pm$ 0.040 (6) & 0.706 $\pm$ 0.023 (8) & 0.493 $\pm$ 0.028 (13) & 0.526 $\pm$ 0.026 (11) & 0.465 $\pm$ 0.029 (13) & 0.730 $\pm$ 0.031 (5) & 0.719 $\pm$ 0.055 (2) \\
EmbRegress & 0.664 $\pm$ 0.024 (1)$^*$ & 0.755 $\pm$ 0.039 (5) & 0.738 $\pm$ 0.022 (1)$^*$ & 0.528 $\pm$ 0.029 (11) & 0.617 $\pm$ 0.027 (1)$^*$ & 0.608 $\pm$ 0.057 (1)$^*$ & 0.672 $\pm$ 0.033 (11) & 0.847 $\pm$ 0.042 (1)$^*$ \\
LN-Entropy & 0.635 $\pm$ 0.037 (3) & 0.639 $\pm$ 0.044 (11) & 0.703 $\pm$ 0.025 (9) & 0.590 $\pm$ 0.026 (3) & 0.573 $\pm$ 0.027 (2) & 0.549 $\pm$ 0.032 (11) & 0.736 $\pm$ 0.029 (1)$^*$ & 0.626 $\pm$ 0.076 (6) \\
Perplexity & 0.635 $\pm$ 0.032 (4) & 0.639 $\pm$ 0.044 (12) & 0.703 $\pm$ 0.024 (10) & 0.590 $\pm$ 0.027 (4) & 0.573 $\pm$ 0.026 (3) & 0.549 $\pm$ 0.032 (12) & 0.736 $\pm$ 0.030 (2) & 0.626 $\pm$ 0.073 (7) \\
\bottomrule
\end{tabular}}
\end{table}

\begin{table}[H]
\centering
\caption{AUROC $\pm$ std (rank) for \texttt{Llama-2-7b-chat-hf} across datasets.}
\label{tab:auroc_Llama27bchat}
\resizebox{\textwidth}{!}{%
\begin{tabular}{lcccccccc}
\toprule
Method & bioasq & coqa & drop & gsm8k & nq\_open & squad & svamp & triviaqa \\
\midrule
CES (unsup) & 0.683 $\pm$ 0.030 (3) & 0.562 $\pm$ 0.040 (11) & 0.608 $\pm$ 0.025 (8) & 0.546 $\pm$ 0.027 (6) & 0.499 $\pm$ 0.028 (10) & 0.629 $\pm$ 0.027 (7) & 0.730 $\pm$ 0.030 (7) & 0.679 $\pm$ 0.044 (2) \\
CES (sup) & 0.683 $\pm$ 0.030 (4) & 0.562 $\pm$ 0.042 (10) & 0.608 $\pm$ 0.026 (11) & 0.545 $\pm$ 0.026 (7) & 0.500 $\pm$ 0.027 (9) & 0.628 $\pm$ 0.027 (8) & 0.727 $\pm$ 0.028 (8) & 0.678 $\pm$ 0.042 (3) \\
SE & 0.601 $\pm$ 0.025 (10) & 0.637 $\pm$ 0.040 (5) & 0.628 $\pm$ 0.024 (7) & 0.543 $\pm$ 0.026 (8) & 0.498 $\pm$ 0.027 (11) & 0.641 $\pm$ 0.026 (5) & 0.751 $\pm$ 0.028 (1)$^*$ & 0.672 $\pm$ 0.044 (4) \\
DSE & 0.626 $\pm$ 0.026 (7) & 0.608 $\pm$ 0.036 (6) & 0.656 $\pm$ 0.025 (3) & 0.611 $\pm$ 0.025 (1)$^*$ & 0.505 $\pm$ 0.027 (5) & 0.653 $\pm$ 0.024 (4) & 0.745 $\pm$ 0.029 (2) & 0.653 $\pm$ 0.043 (10) \\
KLE & 0.641 $\pm$ 0.026 (6) & 0.642 $\pm$ 0.042 (3) & 0.655 $\pm$ 0.024 (5) & 0.597 $\pm$ 0.025 (3) & 0.508 $\pm$ 0.027 (3) & 0.658 $\pm$ 0.026 (2) & 0.741 $\pm$ 0.028 (4) & 0.661 $\pm$ 0.045 (8) \\
KLE-Full & 0.648 $\pm$ 0.024 (5) & 0.640 $\pm$ 0.042 (4) & 0.656 $\pm$ 0.025 (4) & 0.596 $\pm$ 0.025 (4) & 0.508 $\pm$ 0.028 (4) & 0.656 $\pm$ 0.025 (3) & 0.744 $\pm$ 0.029 (3) & 0.659 $\pm$ 0.046 (9) \\
Eigenscore & 0.612 $\pm$ 0.026 (9) & 0.606 $\pm$ 0.040 (7) & 0.628 $\pm$ 0.025 (6) & 0.603 $\pm$ 0.024 (2) & 0.483 $\pm$ 0.026 (12) & 0.632 $\pm$ 0.027 (6) & 0.737 $\pm$ 0.030 (6) & 0.634 $\pm$ 0.051 (11) \\
SelfCheckGPT & 0.582 $\pm$ 0.026 (11) & 0.593 $\pm$ 0.045 (8) & 0.592 $\pm$ 0.026 (12) & 0.577 $\pm$ 0.026 (5) & 0.502 $\pm$ 0.027 (8) & 0.616 $\pm$ 0.027 (11) & 0.725 $\pm$ 0.031 (9) & 0.573 $\pm$ 0.051 (13) \\
P(True) & 0.483 $\pm$ 0.026 (13) & 0.578 $\pm$ 0.042 (9) & 0.406 $\pm$ 0.025 (13) & 0.513 $\pm$ 0.026 (13) & 0.466 $\pm$ 0.026 (13) & 0.563 $\pm$ 0.028 (12) & 0.500 $\pm$ 0.034 (13) & 0.574 $\pm$ 0.048 (12) \\
FAVA & 0.495 $\pm$ 0.027 (12) & 0.661 $\pm$ 0.039 (2) & 0.664 $\pm$ 0.025 (2) & 0.536 $\pm$ 0.027 (12) & 0.519 $\pm$ 0.027 (2) & 0.468 $\pm$ 0.027 (13) & 0.701 $\pm$ 0.032 (12) & 0.662 $\pm$ 0.045 (7) \\
EmbRegress & 0.617 $\pm$ 0.026 (8) & 0.781 $\pm$ 0.032 (1)$^*$ & 0.677 $\pm$ 0.025 (1)$^*$ & 0.536 $\pm$ 0.026 (11) & 0.532 $\pm$ 0.028 (1)$^*$ & 0.799 $\pm$ 0.043 (1)$^*$ & 0.739 $\pm$ 0.031 (5) & 0.773 $\pm$ 0.040 (1)$^*$ \\
LN-Entropy & 0.686 $\pm$ 0.030 (1)$^*$ & 0.549 $\pm$ 0.042 (12) & 0.608 $\pm$ 0.026 (9) & 0.540 $\pm$ 0.025 (9) & 0.502 $\pm$ 0.030 (6) & 0.627 $\pm$ 0.027 (9) & 0.722 $\pm$ 0.030 (10) & 0.670 $\pm$ 0.042 (5) \\
Perplexity & 0.686 $\pm$ 0.031 (2) & 0.549 $\pm$ 0.042 (13) & 0.608 $\pm$ 0.027 (10) & 0.540 $\pm$ 0.024 (10) & 0.502 $\pm$ 0.027 (7) & 0.627 $\pm$ 0.027 (10) & 0.722 $\pm$ 0.028 (11) & 0.670 $\pm$ 0.042 (6) \\
\bottomrule
\end{tabular}}
\end{table}

\begin{table}[H]
\centering
\caption{AUROC $\pm$ std (rank) for \texttt{Llama-3.2-1B-Instruct} across datasets.}
\label{tab:auroc_Llama321BIns}
\resizebox{\textwidth}{!}{%
\begin{tabular}{lcccccccc}
\toprule
Method & bioasq & coqa & drop & gsm8k & nq\_open & squad & svamp & triviaqa \\
\midrule
CES (unsup) & 0.564 $\pm$ 0.027 (9) & 0.666 $\pm$ 0.028 (7) & 0.678 $\pm$ 0.026 (1)$^*$ & 0.508 $\pm$ 0.032 (10) & 0.561 $\pm$ 0.029 (5) & 0.633 $\pm$ 0.027 (4) & 0.756 $\pm$ 0.027 (1)$^*$ & 0.762 $\pm$ 0.025 (1)$^*$ \\
CES (sup) & 0.563 $\pm$ 0.026 (10) & 0.664 $\pm$ 0.030 (8) & 0.675 $\pm$ 0.025 (2) & 0.508 $\pm$ 0.032 (11) & 0.560 $\pm$ 0.028 (6) & 0.633 $\pm$ 0.026 (5) & 0.753 $\pm$ 0.027 (3) & 0.761 $\pm$ 0.027 (2) \\
SE & 0.529 $\pm$ 0.027 (12) & 0.696 $\pm$ 0.029 (4) & 0.634 $\pm$ 0.026 (8) & 0.544 $\pm$ 0.033 (5) & 0.527 $\pm$ 0.027 (12) & 0.595 $\pm$ 0.026 (11) & 0.750 $\pm$ 0.028 (5) & 0.675 $\pm$ 0.030 (11) \\
DSE & 0.572 $\pm$ 0.027 (5) & 0.707 $\pm$ 0.026 (1)$^*$ & 0.655 $\pm$ 0.026 (7) & 0.592 $\pm$ 0.031 (1)$^*$ & 0.559 $\pm$ 0.027 (7) & 0.627 $\pm$ 0.024 (6) & 0.731 $\pm$ 0.029 (8) & 0.738 $\pm$ 0.028 (6) \\
KLE & 0.580 $\pm$ 0.027 (3) & 0.696 $\pm$ 0.026 (3) & 0.664 $\pm$ 0.026 (3) & 0.580 $\pm$ 0.032 (3) & 0.567 $\pm$ 0.028 (3) & 0.650 $\pm$ 0.025 (2) & 0.753 $\pm$ 0.029 (4) & 0.737 $\pm$ 0.028 (7) \\
KLE-Full & 0.580 $\pm$ 0.027 (4) & 0.695 $\pm$ 0.026 (5) & 0.661 $\pm$ 0.026 (6) & 0.581 $\pm$ 0.032 (2) & 0.568 $\pm$ 0.028 (2) & 0.651 $\pm$ 0.025 (1)$^*$ & 0.754 $\pm$ 0.029 (2) & 0.736 $\pm$ 0.027 (8) \\
Eigenscore & 0.584 $\pm$ 0.026 (2) & 0.689 $\pm$ 0.029 (6) & 0.610 $\pm$ 0.027 (11) & 0.511 $\pm$ 0.030 (9) & 0.564 $\pm$ 0.028 (4) & 0.626 $\pm$ 0.025 (9) & 0.684 $\pm$ 0.030 (10) & 0.741 $\pm$ 0.027 (5) \\
SelfCheckGPT & 0.571 $\pm$ 0.027 (6) & 0.609 $\pm$ 0.029 (11) & 0.617 $\pm$ 0.026 (10) & 0.540 $\pm$ 0.032 (6) & 0.534 $\pm$ 0.027 (11) & 0.618 $\pm$ 0.024 (10) & 0.619 $\pm$ 0.032 (11) & 0.682 $\pm$ 0.026 (10) \\
P(True) & 0.527 $\pm$ 0.026 (13) & 0.527 $\pm$ 0.032 (13) & 0.493 $\pm$ 0.028 (13) & 0.575 $\pm$ 0.032 (4) & 0.498 $\pm$ 0.029 (13) & 0.526 $\pm$ 0.029 (12) & 0.513 $\pm$ 0.034 (13) & 0.459 $\pm$ 0.032 (13) \\
FAVA & 0.536 $\pm$ 0.027 (11) & 0.598 $\pm$ 0.029 (12) & 0.586 $\pm$ 0.027 (12) & 0.526 $\pm$ 0.030 (7) & 0.539 $\pm$ 0.028 (10) & 0.525 $\pm$ 0.027 (13) & 0.548 $\pm$ 0.034 (12) & 0.618 $\pm$ 0.031 (12) \\
EmbRegress & 0.628 $\pm$ 0.025 (1)$^*$ & 0.702 $\pm$ 0.027 (2) & 0.632 $\pm$ 0.027 (9) & 0.524 $\pm$ 0.031 (8) & 0.601 $\pm$ 0.026 (1)$^*$ & 0.647 $\pm$ 0.043 (3) & 0.707 $\pm$ 0.031 (9) & 0.711 $\pm$ 0.027 (9) \\
LN-Entropy & 0.564 $\pm$ 0.026 (7) & 0.644 $\pm$ 0.031 (9) & 0.662 $\pm$ 0.025 (4) & 0.501 $\pm$ 0.033 (12) & 0.554 $\pm$ 0.029 (8) & 0.627 $\pm$ 0.027 (7) & 0.741 $\pm$ 0.029 (6) & 0.748 $\pm$ 0.026 (3) \\
Perplexity & 0.564 $\pm$ 0.025 (8) & 0.644 $\pm$ 0.030 (10) & 0.662 $\pm$ 0.025 (5) & 0.501 $\pm$ 0.031 (13) & 0.554 $\pm$ 0.026 (9) & 0.627 $\pm$ 0.027 (8) & 0.741 $\pm$ 0.028 (7) & 0.748 $\pm$ 0.026 (4) \\
\bottomrule
\end{tabular}}
\end{table}

\begin{table}[H]
\centering
\caption{AUROC $\pm$ std (rank) for \texttt{Meta-Llama-3-8B-Instruct} across datasets.}
\label{tab:auroc_MetaLlama38B}
\resizebox{\textwidth}{!}{%
\begin{tabular}{lcccccccc}
\toprule
Method & bioasq & coqa & drop & gsm8k & nq\_open & squad & svamp & triviaqa \\
\midrule
CES (unsup) & 0.584 $\pm$ 0.028 (7) & 0.694 $\pm$ 0.043 (5) & 0.627 $\pm$ 0.026 (7) & 0.685 $\pm$ 0.033 (2) & 0.568 $\pm$ 0.026 (4) & 0.650 $\pm$ 0.029 (10) & 0.792 $\pm$ 0.028 (4) & 0.719 $\pm$ 0.049 (9) \\
CES (sup) & 0.584 $\pm$ 0.030 (6) & 0.693 $\pm$ 0.046 (6) & 0.627 $\pm$ 0.025 (8) & 0.684 $\pm$ 0.034 (3) & 0.567 $\pm$ 0.026 (5) & 0.650 $\pm$ 0.030 (11) & 0.790 $\pm$ 0.029 (5) & 0.719 $\pm$ 0.047 (8) \\
SE & 0.567 $\pm$ 0.030 (11) & 0.702 $\pm$ 0.050 (4) & 0.633 $\pm$ 0.026 (6) & 0.525 $\pm$ 0.039 (13) & 0.544 $\pm$ 0.026 (11) & 0.657 $\pm$ 0.027 (7) & 0.767 $\pm$ 0.030 (8) & 0.703 $\pm$ 0.052 (11) \\
DSE & 0.581 $\pm$ 0.029 (10) & 0.680 $\pm$ 0.045 (9) & 0.638 $\pm$ 0.025 (5) & 0.589 $\pm$ 0.039 (10) & 0.561 $\pm$ 0.027 (9) & 0.667 $\pm$ 0.028 (5) & 0.798 $\pm$ 0.027 (3) & 0.723 $\pm$ 0.048 (7) \\
KLE & 0.592 $\pm$ 0.029 (2) & 0.739 $\pm$ 0.043 (1)$^*$ & 0.652 $\pm$ 0.025 (2) & 0.651 $\pm$ 0.038 (6) & 0.557 $\pm$ 0.026 (10) & 0.672 $\pm$ 0.028 (2) & 0.806 $\pm$ 0.026 (1)$^*$ & 0.734 $\pm$ 0.047 (4) \\
KLE-Full & 0.591 $\pm$ 0.030 (3) & 0.739 $\pm$ 0.041 (2) & 0.650 $\pm$ 0.026 (3) & 0.645 $\pm$ 0.037 (8) & 0.562 $\pm$ 0.025 (8) & 0.670 $\pm$ 0.029 (4) & 0.800 $\pm$ 0.027 (2) & 0.734 $\pm$ 0.048 (3) \\
Eigenscore & 0.591 $\pm$ 0.027 (4) & 0.675 $\pm$ 0.041 (10) & 0.600 $\pm$ 0.027 (11) & 0.647 $\pm$ 0.034 (7) & 0.567 $\pm$ 0.026 (6) & 0.670 $\pm$ 0.029 (3) & 0.729 $\pm$ 0.030 (10) & 0.744 $\pm$ 0.047 (2) \\
SelfCheckGPT & 0.588 $\pm$ 0.027 (5) & 0.640 $\pm$ 0.040 (12) & 0.589 $\pm$ 0.027 (12) & 0.702 $\pm$ 0.034 (1)$^*$ & 0.563 $\pm$ 0.027 (7) & 0.663 $\pm$ 0.028 (6) & 0.732 $\pm$ 0.031 (9) & 0.709 $\pm$ 0.046 (10) \\
P(True) & 0.519 $\pm$ 0.028 (13) & 0.499 $\pm$ 0.048 (13) & 0.509 $\pm$ 0.027 (13) & 0.582 $\pm$ 0.037 (11) & 0.529 $\pm$ 0.026 (12) & 0.624 $\pm$ 0.029 (12) & 0.562 $\pm$ 0.035 (13) & 0.454 $\pm$ 0.048 (13) \\
FAVA & 0.537 $\pm$ 0.028 (12) & 0.667 $\pm$ 0.043 (11) & 0.641 $\pm$ 0.025 (4) & 0.594 $\pm$ 0.041 (9) & 0.523 $\pm$ 0.026 (13) & 0.504 $\pm$ 0.029 (13) & 0.721 $\pm$ 0.031 (11) & 0.659 $\pm$ 0.043 (12) \\
EmbRegress & 0.601 $\pm$ 0.029 (1)$^*$ & 0.729 $\pm$ 0.039 (3) & 0.687 $\pm$ 0.024 (1)$^*$ & 0.559 $\pm$ 0.038 (12) & 0.576 $\pm$ 0.027 (1)$^*$ & 0.678 $\pm$ 0.071 (1)$^*$ & 0.670 $\pm$ 0.032 (12) & 0.863 $\pm$ 0.031 (1)$^*$ \\
LN-Entropy & 0.582 $\pm$ 0.029 (8) & 0.686 $\pm$ 0.042 (7) & 0.606 $\pm$ 0.027 (9) & 0.668 $\pm$ 0.034 (4) & 0.574 $\pm$ 0.025 (2) & 0.652 $\pm$ 0.028 (8) & 0.781 $\pm$ 0.029 (6) & 0.727 $\pm$ 0.045 (5) \\
Perplexity & 0.582 $\pm$ 0.029 (9) & 0.686 $\pm$ 0.042 (8) & 0.606 $\pm$ 0.027 (10) & 0.668 $\pm$ 0.035 (5) & 0.574 $\pm$ 0.025 (3) & 0.652 $\pm$ 0.029 (9) & 0.781 $\pm$ 0.028 (7) & 0.727 $\pm$ 0.044 (6) \\
\bottomrule
\end{tabular}}
\end{table}

\begin{table}[H]
\centering
\caption{AUROC $\pm$ std (rank) for \texttt{Mistral-7B-Instruct-v0.3} across datasets.}
\label{tab:auroc_Mistral7BIns}
\resizebox{\textwidth}{!}{%
\begin{tabular}{lcccccccc}
\toprule
Method & bioasq & coqa & drop & gsm8k & nq\_open & squad & svamp & triviaqa \\
\midrule
CES (unsup) & 0.601 $\pm$ 0.029 (5) & 0.642 $\pm$ 0.053 (9) & 0.671 $\pm$ 0.024 (7) & 0.656 $\pm$ 0.033 (3) & 0.504 $\pm$ 0.026 (11) & 0.643 $\pm$ 0.029 (2) & 0.720 $\pm$ 0.030 (5) & 0.736 $\pm$ 0.051 (4) \\
CES (sup) & 0.601 $\pm$ 0.028 (4) & 0.641 $\pm$ 0.053 (10) & 0.672 $\pm$ 0.023 (6) & 0.656 $\pm$ 0.030 (4) & 0.504 $\pm$ 0.025 (10) & 0.643 $\pm$ 0.030 (1)$^*$ & 0.719 $\pm$ 0.029 (6) & 0.736 $\pm$ 0.053 (5) \\
SE & 0.558 $\pm$ 0.028 (10) & 0.665 $\pm$ 0.045 (8) & 0.661 $\pm$ 0.025 (10) & 0.522 $\pm$ 0.033 (12) & 0.548 $\pm$ 0.026 (2) & 0.589 $\pm$ 0.030 (10) & 0.641 $\pm$ 0.032 (12) & 0.745 $\pm$ 0.055 (2) \\
DSE & 0.594 $\pm$ 0.027 (6) & 0.688 $\pm$ 0.046 (4) & 0.683 $\pm$ 0.024 (3) & 0.593 $\pm$ 0.032 (9) & 0.534 $\pm$ 0.025 (5) & 0.634 $\pm$ 0.031 (6) & 0.731 $\pm$ 0.030 (4) & 0.724 $\pm$ 0.059 (8) \\
KLE & 0.617 $\pm$ 0.028 (2) & 0.739 $\pm$ 0.044 (2) & 0.696 $\pm$ 0.024 (2) & 0.595 $\pm$ 0.032 (7) & 0.547 $\pm$ 0.024 (3) & 0.623 $\pm$ 0.030 (8) & 0.750 $\pm$ 0.027 (3) & 0.733 $\pm$ 0.064 (6) \\
KLE-Full & 0.616 $\pm$ 0.029 (3) & 0.739 $\pm$ 0.045 (3) & 0.698 $\pm$ 0.023 (1)$^*$ & 0.595 $\pm$ 0.031 (8) & 0.547 $\pm$ 0.026 (4) & 0.625 $\pm$ 0.030 (7) & 0.751 $\pm$ 0.029 (2) & 0.739 $\pm$ 0.060 (3) \\
Eigenscore & 0.632 $\pm$ 0.028 (1)$^*$ & 0.682 $\pm$ 0.045 (5) & 0.654 $\pm$ 0.024 (11) & 0.602 $\pm$ 0.030 (5) & 0.531 $\pm$ 0.026 (6) & 0.618 $\pm$ 0.030 (9) & 0.686 $\pm$ 0.031 (9) & 0.691 $\pm$ 0.073 (11) \\
SelfCheckGPT & 0.582 $\pm$ 0.028 (9) & 0.676 $\pm$ 0.043 (6) & 0.645 $\pm$ 0.025 (12) & 0.600 $\pm$ 0.031 (6) & 0.493 $\pm$ 0.026 (13) & 0.638 $\pm$ 0.030 (5) & 0.674 $\pm$ 0.031 (10) & 0.664 $\pm$ 0.066 (12) \\
P(True) & 0.552 $\pm$ 0.029 (12) & 0.508 $\pm$ 0.045 (13) & 0.476 $\pm$ 0.026 (13) & 0.498 $\pm$ 0.034 (13) & 0.510 $\pm$ 0.027 (7) & 0.537 $\pm$ 0.031 (12) & 0.504 $\pm$ 0.033 (13) & 0.486 $\pm$ 0.059 (13) \\
FAVA & 0.519 $\pm$ 0.029 (13) & 0.665 $\pm$ 0.042 (7) & 0.679 $\pm$ 0.024 (4) & 0.591 $\pm$ 0.032 (10) & 0.502 $\pm$ 0.026 (12) & 0.510 $\pm$ 0.029 (13) & 0.761 $\pm$ 0.028 (1)$^*$ & 0.724 $\pm$ 0.056 (7) \\
EmbRegress & 0.554 $\pm$ 0.026 (11) & 0.803 $\pm$ 0.032 (1)$^*$ & 0.676 $\pm$ 0.024 (5) & 0.547 $\pm$ 0.034 (11) & 0.575 $\pm$ 0.025 (1)$^*$ & 0.582 $\pm$ 0.059 (11) & 0.667 $\pm$ 0.030 (11) & 0.817 $\pm$ 0.050 (1)$^*$ \\
LN-Entropy & 0.590 $\pm$ 0.029 (7) & 0.635 $\pm$ 0.053 (11) & 0.667 $\pm$ 0.025 (8) & 0.657 $\pm$ 0.029 (1)$^*$ & 0.504 $\pm$ 0.025 (8) & 0.640 $\pm$ 0.031 (3) & 0.704 $\pm$ 0.030 (7) & 0.722 $\pm$ 0.055 (9) \\
Perplexity & 0.590 $\pm$ 0.028 (8) & 0.635 $\pm$ 0.046 (12) & 0.667 $\pm$ 0.025 (9) & 0.657 $\pm$ 0.031 (2) & 0.504 $\pm$ 0.026 (9) & 0.640 $\pm$ 0.031 (4) & 0.704 $\pm$ 0.032 (8) & 0.722 $\pm$ 0.051 (10) \\
\bottomrule
\end{tabular}}
\end{table}

\begin{table}[H]
\centering
\caption{AUROC $\pm$ std (rank) for \texttt{falcon-7b-instruct} across datasets.}
\label{tab:auroc_falcon7binst}
\resizebox{\textwidth}{!}{%
\begin{tabular}{lcccccccc}
\toprule
Method & bioasq & coqa & drop & gsm8k & nq\_open & squad & svamp & triviaqa \\
\midrule
CES (unsup) & 0.558 $\pm$ 0.033 (8) & 0.700 $\pm$ 0.027 (4) & 0.621 $\pm$ 0.026 (5) & 0.505 $\pm$ 0.032 (5) & 0.573 $\pm$ 0.028 (1)$^*$ & 0.650 $\pm$ 0.025 (3) & 0.723 $\pm$ 0.037 (1)$^*$ & 0.737 $\pm$ 0.031 (6) \\
CES (sup) & 0.557 $\pm$ 0.033 (9) & 0.697 $\pm$ 0.029 (5) & 0.618 $\pm$ 0.027 (7) & 0.505 $\pm$ 0.030 (6) & 0.572 $\pm$ 0.027 (2) & 0.650 $\pm$ 0.026 (4) & 0.718 $\pm$ 0.037 (2) & 0.736 $\pm$ 0.032 (7) \\
SE & 0.577 $\pm$ 0.026 (6) & 0.640 $\pm$ 0.028 (10) & 0.596 $\pm$ 0.025 (10) & 0.454 $\pm$ 0.030 (12) & 0.508 $\pm$ 0.027 (12) & 0.593 $\pm$ 0.026 (10) & 0.503 $\pm$ 0.039 (12) & 0.745 $\pm$ 0.029 (4) \\
DSE & 0.594 $\pm$ 0.026 (4) & 0.696 $\pm$ 0.024 (6) & 0.644 $\pm$ 0.026 (2) & 0.533 $\pm$ 0.029 (3) & 0.561 $\pm$ 0.026 (5) & 0.628 $\pm$ 0.024 (8) & 0.644 $\pm$ 0.038 (5) & 0.759 $\pm$ 0.028 (2) \\
KLE & 0.612 $\pm$ 0.026 (1)$^*$ & 0.718 $\pm$ 0.025 (1)$^*$ & 0.646 $\pm$ 0.026 (1)$^*$ & 0.490 $\pm$ 0.029 (10) & 0.537 $\pm$ 0.026 (9) & 0.643 $\pm$ 0.025 (6) & 0.620 $\pm$ 0.037 (8) & 0.759 $\pm$ 0.029 (3) \\
KLE-Full & 0.607 $\pm$ 0.026 (2) & 0.717 $\pm$ 0.024 (2) & 0.644 $\pm$ 0.026 (3) & 0.490 $\pm$ 0.030 (11) & 0.534 $\pm$ 0.027 (10) & 0.644 $\pm$ 0.024 (5) & 0.623 $\pm$ 0.037 (6) & 0.760 $\pm$ 0.029 (1)$^*$ \\
Eigenscore & 0.590 $\pm$ 0.026 (5) & 0.712 $\pm$ 0.026 (3) & 0.576 $\pm$ 0.028 (11) & 0.509 $\pm$ 0.031 (4) & 0.558 $\pm$ 0.026 (7) & 0.619 $\pm$ 0.025 (9) & 0.554 $\pm$ 0.041 (10) & 0.742 $\pm$ 0.027 (5) \\
SelfCheckGPT & 0.570 $\pm$ 0.026 (7) & 0.630 $\pm$ 0.026 (11) & 0.621 $\pm$ 0.025 (4) & 0.491 $\pm$ 0.031 (9) & 0.542 $\pm$ 0.025 (8) & 0.592 $\pm$ 0.025 (11) & 0.498 $\pm$ 0.038 (13) & 0.701 $\pm$ 0.028 (11) \\
P(True) & 0.469 $\pm$ 0.026 (13) & 0.502 $\pm$ 0.030 (13) & 0.559 $\pm$ 0.025 (12) & 0.453 $\pm$ 0.031 (13) & 0.528 $\pm$ 0.028 (11) & 0.490 $\pm$ 0.027 (12) & 0.550 $\pm$ 0.043 (11) & 0.492 $\pm$ 0.034 (13) \\
FAVA & 0.505 $\pm$ 0.027 (12) & 0.627 $\pm$ 0.026 (12) & 0.620 $\pm$ 0.028 (6) & 0.539 $\pm$ 0.032 (2) & 0.483 $\pm$ 0.028 (13) & 0.473 $\pm$ 0.026 (13) & 0.575 $\pm$ 0.040 (9) & 0.628 $\pm$ 0.032 (12) \\
EmbRegress & 0.595 $\pm$ 0.026 (3) & 0.670 $\pm$ 0.027 (9) & 0.558 $\pm$ 0.029 (13) & 0.554 $\pm$ 0.030 (1)$^*$ & 0.561 $\pm$ 0.028 (6) & 0.638 $\pm$ 0.036 (7) & 0.622 $\pm$ 0.037 (7) & 0.725 $\pm$ 0.028 (8) \\
LN-Entropy & 0.556 $\pm$ 0.033 (10) & 0.672 $\pm$ 0.029 (7) & 0.598 $\pm$ 0.028 (8) & 0.502 $\pm$ 0.031 (7) & 0.566 $\pm$ 0.029 (3) & 0.656 $\pm$ 0.025 (1)$^*$ & 0.687 $\pm$ 0.041 (3) & 0.722 $\pm$ 0.032 (9) \\
Perplexity & 0.556 $\pm$ 0.032 (11) & 0.672 $\pm$ 0.028 (8) & 0.598 $\pm$ 0.027 (9) & 0.502 $\pm$ 0.030 (8) & 0.566 $\pm$ 0.028 (4) & 0.656 $\pm$ 0.025 (2) & 0.687 $\pm$ 0.041 (4) & 0.722 $\pm$ 0.033 (10) \\
\bottomrule
\end{tabular}}
\end{table}

\begin{table}[H]
\centering
\caption{AUROC $\pm$ std (rank) for \texttt{gpt-4.1-2025-04-14} across datasets.}
\label{tab:auroc_gpt412025041}
\resizebox{\textwidth}{!}{%
\begin{tabular}{lcccccccc}
\toprule
Method & bioasq & coqa & drop & gsm8k & nq\_open & squad & svamp & triviaqa \\
\midrule
CES (unsup) & 0.635 $\pm$ 0.034 (2) & 0.634 $\pm$ 0.074 (7) & 0.639 $\pm$ 0.029 (1)$^*$ & 0.769 $\pm$ 0.039 (1)$^*$ & 0.642 $\pm$ 0.026 (6) & 0.761 $\pm$ 0.021 (1)$^*$ & 0.545 $\pm$ 0.084 (9) & 0.618 $\pm$ 0.052 (3) \\
CES (sup) & 0.634 $\pm$ 0.032 (3) & 0.634 $\pm$ 0.076 (8) & 0.637 $\pm$ 0.028 (3) & 0.769 $\pm$ 0.040 (2) & 0.642 $\pm$ 0.026 (5) & 0.761 $\pm$ 0.022 (2) & 0.543 $\pm$ 0.084 (10) & 0.618 $\pm$ 0.049 (4) \\
SE & 0.567 $\pm$ 0.036 (11) & 0.565 $\pm$ 0.055 (12) & 0.614 $\pm$ 0.029 (10) & 0.469 $\pm$ 0.062 (13) & 0.608 $\pm$ 0.026 (11) & 0.630 $\pm$ 0.027 (11) & 0.590 $\pm$ 0.081 (7) & 0.459 $\pm$ 0.044 (13) \\
DSE & 0.591 $\pm$ 0.032 (9) & 0.543 $\pm$ 0.037 (13) & 0.619 $\pm$ 0.021 (7) & 0.479 $\pm$ 0.005 (12) & 0.600 $\pm$ 0.022 (12) & 0.646 $\pm$ 0.021 (9) & 0.634 $\pm$ 0.055 (3) & 0.535 $\pm$ 0.027 (11) \\
KLE & 0.596 $\pm$ 0.036 (7) & 0.741 $\pm$ 0.057 (2) & 0.631 $\pm$ 0.028 (5) & 0.536 $\pm$ 0.047 (10) & 0.632 $\pm$ 0.025 (7) & 0.652 $\pm$ 0.026 (8) & 0.599 $\pm$ 0.079 (4) & 0.591 $\pm$ 0.043 (8) \\
KLE-Full & 0.592 $\pm$ 0.036 (8) & 0.717 $\pm$ 0.055 (4) & 0.632 $\pm$ 0.029 (4) & 0.536 $\pm$ 0.050 (9) & 0.630 $\pm$ 0.025 (9) & 0.653 $\pm$ 0.026 (7) & 0.591 $\pm$ 0.077 (6) & 0.590 $\pm$ 0.041 (9) \\
Eigenscore & 0.610 $\pm$ 0.040 (6) & 0.759 $\pm$ 0.054 (1)$^*$ & 0.611 $\pm$ 0.030 (11) & 0.552 $\pm$ 0.058 (8) & 0.612 $\pm$ 0.026 (10) & 0.642 $\pm$ 0.029 (10) & 0.551 $\pm$ 0.074 (8) & 0.526 $\pm$ 0.050 (12) \\
SelfCheckGPT & 0.567 $\pm$ 0.035 (12) & 0.608 $\pm$ 0.058 (9) & 0.595 $\pm$ 0.027 (12) & 0.608 $\pm$ 0.037 (6) & 0.560 $\pm$ 0.026 (13) & 0.663 $\pm$ 0.025 (6) & 0.518 $\pm$ 0.078 (11) & 0.615 $\pm$ 0.044 (7) \\
P(True) & 0.750 $\pm$ 0.033 (1)$^*$ & 0.589 $\pm$ 0.062 (11) & 0.637 $\pm$ 0.029 (2) & 0.626 $\pm$ 0.076 (5) & 0.649 $\pm$ 0.025 (2) & 0.730 $\pm$ 0.025 (5) & 0.712 $\pm$ 0.065 (2) & 0.586 $\pm$ 0.053 (10) \\
FAVA & 0.538 $\pm$ 0.034 (13) & 0.607 $\pm$ 0.050 (10) & 0.536 $\pm$ 0.030 (13) & 0.589 $\pm$ 0.054 (7) & 0.630 $\pm$ 0.025 (8) & 0.496 $\pm$ 0.029 (13) & 0.597 $\pm$ 0.064 (5) & 0.657 $\pm$ 0.041 (2) \\
EmbRegress & 0.573 $\pm$ 0.028 (10) & 0.736 $\pm$ 0.050 (3) & 0.623 $\pm$ 0.028 (6) & 0.516 $\pm$ 0.050 (11) & 0.773 $\pm$ 0.021 (1)$^*$ & 0.601 $\pm$ 0.076 (12) & 0.738 $\pm$ 0.052 (1)$^*$ & 0.767 $\pm$ 0.041 (1)$^*$ \\
LN-Entropy & 0.626 $\pm$ 0.035 (4) & 0.646 $\pm$ 0.077 (5) & 0.616 $\pm$ 0.028 (8) & 0.756 $\pm$ 0.039 (3) & 0.644 $\pm$ 0.024 (3) & 0.760 $\pm$ 0.022 (3) & 0.515 $\pm$ 0.081 (12) & 0.617 $\pm$ 0.046 (5) \\
Perplexity & 0.626 $\pm$ 0.033 (5) & 0.646 $\pm$ 0.072 (6) & 0.616 $\pm$ 0.028 (9) & 0.756 $\pm$ 0.038 (4) & 0.644 $\pm$ 0.026 (4) & 0.760 $\pm$ 0.022 (4) & 0.515 $\pm$ 0.081 (13) & 0.617 $\pm$ 0.049 (6) \\
\bottomrule
\end{tabular}}
\end{table}

\begin{table}[H]
\centering
\caption{AUROC $\pm$ std (rank) for \texttt{gpt-4.1-mini-2025-04-14} across datasets.}
\label{tab:auroc_gpt41mini202}
\resizebox{\textwidth}{!}{%
\begin{tabular}{lcccccccc}
\toprule
Method & bioasq & coqa & drop & gsm8k & nq\_open & squad & svamp & triviaqa \\
\midrule
CES (unsup) & 0.691 $\pm$ 0.029 (3) & 0.614 $\pm$ 0.059 (8) & 0.642 $\pm$ 0.025 (1)$^*$ & 0.638 $\pm$ 0.053 (5) & 0.669 $\pm$ 0.024 (6) & 0.732 $\pm$ 0.025 (3) & 0.668 $\pm$ 0.073 (8) & 0.630 $\pm$ 0.042 (2) \\
CES (sup) & 0.691 $\pm$ 0.033 (4) & 0.613 $\pm$ 0.059 (9) & 0.642 $\pm$ 0.027 (2) & 0.638 $\pm$ 0.057 (6) & 0.669 $\pm$ 0.024 (7) & 0.732 $\pm$ 0.026 (4) & 0.668 $\pm$ 0.069 (9) & 0.629 $\pm$ 0.041 (3) \\
SE & 0.674 $\pm$ 0.035 (9) & 0.605 $\pm$ 0.045 (11) & 0.564 $\pm$ 0.028 (11) & 0.512 $\pm$ 0.048 (9) & 0.731 $\pm$ 0.023 (1)$^*$ & 0.628 $\pm$ 0.031 (11) & 0.661 $\pm$ 0.071 (10) & 0.600 $\pm$ 0.035 (7) \\
DSE & 0.685 $\pm$ 0.032 (6) & 0.592 $\pm$ 0.035 (12) & 0.582 $\pm$ 0.019 (10) & 0.537 $\pm$ 0.028 (7) & 0.635 $\pm$ 0.021 (11) & 0.645 $\pm$ 0.024 (10) & 0.734 $\pm$ 0.057 (3) & 0.571 $\pm$ 0.033 (13) \\
KLE & 0.714 $\pm$ 0.036 (1)$^*$ & 0.685 $\pm$ 0.052 (3) & 0.637 $\pm$ 0.025 (3) & 0.505 $\pm$ 0.040 (10) & 0.641 $\pm$ 0.023 (8) & 0.707 $\pm$ 0.028 (7) & 0.726 $\pm$ 0.065 (4) & 0.599 $\pm$ 0.044 (8) \\
KLE-Full & 0.708 $\pm$ 0.036 (2) & 0.671 $\pm$ 0.053 (4) & 0.635 $\pm$ 0.026 (4) & 0.504 $\pm$ 0.039 (11) & 0.640 $\pm$ 0.025 (10) & 0.707 $\pm$ 0.027 (6) & 0.725 $\pm$ 0.066 (5) & 0.590 $\pm$ 0.044 (9) \\
Eigenscore & 0.668 $\pm$ 0.039 (10) & 0.720 $\pm$ 0.051 (2) & 0.615 $\pm$ 0.029 (7) & 0.517 $\pm$ 0.043 (8) & 0.597 $\pm$ 0.025 (12) & 0.664 $\pm$ 0.030 (9) & 0.682 $\pm$ 0.071 (7) & 0.581 $\pm$ 0.044 (12) \\
SelfCheckGPT & 0.581 $\pm$ 0.032 (11) & 0.611 $\pm$ 0.055 (10) & 0.554 $\pm$ 0.028 (12) & 0.487 $\pm$ 0.046 (12) & 0.589 $\pm$ 0.026 (13) & 0.611 $\pm$ 0.025 (12) & 0.648 $\pm$ 0.065 (13) & 0.588 $\pm$ 0.042 (10) \\
P(True) & 0.686 $\pm$ 0.042 (5) & 0.660 $\pm$ 0.049 (5) & 0.592 $\pm$ 0.031 (9) & 0.650 $\pm$ 0.048 (4) & 0.679 $\pm$ 0.023 (3) & 0.730 $\pm$ 0.029 (5) & 0.748 $\pm$ 0.059 (2) & 0.627 $\pm$ 0.049 (4) \\
FAVA & 0.529 $\pm$ 0.033 (13) & 0.542 $\pm$ 0.050 (13) & 0.515 $\pm$ 0.030 (13) & 0.692 $\pm$ 0.043 (1)$^*$ & 0.640 $\pm$ 0.024 (9) & 0.544 $\pm$ 0.029 (13) & 0.695 $\pm$ 0.063 (6) & 0.587 $\pm$ 0.040 (11) \\
EmbRegress & 0.562 $\pm$ 0.030 (12) & 0.786 $\pm$ 0.039 (1)$^*$ & 0.612 $\pm$ 0.028 (8) & 0.461 $\pm$ 0.041 (13) & 0.730 $\pm$ 0.021 (2) & 0.665 $\pm$ 0.065 (8) & 0.762 $\pm$ 0.042 (1)$^*$ & 0.801 $\pm$ 0.033 (1)$^*$ \\
LN-Entropy & 0.677 $\pm$ 0.034 (7) & 0.618 $\pm$ 0.058 (6) & 0.628 $\pm$ 0.029 (5) & 0.652 $\pm$ 0.053 (2) & 0.677 $\pm$ 0.023 (4) & 0.737 $\pm$ 0.025 (1)$^*$ & 0.658 $\pm$ 0.074 (11) & 0.626 $\pm$ 0.040 (5) \\
Perplexity & 0.677 $\pm$ 0.033 (8) & 0.618 $\pm$ 0.063 (7) & 0.628 $\pm$ 0.027 (6) & 0.652 $\pm$ 0.049 (3) & 0.677 $\pm$ 0.024 (5) & 0.737 $\pm$ 0.026 (2) & 0.658 $\pm$ 0.069 (12) & 0.626 $\pm$ 0.043 (6) \\
\bottomrule
\end{tabular}}
\end{table}

\begin{table}[H]
\centering
\caption{AUROC $\pm$ std (rank) for \texttt{gpt-4.1-nano-2025-04-14} across datasets.}
\label{tab:auroc_gpt41nano202}
\resizebox{\textwidth}{!}{%
\begin{tabular}{lcccccccc}
\toprule
Method & bioasq & coqa & drop & gsm8k & nq\_open & squad & svamp & triviaqa \\
\midrule
CES (unsup) & 0.714 $\pm$ 0.029 (6) & 0.693 $\pm$ 0.046 (6) & 0.613 $\pm$ 0.025 (8) & 0.726 $\pm$ 0.036 (5) & 0.723 $\pm$ 0.021 (7) & 0.681 $\pm$ 0.025 (4) & 0.680 $\pm$ 0.041 (10) & 0.781 $\pm$ 0.032 (2) \\
CES (sup) & 0.713 $\pm$ 0.027 (7) & 0.694 $\pm$ 0.044 (5) & 0.613 $\pm$ 0.026 (9) & 0.727 $\pm$ 0.036 (4) & 0.722 $\pm$ 0.022 (8) & 0.681 $\pm$ 0.026 (5) & 0.680 $\pm$ 0.040 (11) & 0.781 $\pm$ 0.035 (3) \\
SE & 0.730 $\pm$ 0.029 (5) & 0.566 $\pm$ 0.044 (13) & 0.626 $\pm$ 0.025 (7) & 0.538 $\pm$ 0.041 (12) & 0.740 $\pm$ 0.023 (4) & 0.630 $\pm$ 0.028 (12) & 0.697 $\pm$ 0.043 (6) & 0.629 $\pm$ 0.039 (13) \\
DSE & 0.738 $\pm$ 0.028 (4) & 0.605 $\pm$ 0.035 (12) & 0.654 $\pm$ 0.024 (4) & 0.607 $\pm$ 0.030 (7) & 0.734 $\pm$ 0.021 (6) & 0.656 $\pm$ 0.026 (9) & 0.764 $\pm$ 0.035 (1)$^*$ & 0.670 $\pm$ 0.033 (12) \\
KLE & 0.785 $\pm$ 0.028 (1)$^*$ & 0.648 $\pm$ 0.047 (8) & 0.685 $\pm$ 0.025 (2) & 0.549 $\pm$ 0.039 (9) & 0.758 $\pm$ 0.021 (3) & 0.709 $\pm$ 0.027 (2) & 0.752 $\pm$ 0.039 (3) & 0.736 $\pm$ 0.035 (10) \\
KLE-Full & 0.777 $\pm$ 0.027 (2) & 0.637 $\pm$ 0.043 (10) & 0.679 $\pm$ 0.025 (3) & 0.546 $\pm$ 0.038 (10) & 0.758 $\pm$ 0.021 (2) & 0.710 $\pm$ 0.027 (1)$^*$ & 0.748 $\pm$ 0.040 (4) & 0.737 $\pm$ 0.035 (9) \\
Eigenscore & 0.707 $\pm$ 0.031 (8) & 0.611 $\pm$ 0.046 (11) & 0.609 $\pm$ 0.026 (12) & 0.503 $\pm$ 0.041 (13) & 0.699 $\pm$ 0.022 (11) & 0.641 $\pm$ 0.028 (11) & 0.640 $\pm$ 0.043 (13) & 0.745 $\pm$ 0.034 (8) \\
SelfCheckGPT & 0.592 $\pm$ 0.031 (11) & 0.663 $\pm$ 0.042 (7) & 0.633 $\pm$ 0.026 (6) & 0.605 $\pm$ 0.031 (8) & 0.636 $\pm$ 0.026 (13) & 0.653 $\pm$ 0.026 (10) & 0.686 $\pm$ 0.038 (9) & 0.769 $\pm$ 0.029 (6) \\
P(True) & 0.768 $\pm$ 0.030 (3) & 0.754 $\pm$ 0.042 (2) & 0.703 $\pm$ 0.026 (1)$^*$ & 0.834 $\pm$ 0.028 (1)$^*$ & 0.735 $\pm$ 0.021 (5) & 0.695 $\pm$ 0.028 (3) & 0.753 $\pm$ 0.040 (2) & 0.751 $\pm$ 0.036 (7) \\
FAVA & 0.540 $\pm$ 0.032 (12) & 0.643 $\pm$ 0.039 (9) & 0.583 $\pm$ 0.028 (13) & 0.615 $\pm$ 0.040 (6) & 0.647 $\pm$ 0.024 (12) & 0.570 $\pm$ 0.028 (13) & 0.703 $\pm$ 0.036 (5) & 0.690 $\pm$ 0.029 (11) \\
EmbRegress & 0.518 $\pm$ 0.028 (13) & 0.790 $\pm$ 0.030 (1)$^*$ & 0.637 $\pm$ 0.025 (5) & 0.544 $\pm$ 0.038 (11) & 0.801 $\pm$ 0.020 (1)$^*$ & 0.678 $\pm$ 0.076 (8) & 0.652 $\pm$ 0.039 (12) & 0.845 $\pm$ 0.026 (1)$^*$ \\
LN-Entropy & 0.696 $\pm$ 0.030 (9) & 0.699 $\pm$ 0.044 (3) & 0.612 $\pm$ 0.028 (10) & 0.749 $\pm$ 0.033 (2) & 0.720 $\pm$ 0.024 (9) & 0.681 $\pm$ 0.027 (6) & 0.689 $\pm$ 0.042 (7) & 0.770 $\pm$ 0.033 (4) \\
Perplexity & 0.696 $\pm$ 0.029 (10) & 0.699 $\pm$ 0.042 (4) & 0.612 $\pm$ 0.025 (11) & 0.749 $\pm$ 0.032 (3) & 0.720 $\pm$ 0.023 (10) & 0.681 $\pm$ 0.026 (7) & 0.689 $\pm$ 0.041 (8) & 0.770 $\pm$ 0.031 (5) \\
\bottomrule
\end{tabular}}
\end{table}

\begin{table}[H]
\centering
\caption{AUROC $\pm$ std (rank) for \texttt{gpt-4o-mini-2024-07-18} across datasets.}
\label{tab:auroc_gpt4omini202}
\resizebox{\textwidth}{!}{%
\begin{tabular}{lcccccccc}
\toprule
Method & bioasq & coqa & drop & gsm8k & nq\_open & squad & svamp & triviaqa \\
\midrule
CES (unsup) & 0.698 $\pm$ 0.027 (2) & 0.661 $\pm$ 0.063 (4) & 0.621 $\pm$ 0.031 (4) & 0.669 $\pm$ 0.040 (4) & 0.700 $\pm$ 0.023 (3) & 0.744 $\pm$ 0.023 (5) & 0.692 $\pm$ 0.056 (6) & 0.636 $\pm$ 0.046 (3) \\
CES (sup) & 0.697 $\pm$ 0.028 (3) & 0.661 $\pm$ 0.062 (5) & 0.620 $\pm$ 0.030 (5) & 0.669 $\pm$ 0.042 (5) & 0.700 $\pm$ 0.024 (4) & 0.744 $\pm$ 0.024 (4) & 0.692 $\pm$ 0.055 (5) & 0.635 $\pm$ 0.045 (4) \\
SE & 0.597 $\pm$ 0.033 (9) & 0.530 $\pm$ 0.062 (13) & 0.596 $\pm$ 0.028 (12) & 0.536 $\pm$ 0.037 (11) & 0.645 $\pm$ 0.025 (11) & 0.602 $\pm$ 0.029 (11) & 0.560 $\pm$ 0.067 (13) & 0.555 $\pm$ 0.043 (12) \\
DSE & 0.636 $\pm$ 0.031 (8) & 0.597 $\pm$ 0.040 (12) & 0.611 $\pm$ 0.022 (7) & 0.554 $\pm$ 0.021 (10) & 0.672 $\pm$ 0.022 (8) & 0.665 $\pm$ 0.023 (8) & 0.675 $\pm$ 0.045 (10) & 0.580 $\pm$ 0.032 (8) \\
KLE & 0.643 $\pm$ 0.034 (7) & 0.651 $\pm$ 0.054 (8) & 0.660 $\pm$ 0.026 (2) & 0.576 $\pm$ 0.029 (9) & 0.671 $\pm$ 0.024 (9) & 0.670 $\pm$ 0.027 (6) & 0.685 $\pm$ 0.053 (8) & 0.566 $\pm$ 0.046 (9) \\
KLE-Full & 0.644 $\pm$ 0.034 (6) & 0.646 $\pm$ 0.052 (10) & 0.656 $\pm$ 0.026 (3) & 0.576 $\pm$ 0.030 (8) & 0.672 $\pm$ 0.025 (7) & 0.664 $\pm$ 0.026 (9) & 0.685 $\pm$ 0.056 (7) & 0.564 $\pm$ 0.044 (11) \\
Eigenscore & 0.581 $\pm$ 0.035 (11) & 0.624 $\pm$ 0.063 (11) & 0.601 $\pm$ 0.030 (11) & 0.579 $\pm$ 0.036 (7) & 0.638 $\pm$ 0.025 (12) & 0.617 $\pm$ 0.029 (10) & 0.640 $\pm$ 0.050 (11) & 0.565 $\pm$ 0.048 (10) \\
SelfCheckGPT & 0.589 $\pm$ 0.031 (10) & 0.693 $\pm$ 0.054 (3) & 0.552 $\pm$ 0.027 (13) & 0.475 $\pm$ 0.030 (12) & 0.584 $\pm$ 0.027 (13) & 0.669 $\pm$ 0.025 (7) & 0.590 $\pm$ 0.051 (12) & 0.554 $\pm$ 0.040 (13) \\
P(True) & 0.850 $\pm$ 0.020 (1)$^*$ & 0.717 $\pm$ 0.055 (2) & 0.676 $\pm$ 0.030 (1)$^*$ & 0.602 $\pm$ 0.053 (6) & 0.748 $\pm$ 0.022 (1)$^*$ & 0.804 $\pm$ 0.021 (1)$^*$ & 0.815 $\pm$ 0.047 (1)$^*$ & 0.737 $\pm$ 0.044 (1)$^*$ \\
FAVA & 0.545 $\pm$ 0.032 (12) & 0.647 $\pm$ 0.048 (9) & 0.618 $\pm$ 0.029 (6) & 0.766 $\pm$ 0.035 (1)$^*$ & 0.652 $\pm$ 0.024 (10) & 0.521 $\pm$ 0.027 (13) & 0.678 $\pm$ 0.046 (9) & 0.609 $\pm$ 0.041 (7) \\
EmbRegress & 0.529 $\pm$ 0.028 (13) & 0.792 $\pm$ 0.041 (1)$^*$ & 0.606 $\pm$ 0.028 (8) & 0.456 $\pm$ 0.036 (13) & 0.733 $\pm$ 0.022 (2) & 0.553 $\pm$ 0.066 (12) & 0.784 $\pm$ 0.037 (2) & 0.720 $\pm$ 0.036 (2) \\
LN-Entropy & 0.690 $\pm$ 0.029 (4) & 0.659 $\pm$ 0.063 (6) & 0.601 $\pm$ 0.029 (9) & 0.682 $\pm$ 0.040 (2) & 0.696 $\pm$ 0.024 (5) & 0.748 $\pm$ 0.022 (2) & 0.694 $\pm$ 0.053 (3) & 0.622 $\pm$ 0.044 (5) \\
Perplexity & 0.690 $\pm$ 0.027 (5) & 0.659 $\pm$ 0.061 (7) & 0.601 $\pm$ 0.029 (10) & 0.682 $\pm$ 0.038 (3) & 0.696 $\pm$ 0.024 (6) & 0.748 $\pm$ 0.022 (3) & 0.694 $\pm$ 0.056 (4) & 0.622 $\pm$ 0.045 (6) \\
\bottomrule
\end{tabular}}
\end{table}





\subsection{Independence Assumption and Autocorrelation}\label{app:exp_independence}

\textbf{Setup.}
The KS test assumes independent observations.
Token entropies in autoregressive generation may exhibit serial dependence due to attention patterns and positional encoding effects.
We compute the lag-1 autocorrelation $\rho_1$ for all 18{,}705 generated sequences across all 80 experiments (10 models $\times$ 8 datasets $\times$ ${\sim}$500 samples per cell) and derive the effective sample size ratio $n_{\text{eff}}/n = (1-\rho_1)/(1+\rho_1)$.

\textbf{Results.}
The median lag-1 autocorrelation is $\rho_1 = 0.061$, with 2{,}947/18{,}705 sequences (15.8\%) exhibiting $|\rho_1| > 0.3$.
The fraction of sequences whose autocorrelation falls within the $\pm 1.96/\sqrt{n}$
white-noise bounds
is 80.5\% (15{,}055/18{,}705), indicating that approximately one-fifth of sequences exhibit statistically significant serial dependence.

The practical implication is modest: at $\rho_1 = 0.061$, the effective sample size ratio is $n_{\text{eff}}/n \approx 0.89$, reducing statistical power by only ${\sim}11\%$.
Compared to initial expectations, serial dependence is mild for the majority of sequences.
The KS test remains valid for Type-I error control (it is distribution-free), and exhibits only marginally reduced power for the ${\sim}16\%$ of sequences with $|\rho_1| > 0.3$.
Nevertheless, future work should investigate block-permutation or rank-based tests that explicitly account for serial dependence in the high-autocorrelation tail.

\begin{figure}[H]
    \centering
    \includegraphics[width=\textwidth]{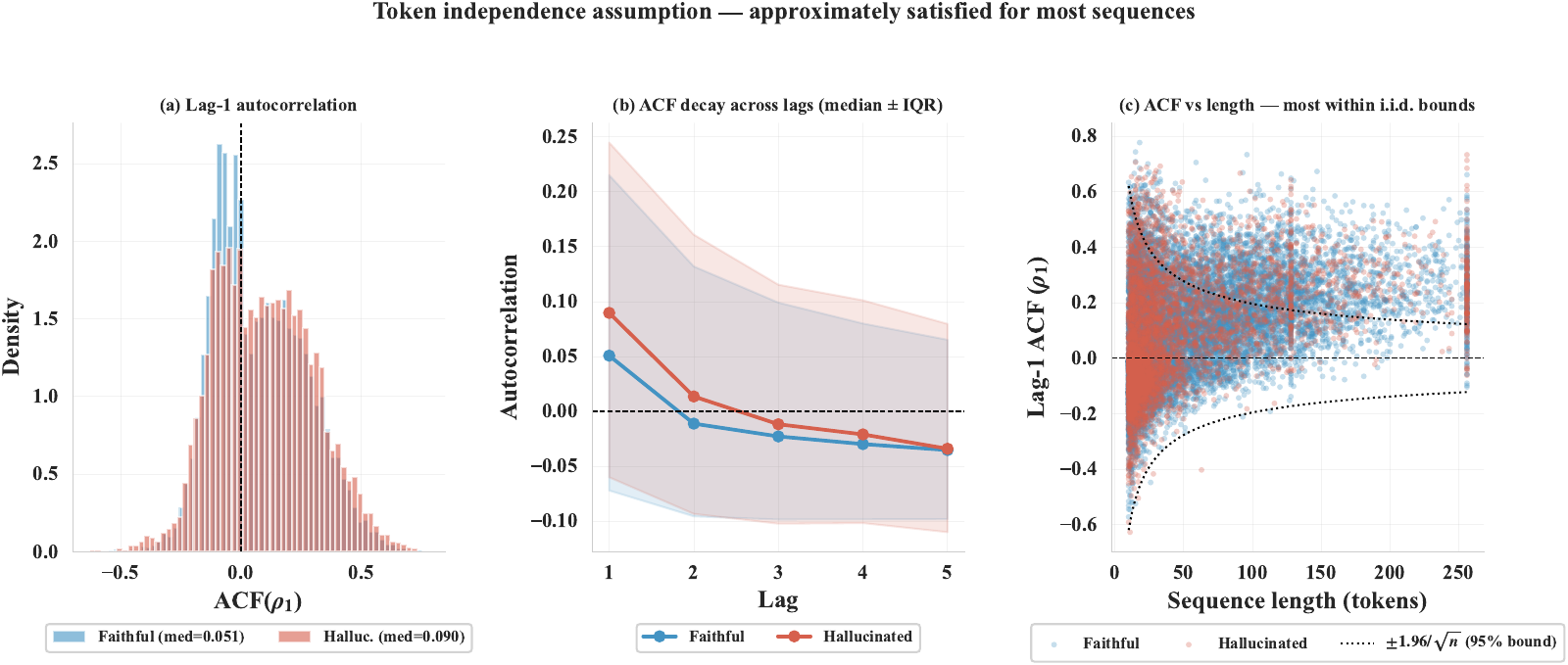}
    \caption{\textbf{Token entropy autocorrelation analysis} across 18{,}705 sequences.
    (a)~Distribution of lag-1 ACF; median $\rho_1 = 0.061$, well below the threshold for concern.
    (b)~Lag-1 vs.\ lag-2 autocorrelation scatter; second-order dependence is weaker (median $\rho_2 \approx 0.03$).
    (c)~Effective sample size ratio $n_\text{eff}/n$; most sequences retain $>$80\% effective samples.
    Only 15.8\% of sequences have $|\rho_1| > 0.3$, yielding substantially reduced effective sample sizes.}
    \label{fig:autocorrelation}
\end{figure}

\subsection{Statistical Power of KS test as a Function of Sample Size}\label{app:exp_power}

\textbf{Setup.}
The KS test's power depends on sample size, so we characterize how many tokens are needed to reliably detect the distributional shift.
For each of the 80 experiments, we subsample $n \in \{50, 100, 200, 500, 1000, 2000, 5000, 10000\}$ tokens from each class and repeat the KS test 5 times per subsample size, recording the proportion of rejections at $\alpha{=}0.05$.

\textbf{Results.}
Figure~\ref{fig:ks_power_matrix} displays the rejection rate as a function of sample size across all $10 \times 8 = 80$ cells.
Overall, 2{,}090/3{,}200 resamples (65\%) achieve significance.
As sample size increases from 50 to 10{,}000, rejection rates increase monotonically for the majority of experiments.
The power does not saturate for all experiments within 10{,}000 tokens, reflecting the genuinely small effect sizes (median KS-$D = 0.100$) and motivating the use of the full token sequence per generation rather than fixed-length subsamples.

Notably, the API models (GPT-4.1 family) show comparable power characteristics to open-weight models despite having no access to logits, the entropy signal is equally detectable when computed from token-level log-probabilities.

\begin{figure}[h]
    \centering
    \includegraphics[width=\textwidth]{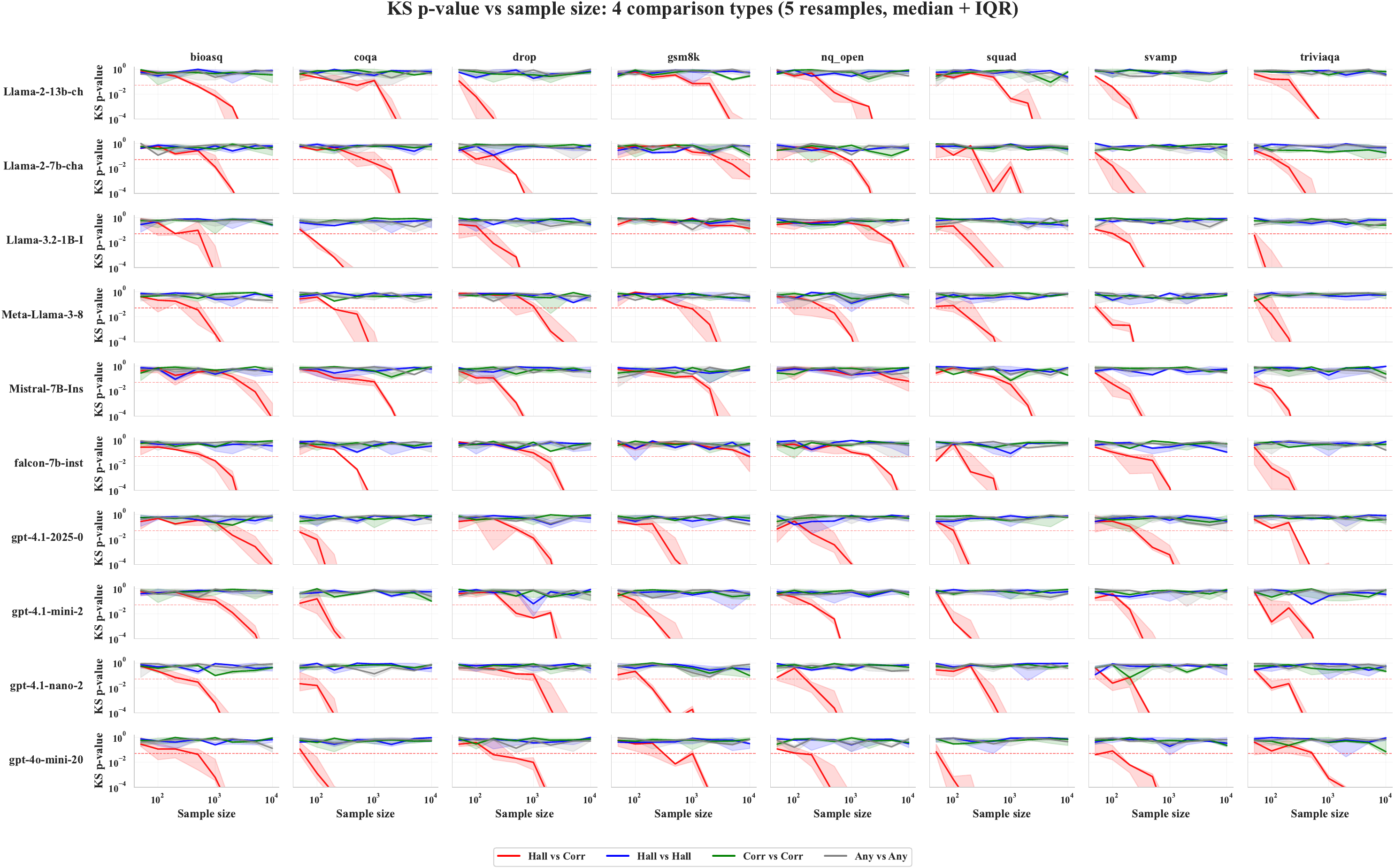}
    \caption{\textbf{KS test power vs.\ sample size} across all 80 experiments (10 models $\times$ 8 datasets).
    Each cell shows the $\log_{10}$ $p$-value as a function of subsample size for one model--dataset pair.
    Significance threshold ($\alpha{=}0.05$) is shown as a horizontal line.
    Power increases monotonically with sample size; 65\% of all 3{,}200 resamples achieve significance.}
    \label{fig:ks_power_matrix}
\end{figure}

\subsection{Per-Dataset ECDF Gallery}\label{app:ecdf_gallery}

Figure~\ref{fig:ecdf_gallery} shows the empirical CDFs of token entropy for faithful and hallucinated generations for all 8 datasets (pooled across all 10 models).
The visual offset between curves corresponds to the $d_{KS}$ distance; larger offsets indicate easier detection.
SVAMP and TriviaQA show the clearest visual separation (KS-$D > 0.13$), while GSM8K shows the smallest visible difference (KS-$D = 0.05$), consistent with its difficulty for entropy-based detection.

\begin{figure}[H]
    \centering
    \includegraphics[width=\textwidth]{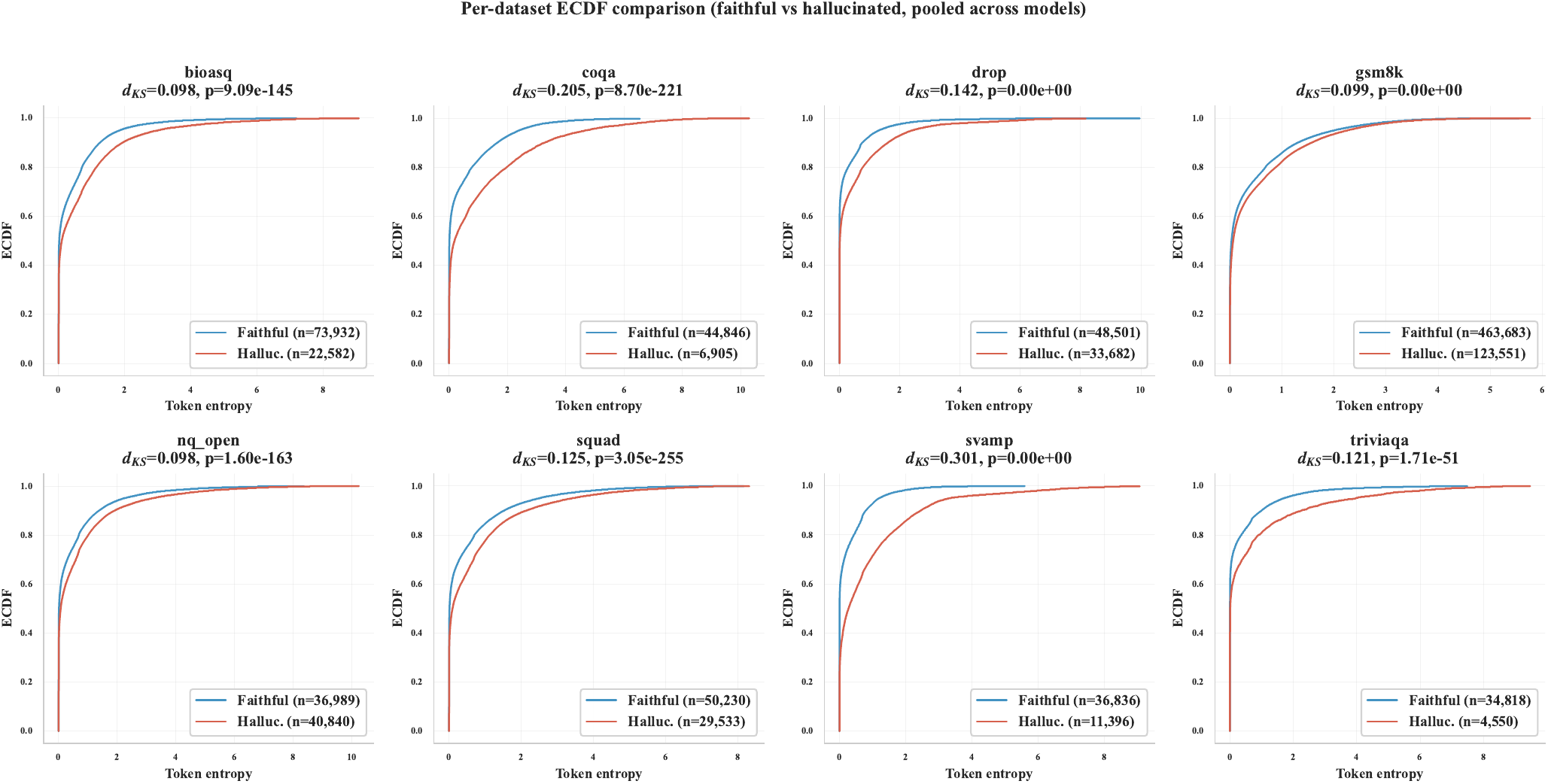}
    \caption{\textbf{Per-dataset ECDF gallery.}
    Faithful (blue) vs.\ hallucinated (red) entropy distributions for each of 8 datasets, pooled across all 10 models.
    $d_{KS}$ and $p$-values are annotated per panel.
    Visual separation directly corresponds to detection performance: datasets with clearly offset curves (SVAMP, TriviaQA) yield the highest AUROC.}
    \label{fig:ecdf_gallery}
\end{figure}

\subsection{Generation Length and Test Power}\label{app:exp_length}

\textbf{Setup.}
The KS test's power depends on sample size, so shorter generations may lack sufficient tokens for reliable discrimination.
We stratify all samples into length bins and compute per-stratum AUROC for CES and baselines within each experiment.
The correlation between the $d_{KS}$ effect size and CES AUROC is also characterised.

\textbf{Results.}
Figure~\ref{fig:length_stratification} shows that CES maintains a consistent advantage across all length strata, with AUROC remaining above 0.5 even for the shortest generations.
As expected, performance improves with generation length for all distributional methods.
The strong positive correlation between $d_{KS}$ and CES AUROC (Spearman $\rho = 0.658$, $p = 3.3 \times 10^{-11}$) confirms that detection performance is fundamentally limited by the strength of the distributional shift.
The rank correlation between dataset difficulty ordered by AUROC and by $d_{KS}$ is $\rho = 0.881$ ($p = 3.9 \times 10^{-3}$), indicating that the two orderings are highly concordant: datasets where the distributional signal is strongest are precisely those where detection is easiest.

\begin{figure}[h]
    \centering
    \includegraphics[width=0.75\textwidth]{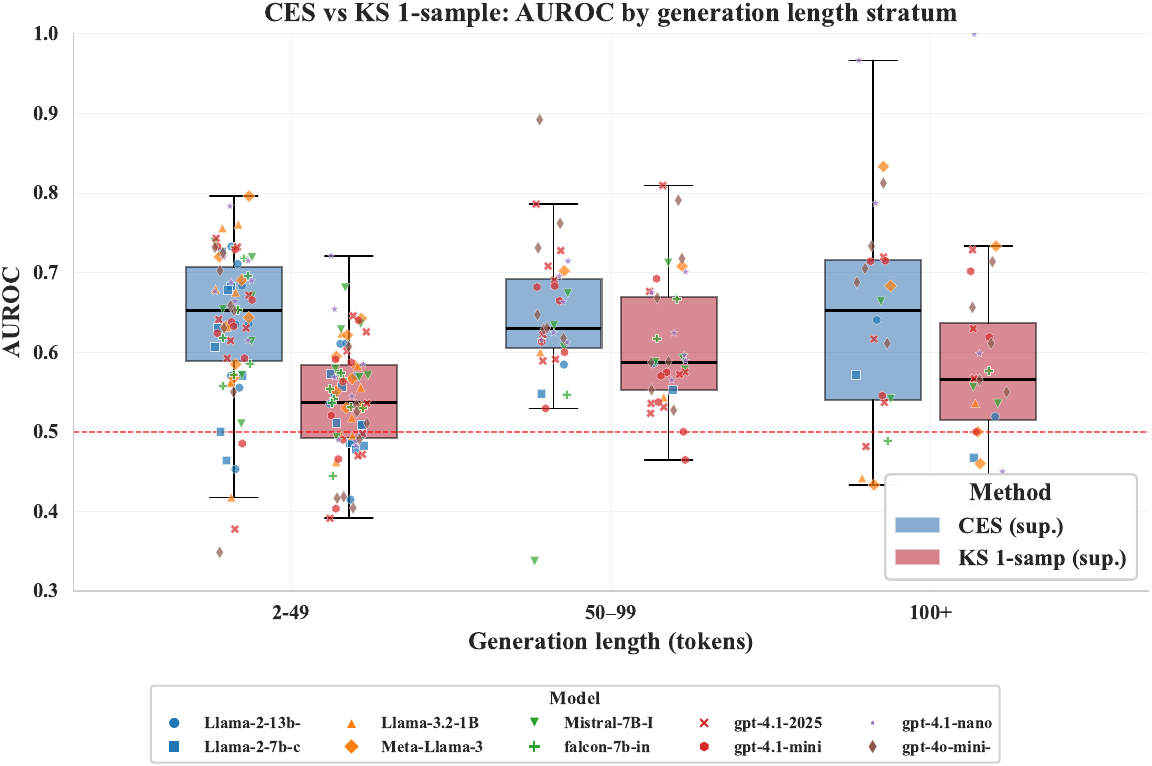}
    \caption{\textbf{AUROC by generation length stratum.}
    CES maintains its advantage at all lengths, with the distributional signal strengthening for longer generations.
    139 stratified rows across 3 length bins are shown.}
    \label{fig:length_stratification}
\end{figure}

\subsection{Combinatorial Statistic Ablation}\label{app:exp_combinatorial}
The choice of geometric mean of mean and maximum generation entropy appears to be arbitrary. 
While we designed CES to capture observed patterns in the data, we also run ablations to show that this is indeed the best choice out of all combinations.  
We run this across 44 variants of our main method (CES) finding that we perform the best (in terms of rank). 

\textbf{Setup.}
CES computes $\sqrt{F_0(\bar{h}) \cdot F_0(h_{\max})}$: the geometric mean of the ECDF-transformed mean and maximum entropy.
To verify this choice is optimal among summary-statistic combinations, we exhaustively evaluate 44 variants combining different entropy summaries (mean, median, max, $q_{25}$, $q_{75}$).
We rank all variants by average rank across 80 experiments and apply the Friedman test followed by Nemenyi post-hoc comparisons.

\textbf{Results.}
The chosen CES formula $\text{geom}(\text{mean}, \text{max})$ achieves the best average rank (5.66 out of 15 top methods selected for statistical comparison) with median AUROC of 0.6493.
The Friedman test strongly rejects exchangeability ($\chi^2_F = 90.95$, $p = 2.5 \times 10^{-13}$; Iman-Davenport $F = 6.98$, $p = 6.1 \times 10^{-14}$).

Pairwise Wilcoxon signed-rank tests with Holm--Bonferroni correction confirm that CES is \emph{significantly} better than all 14 compared alternatives (14/14 significant at $\alpha{=}0.05$).

\begin{figure}[h]
    \centering
    \includegraphics[width=\textwidth]{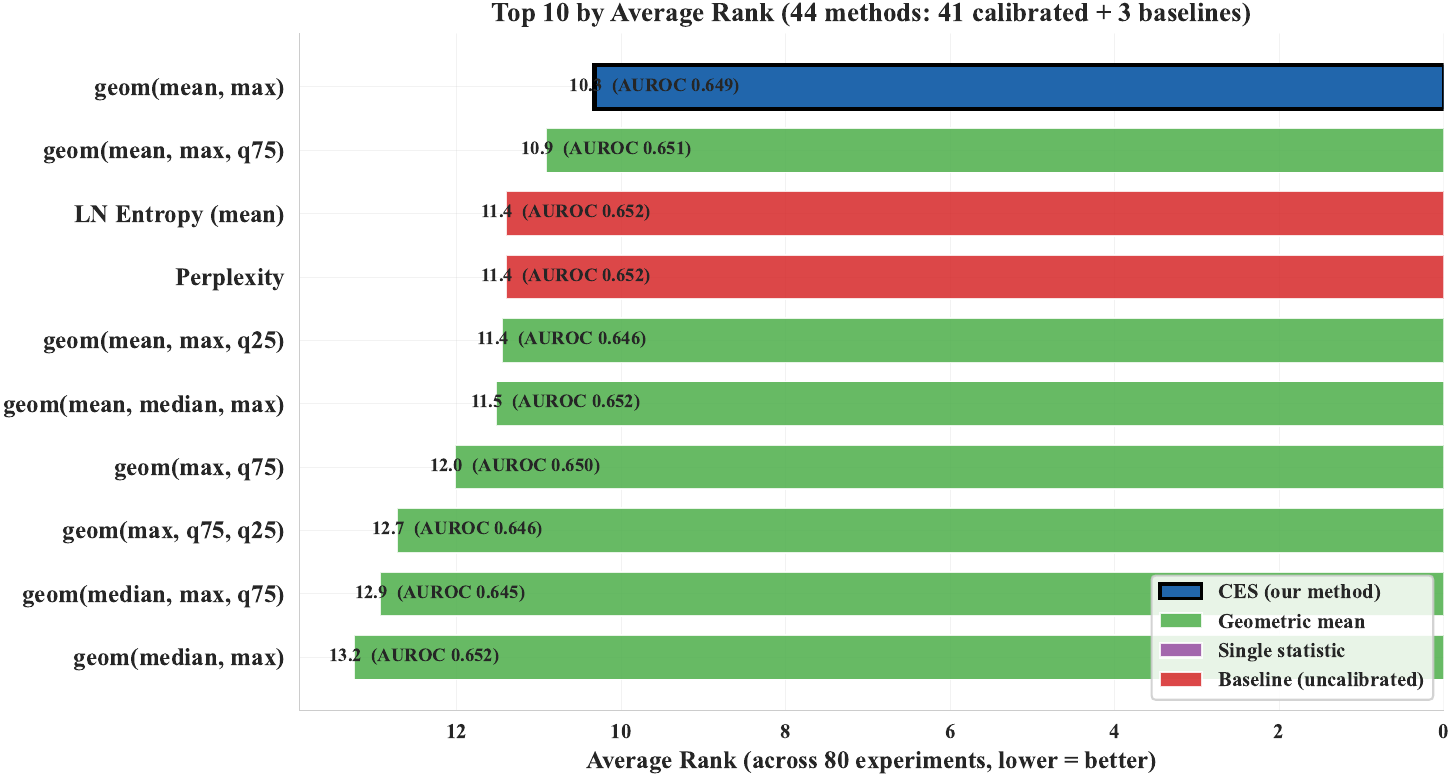}
    \caption{\textbf{Combinatorial statistic ranking.} (top 10 of 44 variants).
    Ranked by average rank across 80 experiments (lower = better).
    The CES formula $\text{geom}(\text{mean}, \text{max})$ achieves the best average rank (10.3), confirming that combining location (mean) and tail (max) entropy statistics via geometric aggregation optimally captures both distributional features.}
    \label{fig:combinatorial_ranking}
\end{figure}

\subsection{Calibration Contamination}\label{app:exp_contamination}

\textbf{Setup.}
The supervised CES variant uses a reference ECDF built from known faithful generations.
In practice, the calibration set may be contaminated with hallucinated samples.
We simulate contamination by randomly replacing $c\%$ of the calibration pool with hallucinated tokens, for $c \in \{0, 10, 20, 30, 40, 50\}$, across all 80 experiments.

\textbf{Results.}
CES is remarkably robust to calibration contamination (Figure~\ref{fig:contamination}).
At 0\% contamination, median AUROC is 0.6531; at 50\% contamination, it is 0.6533 ($\Delta = +0.0002$).
This near-zero degradation arises because the entropy distributions of faithful and hallucinated sequences have substantial overlap (median KS-$D = 0.100$): contaminating the reference with hallucinated tokens slightly broadens the ECDF but does not qualitatively alter its shape.

In contrast, the KS 1-sample test is more sensitive to contamination: supervised KS drops from 0.431 to 0.421 ($\Delta = -0.010$) at 50\% contamination, while unsupervised KS remains constant (by construction).

This result also explains why the \emph{unsupervised} CES variant (which uses a pooled reference from both classes) performs comparably to the supervised variant: the unsupervised ECDF is effectively a ``maximally contaminated'' reference, yet discrimination is preserved because the CES score captures the \emph{relative} position within the distribution rather than absolute values.

\begin{figure}[H]
    \centering
    \includegraphics[width=1\textwidth]{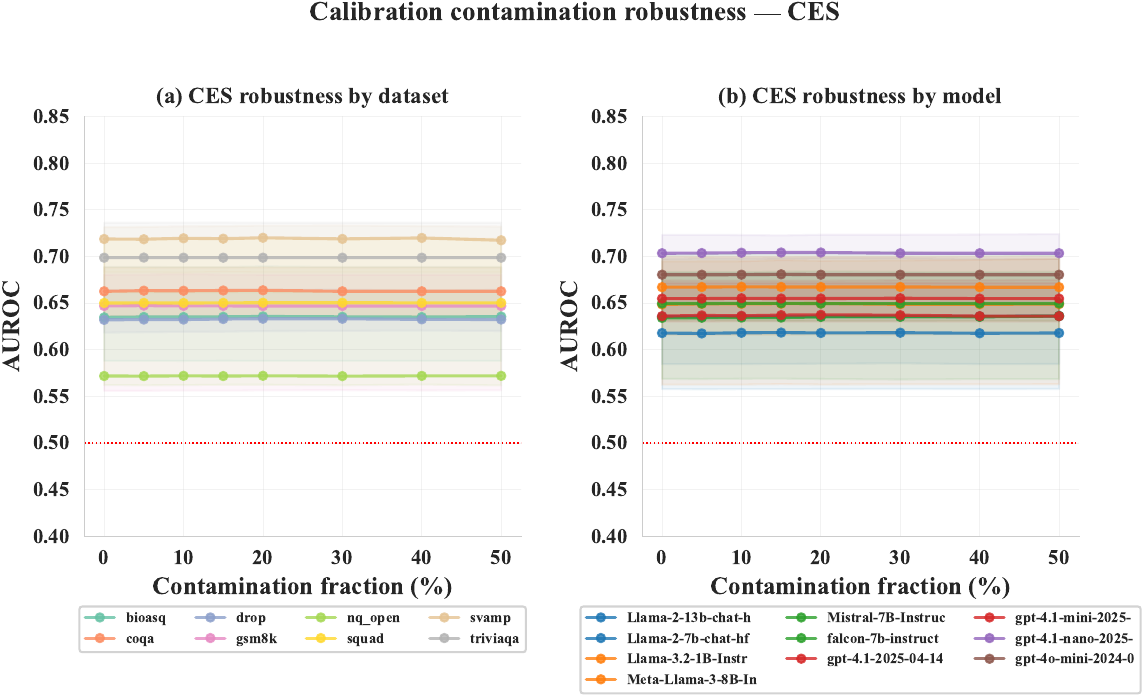}
    \caption{\textbf{CES is robust to calibration contamination.}
    AUROC as a function of the fraction of hallucinated tokens injected into the reference ECDF, across all 80 experiments.
    Even at 50\% contamination, CES performance is unchanged ($\Delta = +0.0002$), while the KS 1-sample test degrades moderately.}
    \label{fig:contamination}
\end{figure}

\subsection{Robustness to Noisy Judgments}\label{app:exp_noisy_judge}

\textbf{Setup.}
The preceding experiments assume access to a perfect binary oracle for hallucination labels.
In practice, labels may be noisy due to imperfect judges (e.g., GPT-based evaluation) or inherently ambiguous cases.
We simulate label noise by randomly flipping $p \in \{0.001, 0.01, 0.02, 0.05, 0.10, 0.20, 0.50\}$ fraction of the calibration labels (corrupting the reference ECDF) and evaluating on unchanged test labels.
Each noise level is repeated 30 times across all 80 experiments, and we report median AUROC across repetitions.

\textbf{Results.}
CES exhibits remarkable robustness to label noise (Table~\ref{tab:noisy_judge} and Figure~\ref{fig:noisy_judge}).
Even at $p = 0.50$ (50\% of calibration labels flipped), the median AUROC remains virtually unchanged from the clean baseline.
The unsupervised variant (which does not use labels at all) provides a natural performance floor that is already within $0.001$ of the supervised variant.
The gap between supervised and unsupervised CES converges to zero monotonically with noise level.

This result has direct practical implications: when using LLM-as-judge for hallucination labelling (which typically achieves 80--90\% accuracy), CES performance will not degrade meaningfully.
The method can be deployed with noisy supervision without any performance penalty.

\begin{table}[H]
\centering
\small
\caption{\textbf{CES is robust to label noise.} Median AUROC across 80 experiments at varying noise fractions (30 repetitions each). Both supervised and unsupervised variants are stable; the gap converges to zero under noise.}
\label{tab:noisy_judge}
\begin{tabular}{ccccc}
\toprule
$p_{\text{noise}}$ & AUROC (sup.) & AUROC (unsup.) & $\Delta$ (sup$-$unsup) & Hall.\ in ref. \\
\midrule
0.000 & 0.6531 & 0.6531 & $-$0.0004 & 0 \\
0.001 & 0.6531 & 0.6531 & $-$0.0004 & 0 \\
0.010 & 0.6531 & 0.6531 & $-$0.0004 & 1 \\
0.020 & 0.6531 & 0.6531 & $-$0.0004 & 2 \\
0.050 & 0.6531 & 0.6531 & $-$0.0004 & 6 \\
0.100 & 0.6532 & 0.6531 & $-$0.0004 & 13 \\
0.200 & 0.6531 & 0.6531 & $-$0.0003 & 26 \\
0.500 & 0.6530 & 0.6531 & $+$0.0000 & 64 \\
\bottomrule
\end{tabular}
\end{table}

\begin{figure}[H]
    \centering
    \includegraphics[width=\textwidth]{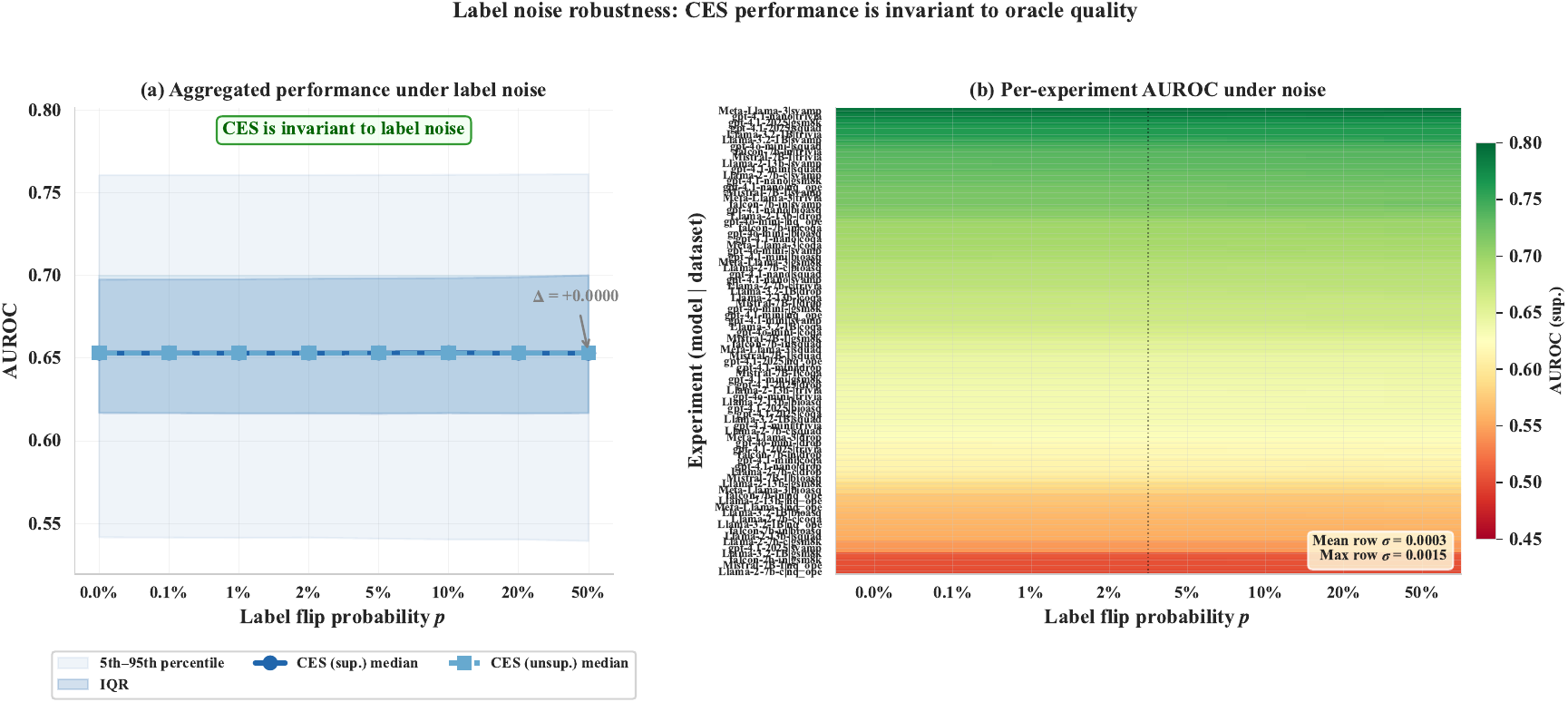}
    \caption{\textbf{Noisy judge experiment} across all 80 model--dataset cells.
    Each panel shows one experiment's AUROC as a function of label noise $p$.
    Performance is essentially flat across the entire noise range, confirming CES's insensitivity to label quality.}
    \label{fig:noisy_judge}
\end{figure}

\subsection{Critical Difference Analysis}\label{app:critical_difference_analysis}

\textbf{Setup.}
We compare CES against 16 external benchmark methods from the literature (Semantic Entropy, Discrete SE, KLE variants, Eigenscore, Effective Rank, SelfCheckGPT, Embedding Regression, FAVA, P(True), and LN Entropy) across all 80 experiments.
Formal statistical testing (Friedman + Nemenyi) is performed on the full 80-experiment grid with all 17 methods.

\textbf{Results (80 experiments, average rank).}
Across all 80 experiments (6 open-weight + 4 API models $\times$ 8 datasets), CES (unsupervised) achieves average rank 6.29, behind KLE (heat) (6.16) and KLE (6.21).
Embedding Regression achieves the most first-place finishes (32/80), followed by CES and P(True) (9/80 each) and KLE (8/80).
Overall, CES wins 854/1279 pairwise comparisons (66.8\%) and beats 12/16 methods at $>$50\% win rate.

\textbf{Formal test (80 experiments, 17 methods).}
The Friedman test strongly rejects method exchangeability ($\chi^2_F = 272.80$, $p \approx 0$; Iman-Davenport $F = 21.47$, $p \approx 0$).
The Nemenyi CD is 2.78 ($k = 17$).
CES (unsupervised) achieves average rank 6.29 and belongs to the top clique (Table~\ref{tab:benchmark_clique}).

\begin{table}[t]
\centering
\small
\caption{\textbf{Top statistical clique} (Nemenyi CD = 2.78, $\alpha = 0.05$). Methods within CD of the best rank (6.16) are statistically indistinguishable. Multi-sample methods ($\dagger$) require multiple forward passes per input.}
\label{tab:benchmark_clique}
\begin{tabular}{lcc}
\toprule
Method & Avg.\ Rank & Multi-sample? \\
\midrule
KLE (heat)$^\dagger$         & 6.16 & Yes \\
KLE$^\dagger$                & 6.16 & Yes \\
\textbf{CES (unsup.)}       & \textbf{6.29} & \textbf{No} \\
KLE (full)$^\dagger$         & 6.65 & Yes \\
Embed.\ Regr.$^\dagger$     & 6.92 & Yes \\
CES (sup.)                   & 7.03 & No \\
LN Entropy                   & 8.20 & No \\
Perplexity                   & 8.20 & No \\
Discrete SE$^\dagger$        & 8.61 & Yes \\
LNE (bench)                  & 8.90 & No \\
\bottomrule
\end{tabular}
\end{table}

Pairwise Wilcoxon signed-rank tests with Holm--Bonferroni correction confirm that CES is significantly better than 7/16 methods:
Gen.\ Length ($p = 5.4 \times 10^{-11}$, $\Delta = +0.095$),
CES sup.\ ($p = 3.0 \times 10^{-7}$),
LN Entropy/Perplexity ($p = 7.7 \times 10^{-7}$, $\Delta = +0.007$),
LNE bench ($p = 8.3 \times 10^{-7}$, $\Delta = +0.009$),
SelfCheckGPT ($p = 2.1 \times 10^{-5}$, $\Delta = +0.023$), and
FAVA ($p = 2.6 \times 10^{-5}$, $\Delta = +0.046$).

CES is \emph{not} significantly different from KLE ($p = 0.045$ uncorrected, not significant after Holm), KLE variants ($p = 0.045$--$0.050$), Semantic Entropy ($p = 0.018$), Discrete SE ($p = 0.59$), Eigenscore ($p = 0.078$), Effective Rank ($p = 0.059$), or Embedding Regression ($p = 0.25$).
This places CES firmly in the competitive tier while requiring only a single forward pass.

\begin{figure}[h]
    \centering
    \includegraphics[width=\textwidth]{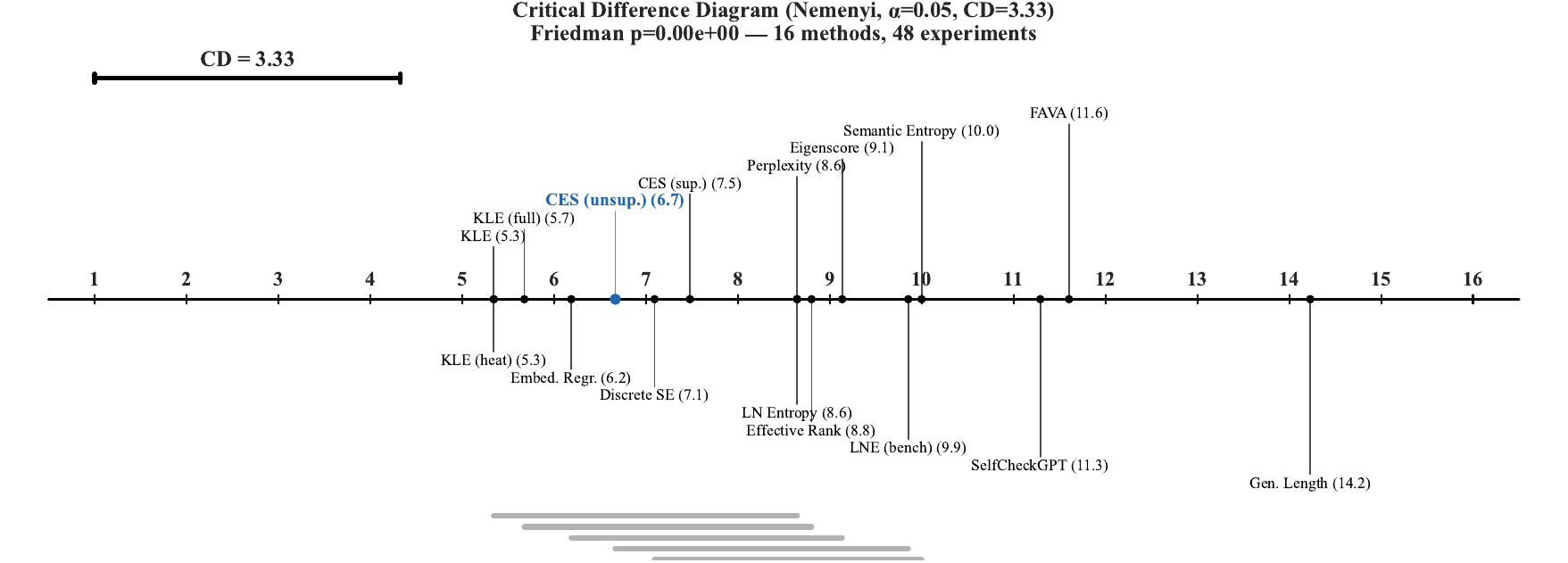}
    \caption{\textbf{Nemenyi Critical Difference diagram} for 17 methods on 80 experiments.
    Methods connected by a thick bar are not significantly different ($\alpha = 0.05$, CD = 2.78).
    CES (unsupervised) at rank 6.29 is within the top clique and statistically indistinguishable from the best-ranked methods (KLE (heat) at 6.16).}
    \label{fig:nemenyi_cd}
\end{figure}

\begin{figure}[H]
    \centering
    \includegraphics[width=\textwidth]{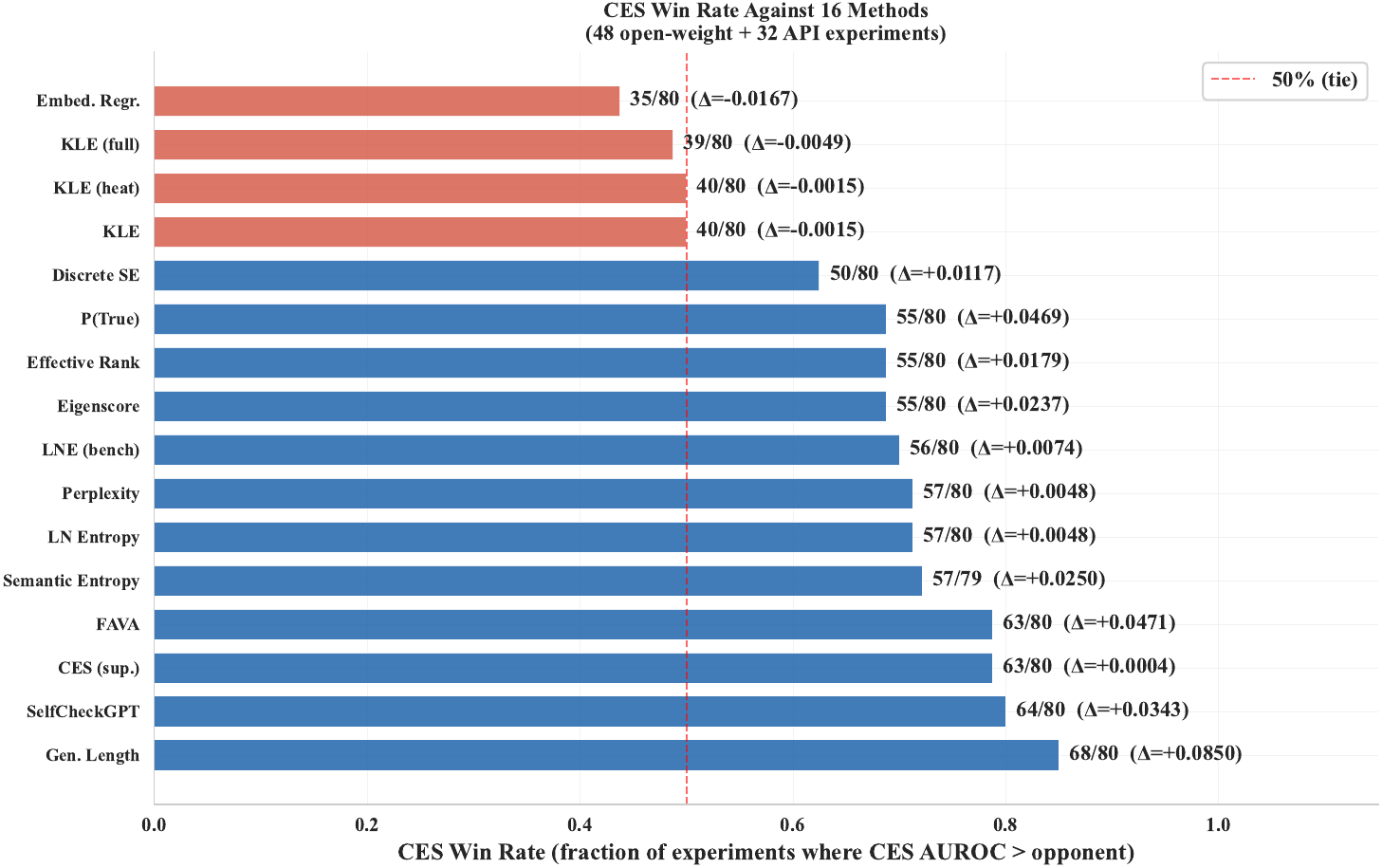}
    \caption{\textbf{CES pairwise win rates} against each of 16 benchmark methods.
    CES wins 854/1279 pairwise comparisons (66.8\%) overall and beats 12/16 methods at $>$50\% win rate.
    Median AUROC advantage ($\Delta$) annotated for each opponent.}
    \label{fig:ces_win_rates}
\end{figure}

 
\subsection{API vs.\ Open-Weight Model Comparison}\label{app:exp_api}

\textbf{Setup.}
A key practical question is whether entropy-based detection generalises to API-only models where only token-level log-probabilities (not full logit vectors) are available.
We evaluate four GPT-4.1 family models (GPT-4.1, GPT-4.1-mini, GPT-4.1-nano, GPT-4o-mini) on the same 8 datasets with 500 samples each, computing CES from log-probability--derived entropy.

\textbf{Results.}
CES generalises well to API models:
the overall AUROC median for API models is 0.669 (IQR: [0.635, 0.703]) compared against open-weight median of 0.642 (IQR: [0.572, 0.694]).
The difference is marginally significant (Mann-Whitney $U = 576$, $p = 0.060$; KS $D = 0.323$, $p = 0.031$).
We note that API models hallucinate less frequently (mean 21.0\% vs.\ 35.3\%), which may partly explain the higher AUROC (fewer hallucinations are ``confident'').
Moreover, 4/8 datasets show significant differences between API and open-weight CES AUROC (Table~\ref{tab:api_per_dataset}).
API models outperform on knowledge-intensive tasks (BioASQ $+0.10$, NQ-Open $+0.12$, SQuAD $+0.10$) but underperform on short-answer arithmetic (SVAMP $-0.06$, TriviaQA $-0.09$).

\begin{table}[t]
\centering
\small
\caption{\textbf{CES AUROC by dataset}: open-weight vs.\ API models. Significant differences ($p < 0.05$, Mann-Whitney) marked with $*$.}
\label{tab:api_per_dataset}
\begin{tabular}{lcccl}
\toprule
Dataset & Open (median) & API (median) & $\Delta$ & $p$-value \\
\midrule
BioASQ   & 0.592 & 0.694 & $+$0.101 & 0.038$^*$ \\
CoQA     & 0.669 & 0.648 & $-$0.021 & 0.610 \\
Drop     & 0.649 & 0.629 & $-$0.021 & 0.610 \\
GSM8K    & 0.567 & 0.698 & $+$0.131 & 0.067 \\
NQ-Open  & 0.564 & 0.684 & $+$0.121 & 0.010$^*$ \\
SQuAD    & 0.638 & 0.738 & $+$0.100 & 0.010$^*$ \\
SVAMP    & 0.730 & 0.674 & $-$0.056 & 0.010$^*$ \\
TriviaQA & 0.727 & 0.632 & $-$0.095 & 0.257 \\
\bottomrule
\end{tabular}
\end{table}

\textbf{Interpretation.}
The fact that CES works comparably on API models is significant for two reasons.
First, it demonstrates that the distributional signal exists even without access to the full vocabulary logit distribution; top-$k$ log-probabilities (as returned by the OpenAI API) suffice to construct a meaningful entropy trace.
Second, it establishes CES as one of very few hallucination detection methods that can operate on closed-source models without multiple API calls per sample (unlike Semantic Entropy or SelfCheckGPT, which require 5--10 regenerations).

The dataset-specific differences suggest an interaction between model capability and entropy signal strength: API models are ``better'' at knowledge recall (lower hallucination rate on factoid QA), which concentrates the hallucinations in genuinely harder cases. 
This makes the surviving hallucinations either easier (more uncertain) or harder (more confident) to detect depending on the dataset characteristics.


\begin{figure}[H]
    \centering
    \includegraphics[width=\textwidth]{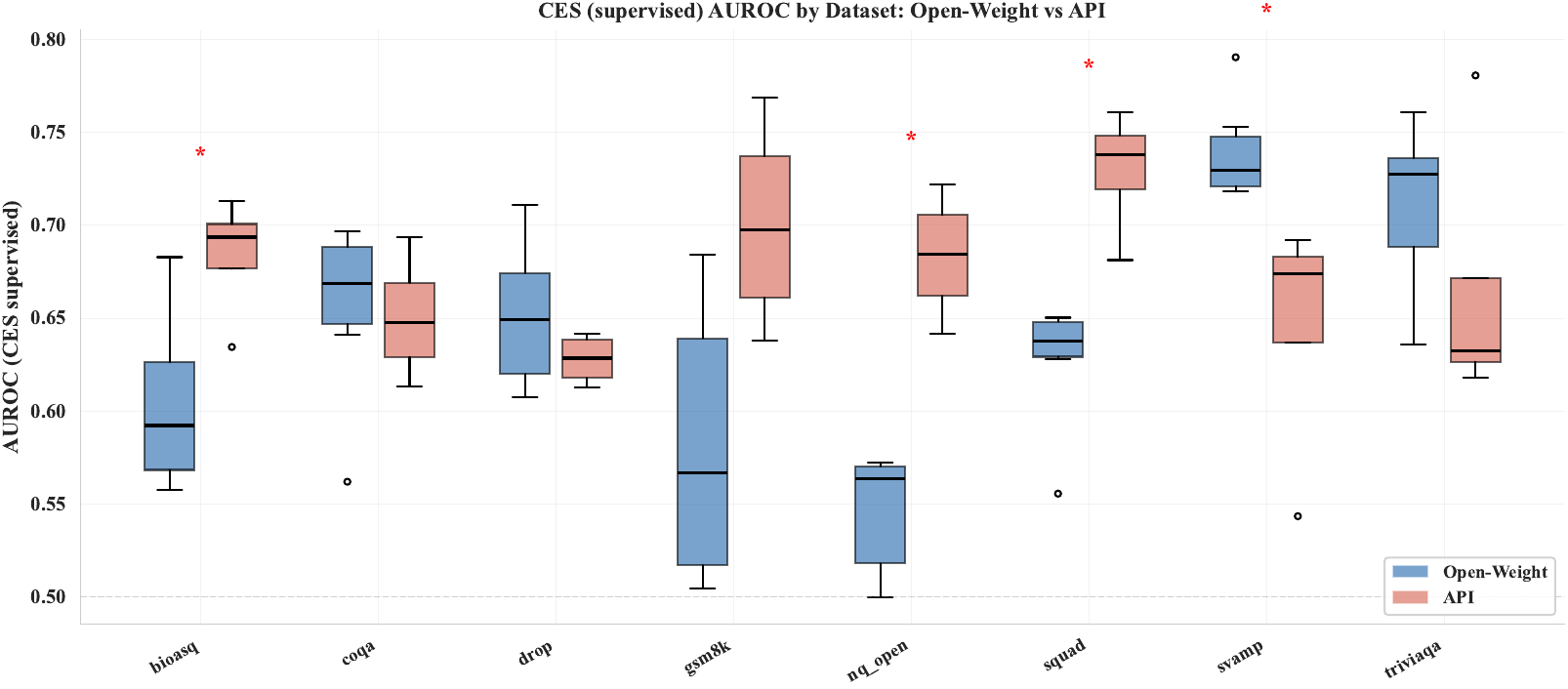}
    \caption{\textbf{Per-dataset CES AUROC comparison}: API (red) vs.\ open-weight (blue) models.
    Significant differences ($p < 0.05$) are marked with asterisks.
    4/8 datasets show statistically significant differences, with the direction varying by task type.}
    \label{fig:api_per_dataset}
\end{figure}

\begin{figure}[H]
    \centering
    \includegraphics[width=\textwidth]{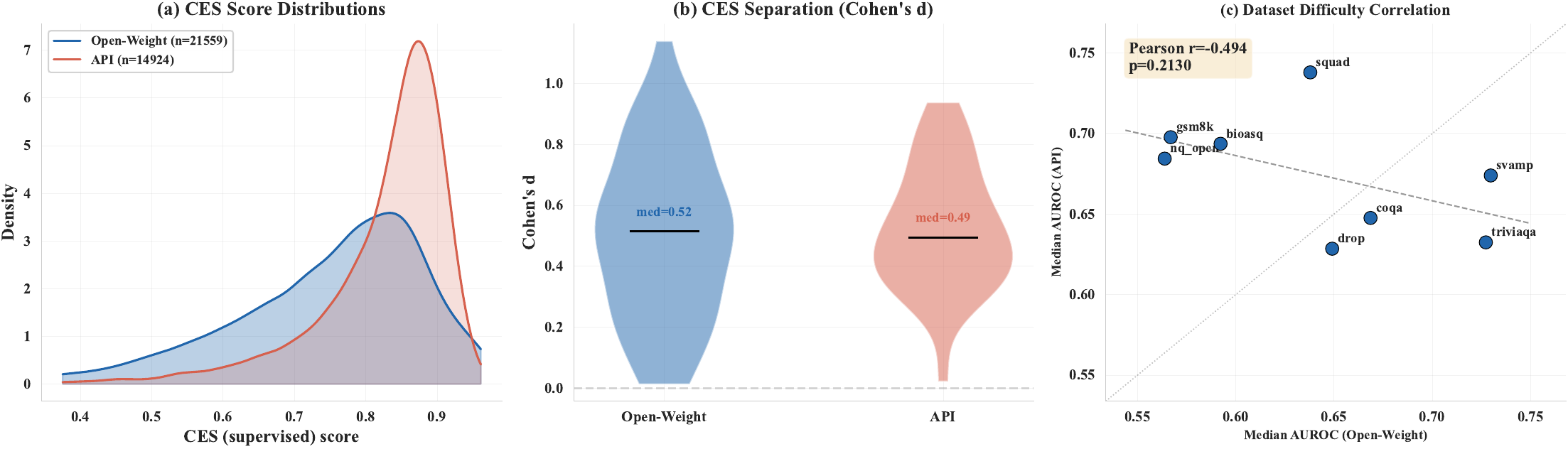}
    \caption{\textbf{Distributional properties}: API vs.\ open-weight models.
    (a)~AUROC distributions overlap substantially.
    (b)~Cohen's $d$ distributions are comparable (API median 0.494, open 0.515).
    (c)~Dataset difficulty correlation between model types (Pearson $r = -0.49$), suggesting that difficult datasets for API models are not necessarily difficult for open-weight models and vice versa.}
    \label{fig:api_distributional}
\end{figure}

\begin{figure}[H]
    \centering
    \includegraphics[width=\textwidth]{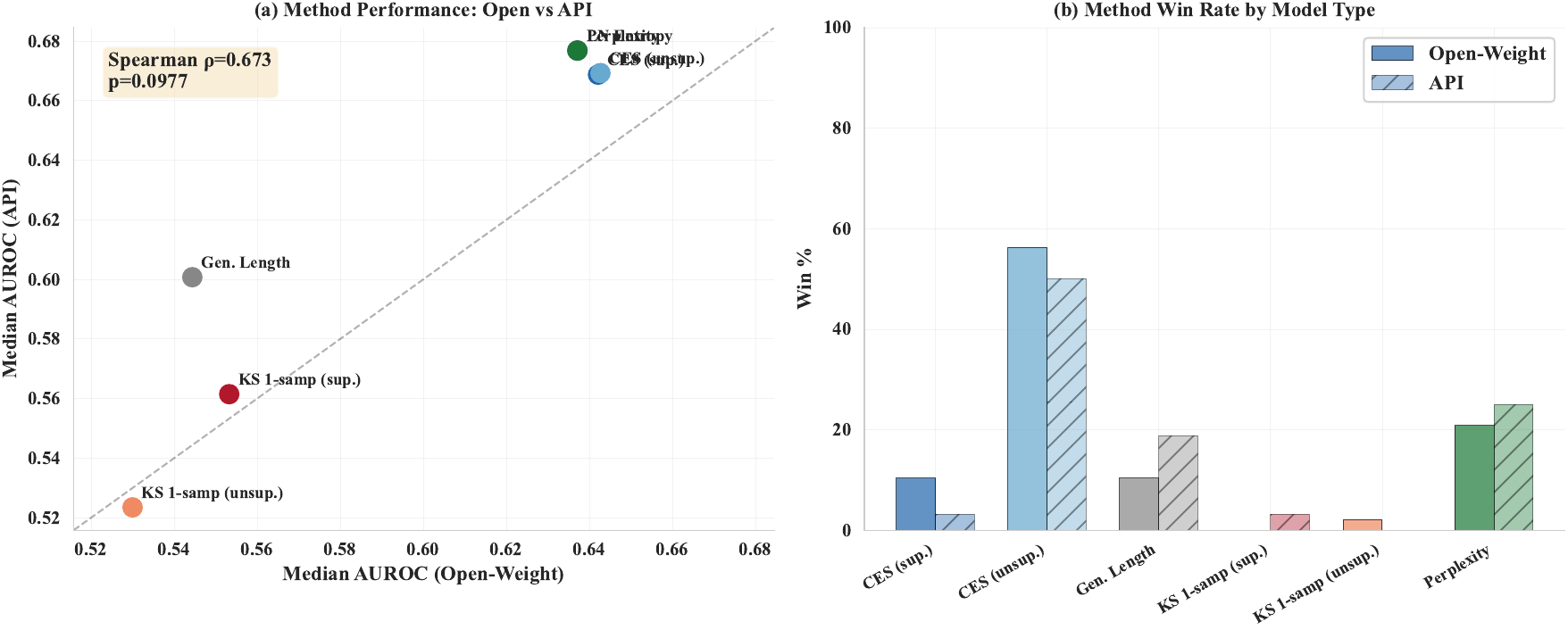}
    \caption{\textbf{Method ranking comparison}: internal methods ranked by median AUROC separately for API and open-weight models.
    Spearman rank correlation $\rho = 0.673$ ($p = 0.098$) indicates moderate agreement in method effectiveness across model types.}
    \label{fig:api_method_ranking}
\end{figure}

\subsection{Empirical Verification of CES Error Bounds}\label{app:exp_error_bounds}

\textbf{Setup.}
Theorem~\ref{thm:ces_power} establishes that the Type~I and Type~II errors of the CES test decay exponentially with generation length $m$.
Specifically, for a threshold $c = \sqrt{F_0(\mu) F_0(\zeta)}$ with tunable parameters $\mu \in (\mu_0, \mu_1)$ and $\zeta \in (\zeta_0, \zeta_1)$:
\begin{align}
    \text{Type I:} \quad & \mathbb{P}_{F_0}^m(\mathrm{CES} > c) \leq \exp\!\left(\frac{-2m(\mu - \mu_0)^2}{(\log d)^2}\right), \\
    \text{Type II:} \quad & \mathbb{P}_{F_1}^m(\mathrm{CES} < c) \leq F_1(\zeta)^m + \exp\!\left(\frac{-2m(\mu_1 - \mu)^2}{(\log d)^2}\right).
\end{align}
We verify these bounds empirically by constructing synthetic generations of controlled length under both hypotheses.

Since our QA datasets produce short generations (median 4--15 tokens), real data does not span a sufficient range of $m$ to observe the predicted exponential decay.
We overcome this by directly instantiating the theorem's i.i.d.\ assumption: we pool all token-level entropies across experiments into class-separated pools (faithful and hallucinated), then draw i.i.d.\ samples of length $m \in \{5, 10, 20, 50, 100, 200, 500, 1000, 2000, 5000\}$ from each pool.
For each synthetic generation, we compute CES using a reference ECDF $\widehat{F}_0$ constructed from 30\% of the faithful token pool.
A fixed detection threshold $c$ is determined via Youden's $J$ statistic at a reference length ($m = 50$) and applied uniformly across all lengths.
We generate 1{,}000 synthetic sequences per length per class, yielding well-powered estimates of the error rates at each $m$.

\textbf{Results.}
Figure~\ref{fig:error_bounds} confirms the theorem's predictions.
Both Type~I and Type~II error rates decay approximately exponentially with $m$ (linear on log scale), and the empirical rates lie below the theoretical upper bounds at all tested lengths, validating the bounds.

The empirical decay rate, estimated via linear regression on $\log(\text{error})$ versus $m$, exceeds the theoretical rate constant $2\Delta^2 / (\log d)^2$, indicating that the Hoeffding-based bounds are conservative (as expected, since Hoeffding's inequality does not exploit variance information).
By $m = 1{,}000$ tokens, both empirical error rates fall below $10^{-3}$, confirming that CES achieves near-perfect detection for sufficiently long generations.

\begin{figure}[H]
    \centering
    \includegraphics[width=0.7\textwidth]{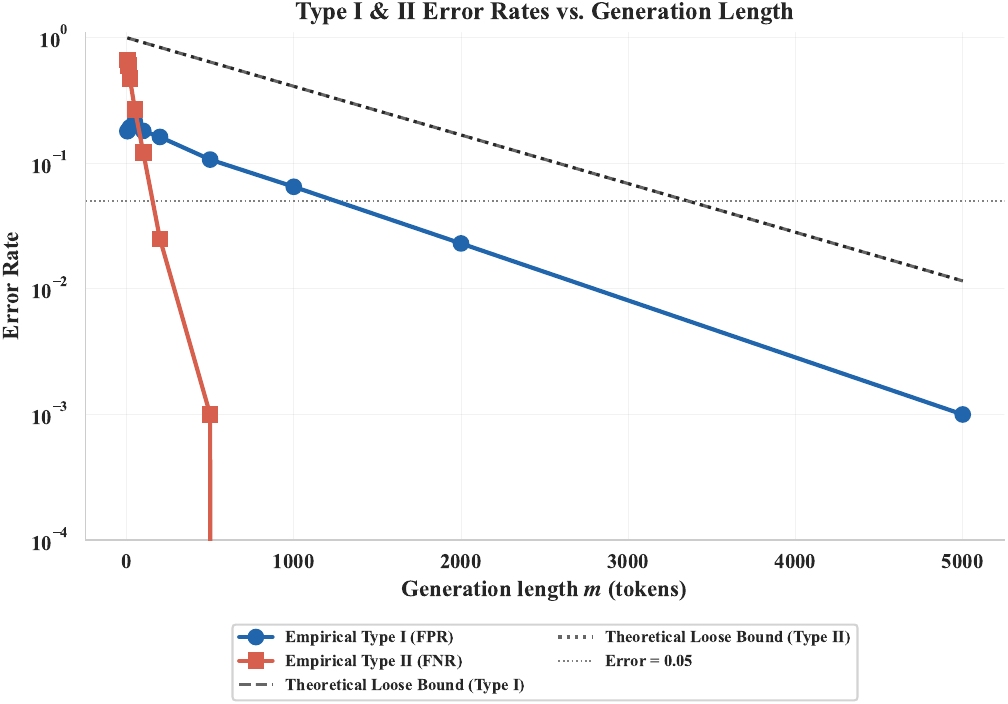}
    \caption{\textbf{Empirical verification of CES error bounds (Theorem~\ref{thm:ces_power}).}
    Type~I error (false positive rate) and Type~II error (false negative rate) as a function of generation length $m$, measured on synthetic i.i.d.\ generations drawn from pooled faithful and hallucinated token entropy distributions.
    Dashed lines show the theoretical exponential upper bounds.
    Both error rates decay exponentially with $m$, and the empirical curves lie below the theoretical bounds at all lengths, confirming the theorem.
    The gap between empirical and theoretical rates reflects the conservatism of the Hoeffding-based proof technique.}
    \label{fig:error_bounds}
\end{figure}

\textbf{Interpretation.}
This experiment validates the central theoretical claim: CES detects hallucinations with probability converging to one exponentially fast in the generation length.
The practical implication is that for generations exceeding ${\sim}100$ tokens, CES achieves very low error rates under the i.i.d.\ assumption.
For the short generations typical of factoid QA ($m \approx 10$), the error rates remain moderate (${\sim}30$--$40\%$), consistent with the AUROC values observed in the main experiments.

\subsection{Datasets and Example Outputs} \label{app:example_outputs}
In this section, we present some examples of the datasets, the generations, and the hallucination judge outputs for these samples. 
Moreover we provide descriptions of the datasets used. 

\begin{itemize}
    \item \textbf{BioASQ} \cite{krithara2023bioasq}: A biomedical question answering dataset designed to evaluate systems on questions from the biomedical domain. 

    \item \textbf{CoQA}: A conversational question answering dataset where questions are posed in a multi-turn dialogue context. 
    
    \item \textbf{DROP} \cite{dua2019drop}: A reading comprehension benchmark requiring discrete reasoning over paragraphs, such as addition, counting, or sorting. 

    \item \textbf{GSM8K} \cite{cobbe2021training}: A dataset of grade-school math word problems requiring multi-step arithmetic reasoning. 

    \item \textbf{NQ-Open} \cite{kwiatkowski2019natural}: The open-domain variant of Natural Questions, where models must answer factoid questions without a provided context passage. 

    \item \textbf{SQuAD} \cite{rajpurkar2018know}: The Stanford Question Answering Dataset, a reading comprehension benchmark where answers are spans extracted from Wikipedia passages.

    \item \textbf{SVAMP} \cite{patel2021nlp}: A challenge set of simple math word problems designed to test robustness to structural variations. 
    
    \item \textbf{TriviaQA} \cite{joshi2017triviaqa}: A large-scale reading comprehension dataset with question-answer pairs authored by trivia enthusiasts, paired with evidence documents.
\end{itemize}

In Figure \ref{fig:generation_examples} we show how example generations and outcomes of the pipelines. 

\begin{tikzpicture}[
    node distance=0.8cm,
    resultbox/.style={
        rectangle,
        draw=black!70,
        fill=blue!5,
        thick,
        rounded corners=3pt,
        text width=13.5cm,
        inner sep=8pt,
        font=\small,
        align=left
    },
    titlebox/.style={
        rectangle,
        draw=black!80,
        fill=gray!20,
        thick,
        rounded corners=4pt,
        text width=13.5cm,
        inner sep=10pt,
        font=\large\bfseries,
        align=center
    }
]

\node[titlebox] (title) {Generation Examples};

\node[resultbox, below=0.6cm of title] (bioasq) {%
    \textbf{Dataset:} bioasq \hfill \textbf{Model:} gpt-4.1-2025-04-14 \\[4pt]
    \textbf{Prompt:} Question: What is the number of long non coding RNAs in the human genome \\[2pt]
    \textbf{Ground Truth:} between 10,000 and 20,000 \\[2pt]
    \textbf{Generated:} There are estimated to be over 17,000 long non-coding RNAs (lncRNAs) in the human genome, according to current annotations such as GENCODE and Ensembl. The exact number may vary depending on the database and criteria used. \\[4pt]
    {\footnotesize\itshape\color{green!40!black}%
    Judge: Hallucinated = False | Confidence = 0.96 | Method: api:gpt-5.4-2026-03-05}
};

\node[resultbox, below=0.6cm of bioasq] (coqa) {%
    \textbf{Dataset:} coqa \hfill \textbf{Model:} gpt-4.1-nano-2025-04-14 \\[4pt]
    \textbf{Prompt:} Context: (CNN) -- Iran's vital oil industry appears to be the latest front in a power struggle between President Mahmoud Ahmadinejad and Supreme Leader Ali Khamenei\ldots \\[2pt]
    \textbf{Ground Truth:} RAND Corp \\[2pt]
    \textbf{Generated:} Nader works at the RAND Corporation. \\[4pt]
    {\footnotesize\itshape\color{green!40!black}%
    Judge: Hallucinated = False | Confidence = 0.98 | Method: api:gpt-5.4-2026-03-05}
};

\node[resultbox, below=0.6cm of coqa] (triviaqa) {%
    \textbf{Dataset:} triviaqa \hfill \textbf{Model:} gpt-4o-mini-2024-07-18 \\[4pt]
    \textbf{Prompt:} Context: Art auctions Great Britian Peridocials\ldots Christie, Manson \& Woods International Inc.\ldots \\[2pt]
    \textbf{Ground Truth:} auctions \\[2pt]
    \textbf{Generated:} Christie, Manson and Woods is involved in art auctions. \\[4pt]
    {\footnotesize\itshape\color{green!40!black}%
    Judge: Hallucinated = False | Confidence = 0.94 | Method: api:gpt-5.4-2026-03-05}
};
\label{fig:generation_examples}
\end{tikzpicture}

\section{Extended Related Works}
\label{app:extended_related_work}
We organise related work into three areas of increasing
proximity to our contribution: the broader landscape of
hallucination detection, uncertainty quantification methods
that form our primary baselines, and statistical testing
frameworks that share our methodological perspective.

\subsection{Hallucination Detection in LLMs}
\label{sec:rw_hallucination}

Hallucination in language models has been studied from
multiple angles.
\cite{kalai2025language} provide a computational-complexity
perspective, showing that hallucination \emph{classification}
is fundamentally easier than generation: a result that
motivates post-hoc detection approaches such as ours.
\cite{maynez2020faithfulness} offer an early taxonomy of
hallucination types in abstractive summarisation,
distinguishing intrinsic (contradicting the source) from
extrinsic (unverifiable) fabrications.
\cite{azaria2023internal} demonstrate that model-internal
representations carry predictive signal for hallucination,
motivating the hidden-state methods we compare against.

A prominent line of work grounds detection in external
evidence.
SAFE \cite{wei2024long} decomposes generated text into atomic
claims and verifies each against Google Search results,
enabling fine-grained, claim-level factuality scoring.
\cite{shah2025validation} retrieves web evidence and applies
NLI models to assess support, while
\cite{obeso2025real} targets real-time entity-level
hallucination identification via web search and strong-model
annotation.
These retrieval-based pipelines achieve high precision when
evidence is available, but incur latency from external calls
and degrade when retrieval coverage is poor. 
These constraints
make them complementary to, rather than competitive with,
inference-time methods like CES.

A parallel body of work avoids external retrieval entirely.
Attention-based detectors exploit the structure of
self-attention as a proxy for grounding.
Lookback Lens \cite{chuang2024lookback} computes the ratio of
attention allocated to the input context versus newly generated
tokens and feeds this feature into a logistic classifier,
operating as a lightweight chunk-level detector.
\cite{ogasa2025hallucination} construct richer
attention-derived features (average attention received,
diversity of attention, diversity of attended tokens) and train
a Transformer encoder with a CRF head for token-level span
detection.
\cite{niu2025robust} learn an adaptive policy for identifying
the most hallucinated tokens, while
\cite{moslonka2026learned} address hallucinations arising from
insufficient context in retrieval-augmented generation.
These methods require either access to attention matrices or
supervised training on hallucination-labeled data; CES
requires neither.

\subsection{Uncertainty Quantification for Detection}
\label{sec:rw_uq}

Uncertainty-based hallucination detection is the family most
closely related to our work.
We distinguish methods by their inference-time requirements.

\textbf{Single-pass methods.}
Perplexity \cite{jelinek1977perplexity} and
length-normalised entropy (LNE) \cite{malinin2020uncertainty}
are the most widely used single-pass signals.
Both reduce the entropy sequence to a scalar mean, discarding
distributional information.
\cite{janiak2025illusion} show that generation length alone
is a surprisingly competitive predictor, and argue that
progress in single-pass detection has been limited.
\cite{varshney2023stitch} and \cite{huang2025look}
further explore token-level entropy and logit-based features
but rely on heuristic thresholds without formal guarantees.
CES extends this line by retaining distributional information
(via the max-entropy tail statistic) and providing a
calibration layer that yields interpretable, cross-task
scores.

\textbf{Multi-sample methods.}
Semantic Entropy \cite{farquhar2024detecting} clusters $K$
generations into meaning classes via NLI and computes entropy
over the class distribution.
Kernel Language Entropy (KLE) \cite{nikitin2024kernel}
replaces hard NLI clustering with soft graph-Laplacian
kernels over entailment graphs.
EigenScore \cite{chen2023going} measures the volume spanned
by $K$ generation embeddings.
SAR \cite{duan2024shifting} weights token-level uncertainty
by relevance and aggregates across samples using semantic
similarity.
SelfCheckGPT \cite{manakul2023selfcheckgpt} measures
consistency across $K$ generations via $n$-gram overlap or
NLI.
$\mathbb{P}(\text{True})$ \cite{kadavath2022language} appends
a verification question and reads off the affirmative-token
probability.
These methods achieve strong empirical performance but require
$K \geq 5$ forward passes (typically at elevated temperature),
making them $K\times$ more expensive than CES at inference
time.

\textbf{Model-internal methods.}
INSIDE \cite{chen2024inside} projects hidden states onto
probing directions trained to separate faithful from
hallucinated representations.
Covariance-based diagnostics (log-determinant, effective rank)
measure the dispersion of hidden-state features.
These require access to intermediate activations, which is
unavailable through standard inference APIs.

Comprehensive surveys are provided by
\cite{xia2025survey} and \cite{janiak2025illusion}.

\subsection{Statistical Testing and Distributional Methods}
\label{sec:rw_statistical}

This subsection covers the work most closely related to our
methodological contribution: the use of formal statistical
tests and distributional comparisons for model evaluation.

\textbf{Distributional evaluation of language models.}
MAUVE \cite{pillutla2021mauve} uses the KL divergence between
the generated and reference text distributions (in an
embedding space) to evaluate open-ended generation quality.
MAUVE operates at the \emph{corpus} level, comparing
distributions of generations; our test operates at the
\emph{instance} level, comparing a single generation's entropy
sequence against a reference CDF.
The conceptual parallel that distributional divergence is a
more informative signal than scalar summaries is shared, but
the objects being compared and the statistical machinery
differ entirely.

\textbf{Conformal prediction for language models.}
\cite{quach2023conformal} and \cite{kumar2023conformal}
apply conformal prediction to construct prediction sets with
finite-sample coverage guarantees.
These methods guarantee that the true answer is contained in
the output set with probability $1 - \alpha$, but they do not
produce a detection score for individual generations.
Our framework is complementary: conformal methods control
\emph{coverage} (the probability of including the correct
answer), while CES controls \emph{detection error} (the
probability of flagging a faithful generation or missing a
hallucinated one).


\textbf{Distinction from prior work.}
To our knowledge, no prior work formalises hallucination
detection as a hypothesis test over the entropy distribution.
The closest existing methods (perplexity and LNE) can be
viewed as tests of the \emph{mean} of the entropy
distribution, but they lack a formal null distribution,
provide no Type~I or Type~II error guarantees, and require
task-specific threshold selection. 
\cite{du2024haloscope} frame hallucination detection via a Huber contamination model, a framework that resembles our own. 
However, their framework does not include oracles which permit diverse definition of hallucinations nor provable sample convergence guarantees via the calibration distributions; these are an integral part of our construction that distinguishes our contribution from prior work. 
Moreover, \cite{huang2023look} test many varieties of statistics of the entropy signals, but do not extend over to geometric means. 
CES differs in three respects:
(i)~it tests the properties of the full distribution (mean \emph{and} tail)
rather than a single moment;
(ii)~allows for both supervised and unsupervised variants with comparable performance;
and (iii)~it comes with finite-sample calibration guarantees,
exponential power bounds, and contamination robustness for
unsupervised deployment.

\end{document}